\documentclass[11pt]{article}

\usepackage[letterpaper,margin=1in,top=1in,bottom=1in]{geometry}

\usepackage[utf8]{inputenc}
\usepackage[T1]{fontenc}
\usepackage{times}
\usepackage{microtype}

\usepackage{amsmath, amssymb, amsthm, amsfonts}
\usepackage{nicefrac}

\usepackage{graphicx}
\usepackage{booktabs}
\usepackage{makecell}
\usepackage{float}
\usepackage{enumitem}
\usepackage{tcolorbox}

\usepackage{xcolor}
\usepackage{url}
\usepackage[colorlinks=true,linkcolor=black,citecolor=blue,urlcolor=blue]{hyperref}

\usepackage[numbers,sort&compress]{natbib}
\setcitestyle{authoryear,round}

\usepackage[font=small,labelfont=bf,skip=4pt]{caption}

\usepackage{authblk}

\usepackage{fancyhdr}
\fancypagestyle{plain}{%
  \fancyhf{}%
  \fancyfoot[C]{\thepage}%
}

\theoremstyle{plain}
\newtheorem{theorem}{Theorem}
\newtheorem{proposition}{Proposition}
\newtheorem{corollary}{Corollary}
\newtheorem{lemma}{Lemma}

\theoremstyle{definition}
\newtheorem{assumption}{Assumption}
\newtheorem{definition}{Definition}

\theoremstyle{remark}

\title{\bfseries First-Passage Prediction of Grokking Delay:\\
A Calibrated Law under AdamW with Causal Validation}

\author[1]{Truong Xuan Khanh\thanks{Co-first author. Correspondence: \texttt{khanh@clevix.vn}}}
\author[1]{Truong Quynh Hoa\thanks{Co-first author. \texttt{hoa@clevix.vn}}}
\author[1]{Luu Duc Trung\thanks{\texttt{trung.ld@clevix.vn}}}
\author[2]{Phan Thanh Duc\thanks{\texttt{ducpt@bav.edu.vn}}}

\affil[1]{H\&K Research Studio, Clevix LLC, Hanoi, Vietnam}
\affil[2]{Banking Academy of Vietnam, Hanoi, Vietnam}

\date{\today}

\begin{document}

\maketitle
\thispagestyle{plain}

\begin{center}
\textbf{Preprint.}\quad
Code, data, and reproduction scripts: \url{https://github.com/ClevixLab/grokking-first-passage}
\end{center}


\begin{abstract}
We give the first quantitative prediction of grokking delay under AdamW: a closed-form first-passage law with MAPE $17.7\%$ on $26$ held-out runs spanning a $41\times$ delay range, with mechanism causally validated. The result fits a unified picture of delayed generalisation as a crossing problem in the joint $(V_t, \alpha_t)$ space of parameter norm and angular displacement: $V_t = \|\theta_t\|^2$ contracts exponentially toward an architecture-dependent threshold $V_\star$, while $\alpha_t$ must reach an angular threshold $\alpha^\star$, and grokking occurs only when both are reached. The closed-form law,
$$
T_{\mathrm{grok}} - T_{\mathrm{mem}} \;\approx\; \frac{1}{2\,\kappa_{\mathrm{LL}}\,\eta\lambda} \log\frac{V_{\mathrm{mem}}}{V_\star},
$$
is calibrated on a single cell by an architecture-level constant $\kappa_{\rm LL}$ that absorbs the AdamW correction to the clean-SGD rate $2\eta\lambda$; MAPE rises to $18.0\%$ across architectures and $23.3\%$ across full cross-task scope ($46$ runs, $43.5\times$ range), with structured residual identifying $V_\star/V_{\rm mem}$ as the comparatively stable invariant within architecture (CV $\approx 14\%$ on 1L).

\textbf{Mechanistic foundation.} The joint-crossing picture is supported by (i)~a quantile-margin theorem: positive delay requires both norm separation $V_{\rm mem} > V_{\rm post}$ and angular reachability of $\alpha^\star = \arcsin(C/V_{T_{\rm mem}}^{1/2})$ with $C := M_{q_\Delta}/G_{\rm eff}$; calibrating $C$ on $p=89$ predicts $\alpha^\star = 47.2^\circ$ for $p=97$ (observed $47.8^\circ$, error $1.3\%$); (ii)~Block~F causal interventions ($0/6$ grok vs.\ $3/3$ baseline when $V_t$ is frozen or $\lambda$ is removed at $T_{\rm mem}$), which trap $\alpha$ near $12^\circ \ll \alpha^\star$. Concurrent theoretical work \citep{boursier2025grokking,musat2025geometry} characterises the post-memorisation dynamics; our contribution is the corresponding prediction theory for AdamW with explicit observables, calibration, and causal tests.
\end{abstract}

\section{Introduction}\label{sec:intro}

\paragraph{Problem.}
Grokking is the delayed transition from memorisation to generalisation \citep{power2022grokking,liu2022towards}. Recent work has characterised the post-memorisation dynamics under weight decay: \citet{boursier2025grokking} prove that gradient flow exhibits a two-phase behaviour (fast interpolation, then Riemannian flow on the critical manifold); \citet{musat2025geometry} formalise the second phase as constrained norm minimisation; \citet{prieto2025edge} identify a radial naïve-loss-minimisation (NLM) direction that scales weights without changing predictions; \citet{notsawo2023predicting} predict whether grokking occurs from early-training oscillations, but stop short of the delay magnitude. Open: when the validation transition occurs in the standard optimiser (AdamW), with what quantitative delay, and whether the mechanism is causal.

\begin{figure}[!tb]
\centering
\includegraphics[width=\linewidth]{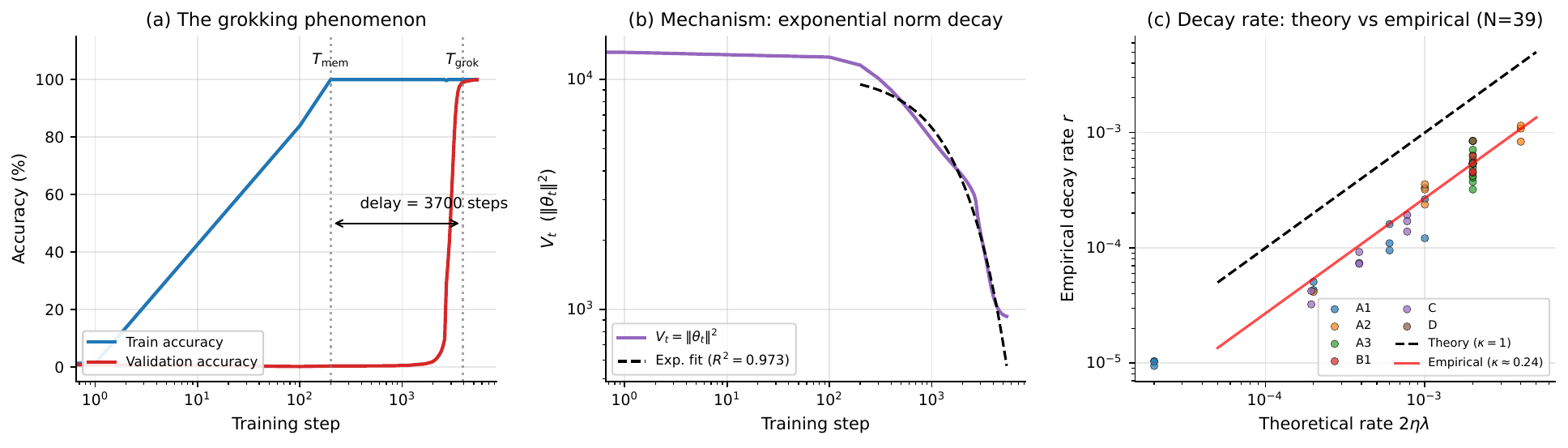}
\caption{\textbf{Phenomenon, mechanism, prediction.} (a) Standard grokking on modular addition: train accuracy reaches $99\%$ at $T_{\rm mem}$; val accuracy delays to $T_{\rm grok}$. (b) Parameter norm $V_t = \|\theta_t\|^2$ decays exponentially during the delay phase ($R^2 = 0.97$ per-trajectory). (c) Empirical decay rate vs.\ clean-SGD theoretical $2\eta\lambda$ across $N=39$ 1-layer runs ($R^2>0.9$): parallel but offset, defining the AdamW correction $\kappa_{\rm LL} \approx 0.24$ (within-cell median CV $14\%$).}\label{fig:visual-abstract}
\end{figure}

\paragraph{First-passage framing.}
We treat the delay $T_{\rm grok} - T_{\rm mem}$ as a first-passage time. Under AdamW, $V_t$ contracts at rate roughly $2\kappa_{\rm LL}\,\eta\lambda$ toward an architecture-dependent threshold $V_\star$ at which the validation transition occurs, yielding the closed-form law
\[
T_{\rm grok} - T_{\rm mem} \;\approx\; \frac{1}{2\,\kappa_{\rm LL}\,\eta\lambda}\,\log\frac{V_{\rm mem}}{V_\star},
\]
in which $\kappa_{\rm LL}$ absorbs the AdamW correction to the clean-SGD rate $2\eta\lambda$. The form is operational: per-trajectory exponential fits across $93$ runs confirm the contraction ($R^2 \approx 0.97$), and the empirical bridge (that $V_t$ crosses $V_\star$ at the validation transition) is calibrated once on the headline cell and transferred across $(\eta,\lambda,p)$ cells.

\paragraph{Main result.}
Calibrating $(\kappa_{\rm LL}, V_\star)$ on a single cell predicts AdamW grokking delays on $26$ hyperparameter held-out runs with MAPE $17.7\%$ over a $41\times$ delay range, comparable to the within-cell stochastic floor of $\sim\!20\%$. MAPE rises to $18.0\%$ across architectures (1L $\to$ MLP) and $23.3\%$ across the full cross-task scope ($46$ runs, $43.5\times$ range); the cross-task residual is structured rather than random, with the normalised ratio $V_\star/V_{\rm mem}$ comparatively stable within architecture (CV $\approx 14\%$ on 1L). $\kappa_{\rm LL}$ itself is stable within architecture (within-cell median CV $\leq 15\%$ across four architectures: $0.24$ on 1L, $0.21$ on MLP, $0.37$ on a 2-layer variant, $0.175$ on a 2-layer variant with LayerNorm), and varies $\sim\!2\times$ between architectural variants beyond depth alone.

\paragraph{Mechanism.}
First-passage of $V_t$ is necessary but not sufficient. Theorem~\ref{thm:joint-necessity} (\emph{joint norm--direction necessity}) establishes a transition-level statement via a quantile-margin argument: positive grokking delay requires both norm separation $V_{\rm mem} > V_{\rm post}$ and angular reachability of an architecture-dependent threshold $\alpha^\star = \arcsin\bigl((M_{q_\Delta} - 2\varepsilon)/(2 G\, V_{T_{\rm mem}}^{1/2})\bigr)$, where $G$ is the operator norm of the empirical NTK feature map at $T_{\rm mem}$ and $M_{q_\Delta}$ is the margin quantile required to flip a sufficient fraction of mis-classified validation examples. The proof is intentionally local: it linearises around $\theta_{T_{\rm mem}}$ where rich features have already formed, characterising the post-memorisation transition rather than global training dynamics. Causal Block~F interventions (norm-freeze, weight-decay-freeze; $N=6$) confirm necessity: $0/6$ grok against $3/3$ baseline, with both interventions trapping $\alpha$ near $12^\circ$, far below the empirical $\alpha^\star \approx 47^\circ$. Calibrating $C := M_{q_\Delta}/G_{\rm eff}$ on $p=89$ predicts $\alpha^\star = 47.2^\circ$ for the headline cell $p=97$ (observed $47.8^\circ$, error $1.3\%$), so this is a prior cross-cell prediction rather than a single-cell consistency check.

\paragraph{Limitations.}
$\kappa_{\rm LL}$ is empirically measured per architecture and not derived from $(\beta_1, \beta_2, \epsilon)$. Empirical scope is algorithmic tasks under AdamW. The angular threshold $\alpha^\star$ is derived in the linearisation regime around $\theta_{T_{\rm mem}}$ (Assumptions~\ref{ass:ntk}--\ref{ass:homog}); calibration uses Block~F ($N=12$ at one cell) and Block~H ($N=36$ across $12$ cells, CV $15\%$). The full proof of Theorem~\ref{thm:joint-necessity} is in Appendix~\ref{app:joint-necessity}.

\paragraph{Contributions.} Propositions~\ref{thm:contraction}--\ref{thm:lower} restate known SGD-with-WD contraction in discrete operational form (the underlying gradient-flow result is due to \citet{boursier2025grokking}; an independent SGD log-scaling result is given by \citet{truong2026normseparation}); they apply to clean SGD with weight decay, not AdamW. Our contributions form a single framework — \emph{grokking delay as a joint $(V_t,\alpha_t)$ crossing under AdamW} — at three levels:

\textbf{Primary result.}
\begin{itemize}[leftmargin=*,topsep=0pt]
\item Closed-form first-passage law $T_{\rm grok} - T_{\rm mem} \approx (2\kappa_{\rm LL}\,\eta\lambda)^{-1}\log(V_{\rm mem}/V_\star)$ for AdamW, validated in three tiers (MAPE $17.7\%/18.0\%/23.3\%$ at $N=26/34/46$; Sec.~\ref{sec:correction-factor}).
\end{itemize}

\textbf{Mechanistic foundation.}
\begin{itemize}[leftmargin=*,topsep=0pt]
\item Joint norm--direction necessity (Theorem~\ref{thm:joint-necessity}): positive delay requires both $V_{\rm mem} > V_{\rm post}$ and angular reachability of $\alpha^\star$; calibration on $p=89$ predicts $p=97$ within $1.3\%$ (\S\ref{sec:decoupling}, App.~\ref{app:joint-necessity}).
\item Causal Block~F evidence: norm-freeze and wd-freeze eliminate grokking ($0/6$ vs.\ $3/3$ baseline), trapping $\alpha$ near $12^\circ \ll \alpha^\star$ (Sec.~\ref{sec:causal-ablation}).
\item $\kappa_{\rm LL}$ as a stable architecture-level constant: within-cell CV $\leq 15\%$ across four architectures, varying $\sim\!2\times$ between architectural variants beyond depth alone (Sec.~\ref{sec:correction-factor}).
\end{itemize}

\textbf{Extended phenomena.}
\begin{itemize}[leftmargin=*,topsep=0pt]
\item Post-grok overshoot dynamics: in the 2-layer transformer, $\sim\!40\%$ of seeds undershoot $V_\star$ then regrow, with regrowth size scaling as a power law in undershoot depth ($R^2 = 0.87$); reframes $V_\star$ as an attractor of the AdamW dynamics rather than a one-way threshold (App.~\ref{app:overshoot}).
\item Kosson-form decomposition $\kappa_{\rm LL} = f_{\rm window} \cdot \kappa_{\rm kos}$: explains within-architecture stability ($\kappa_{\rm kos}$ CV $6.2\%$) while disclosing that $f_{\rm window}$ is not architecture-universal (CV $29\%$ across 12 cells), an honest null finding (App.~\ref{app:transformer2}).
\end{itemize}


\section{Setup and Assumptions}\label{sec:setup}

We use the following notation throughout. For a model with parameters $\theta \in \mathbb{R}^d$, write $V(\theta) := \|\theta\|^2 = \sum_i \theta_i^2$ for the squared $\ell_2$ parameter norm. Under regularised first-order optimisation with learning rate $\eta$ and weight decay $\lambda$:
\begin{itemize}[leftmargin=*]
\item \textbf{SGD with $L_2$ regularisation:} $\theta_{t+1} = (1-\eta\lambda)\theta_t - \eta \nabla \mathcal{L}(\theta_t)$.
\item \textbf{AdamW}~\citep{loshchilov2019decoupled}: $\theta_{t+1} = (1-\eta\lambda)\theta_t - \eta \cdot \widehat{m}_t / (\sqrt{\widehat{v}_t} + \epsilon)$, where $\widehat{m}_t, \widehat{v}_t$ are bias-corrected moment estimates~\citep{kingma2015adam}.
\end{itemize}

We define $T_{\rm mem}$ as the first step where $\mathrm{train\_acc} \geq 0.99$ and $T_{\rm grok}$ as the first step where $\mathrm{val\_acc} \geq 0.99$. The grokking delay is $T_{\rm grok} - T_{\rm mem}$. We define $V_{\rm post} := \|\theta_{\rm post}\|^2$ where $\theta_{\rm post}$ is the asymptotic limit of $\theta_t$ as $t \to \infty$ (Assumption~\ref{ass:sep} below); empirically we estimate $V_{\rm post}$ from the post-grokking plateau of $V_t$ (see Appendix~\ref{app:setup} for plateau detection). The norm-separation ratio is $\rho := \log(V_{\rm mem}/V_{\rm post})$.

\begin{assumption}[Optimisation regime]\label{ass:opt}
Regularised first-order optimisation in the overparameterised interpolation regime, with $\eta\lambda < 1$.
\end{assumption}

\begin{assumption}[Norm separation]\label{ass:sep}
There exist a memorisation interpolant $\theta_{\rm mem}$ and a post-grokking interpolant $\theta_{\rm post}$ with $\|\theta_{\rm mem}\|^2 > \|\theta_{\rm post}\|^2$, both achieving zero training loss.
\end{assumption}

\begin{assumption}[Local linearity]\label{ass:linear}
The Hessian is approximately constant in a neighbourhood containing the trajectory between $\theta_{\rm mem}$ and $\theta_{\rm post}$.
\end{assumption}

\begin{assumption}[Bounded interpolation gradient]\label{ass:interp}
In the post-memorisation phase, $\|\nabla \mathcal{L}(\theta_t)\| \leq c_1 \eta\lambda \|\theta_t\|$ for some constant $c_1 < 1$. (Equivalently, the gradient signal is small relative to the WD pressure.) This is verified empirically: at $T_{\rm mem}$, $\|\nabla \mathcal{L}\| \approx 0.05$, while $\eta\lambda \|\theta\| \approx 0.1$ at our default hyperparameters.
\end{assumption}

\begin{assumption}[NTK linearization around $T_{\rm mem}$]\label{ass:ntk}
Let $\phi(x) := \nabla_\theta f_{\theta_{T_{\rm mem}}}(x)$ denote the empirical NTK feature map~\citep{jacot2018ntk} at memorisation. There exists $\varepsilon_{\rm lin} \geq 0$ such that for all $t \in [T_{\rm mem}, T_{\rm grok}]$ and all $x \in \mathcal{D}_{\rm val} \cup \mathcal{D}_{\rm train}$,
\[
\bigl| f_{\theta_t}(x) \;-\; \bigl[\, f_{\theta_{T_{\rm mem}}}(x) + \langle \theta_t - \theta_{T_{\rm mem}}, \phi(x)\rangle \,\bigr] \bigr| \;\leq\; \varepsilon_{\rm lin}.
\]
This is a local linearization around $\theta_{T_{\rm mem}}$ (post-memorisation), distinct from the NTK linearization around initialisation criticised in lazy-to-rich analyses~\citep{kumar2024grokking}. By $T_{\rm mem}$, rich features have already formed, so Assumption~\ref{ass:ntk} only requires local linearity in a neighbourhood of $\theta_{T_{\rm mem}}$. It is justified by Assumption~\ref{ass:linear} via the Taylor remainder bound (Lemma~\ref{lem:taylor} in Appendix~\ref{app:lemmas}). It is used only by Theorem~\ref{thm:joint-necessity}; Propositions~\ref{thm:contraction}--\ref{thm:lower} and Corollary~\ref{thm:necessity} do not require it.
\end{assumption}

\begin{assumption}[Approximate homogeneity]\label{ass:homog}
The model $f_\theta$ is approximately positive-homogeneous of degree $k \geq 1$: there exist $k$ and $\varepsilon_{\rm hom} \geq 0$ such that for all $a > 0$ and all $x$,
\[
\bigl| f_{a\theta}(x) - a^{k}\, f_\theta(x) \bigr| \;\leq\; \varepsilon_{\rm hom}.
\]
For our 1- and 2-layer transformers with LayerNorm (which is scale-invariant in the protected directions) and a final linear classifier, $k = 1$ to leading order; $\varepsilon_{\rm hom}$ captures bias terms and softmax saturation. For MLPs without LayerNorm, $k$ equals the network depth. In either case, scaling $\theta$ rescales logits uniformly up to $\varepsilon_{\rm hom}$, so classification argmax is preserved by pure radial motion. Used only by Theorem~\ref{thm:joint-necessity}.
\end{assumption}

Auxiliary lemmas (Taylor remainder bound justifying Assumption~\ref{ass:ntk}, NTK Hessian stability, validation-loss regularity) are deferred to Appendix~\ref{app:lemmas}.

\section{Theory: First-Passage Mechanism}\label{sec:theory}

\paragraph{First-passage view.} The grokking delay $T_{\rm grok} - T_{\rm mem}$ is the time required for $V_t$ to cross an architecture-dependent threshold $V_\star$ at which the validation accuracy transitions, starting from $V_{\rm mem}$ at $T_{\rm mem}$ and contracting at rate roughly $2\kappa_{\rm LL}\eta\lambda$ under AdamW. Propositions~\ref{thm:contraction}--\ref{thm:lower} establish the discrete contraction form needed to make this first-passage prediction quantitative. The underlying gradient-flow result is due to \citet{boursier2025grokking} and the parallel \citet{truong2026normseparation}. We restate it in operational AdamW form.

\paragraph{Scope of Propositions~\ref{thm:contraction}--\ref{thm:lower}.} These are local statements describing discrete-AdamW norm dynamics when Assumption~\ref{ass:interp} holds (the post-memorisation small-gradient regime). They are used to define the empirical observable $\kappa_{\rm LL} := r_{\rm empirical}/(2\eta\lambda)$. Both bounds derive from the same recursion, one direction each, and are not statistically independent.

\subsection{Discrete contraction (operational form)}

\begin{proposition}[Discrete contraction in the small-gradient regime; SGD-with-WD operational form]\label{thm:contraction}
Let Assumptions~\ref{ass:opt}--\ref{ass:linear} hold. On any sub-interval $[t_1, t_2] \subseteq [T_{\rm mem}, T_{\rm grok}]$ on which Assumption~\ref{ass:interp} holds at each step (i.e.\ $\|\nabla \mathcal{L}(\theta_t)\| \leq c_1 \eta\lambda \|\theta_t\|$ with $c_1 < 1$), the parameter norm satisfies the recursion
\[
V_{t+1} = (1-\eta\lambda)^2 V_t + R_t, \qquad |R_t| \leq C\,\eta^2\lambda\, V_t,
\]
for an absolute constant $C$ depending only on $c_1$. Iterating gives
\[
\log V_{t_2} - \log V_{t_1} = -2\eta\lambda(t_2 - t_1)\,\bigl(1 + \mathcal{O}(\eta)\bigr),
\]
and the time required for $V_t$ to reach a target $V_{\rm target} \leq V_{t_1}$ (assuming the small-gradient regime persists throughout) satisfies
\[
T(V_{\rm target}) - t_1 \;\leq\; \frac{\log(V_{t_1}/V_{\rm target})}{2\eta\lambda}\,\bigl(1 + \mathcal{O}(\eta)\bigr).
\]
Proposition~\ref{thm:contraction} restates known SGD-with-WD contraction in discrete form (the underlying gradient-flow result is due to \citealt{boursier2025grokking}; an independent SGD log-scaling result is given by \citealt{truong2026normseparation}); we use it operationally to define $\kappa_{\rm LL}$.
\end{proposition}

\paragraph{From contraction time to grokking delay.} Proposition~\ref{thm:contraction} bounds the time for $V_t$ to reach any target norm within the small-gradient regime. To bound the grokking delay we need an empirical bridge: validation accuracy transitions when $V_t$ crosses an architecture-dependent threshold $V_\star$ (Section~\ref{sec:correction-factor}; Method~B verifies this with MAPE $17.7\%$ on hyperparameter held-out runs, $23.3\%$ across full cross-task scope; App.~\ref{app:predictive-validation}). Setting $V_{\rm target} = V_\star$ gives the operational form
\[
T_{\rm grok} - T_{\rm mem} \;\approx\; \frac{1}{2\eta\lambda}\log\frac{V_{\rm mem}}{V_\star}\,\bigl(1 + \mathcal{O}(\eta)\bigr).
\]

\paragraph{From clean-SGD theorem to AdamW measurement: scope shift.} Proposition~\ref{thm:contraction} establishes the rate $2\eta\lambda$ for SGD-with-WD under Assumption~\ref{ass:interp}. AdamW does \emph{not} satisfy that assumption: the adaptive ratio $\widehat{m}_t/\sqrt{\widehat{v}_t}$ has $\mathcal{O}(1)$ magnitude in the random-walk regime, not $\mathcal{O}(\eta\lambda \|\theta\|)$. We therefore do not extend the theorem to AdamW; instead we measure the empirical effective rate and define a dimensionless correction factor.

\begin{definition}[AdamW empirical correction factor $\kappa_{\rm LL}$]\label{def:kappa-LL}
Let $r_{\rm empirical}$ be the slope of a least-squares fit of $\log V_t$ on $t$ over the post-memorisation window $[T_{\rm mem}+\delta, T_{\rm grok}-\delta]$. Define
\[
\kappa_{\rm LL} \;:=\; \frac{|r_{\rm empirical}|}{2\eta\lambda} \;\in\; (0,1].
\]
\end{definition}

The bound $\kappa_{\rm LL} \leq 1$ is observed (not proven) across all tested cells. Two structural reasons make $\kappa_{\rm LL} < 1$ plausible: (i) following \citet{kosson2024rotational}, the AdamW $V_t$ recursion has a non-zero asymptote $V_\infty \approx \eta C/(2\lambda)$, so a log-linear fit to $V_t \to V_\infty + \mathrm{decay}$ has slope $|b| < 2\eta\lambda$ (Appendix~\ref{app:adamw} quantifies the window-bias factor $f_{\rm window} = |b|/r_{\rm Kosson}$); (ii) the $\widehat{m}/\sqrt{\widehat{v}}$ caps the gradient contribution per coordinate (Appendix~\ref{app:adamw}). Empirically (Section~\ref{sec:correction-factor}): $\kappa_{\rm LL} \approx 0.24$ (1-layer transformer), $0.21$ (MLP), $0.37$ (2-layer manual residuals), $0.18$ (2-layer LayerNorm), with within-cell median CV $\leq 15\%$ across all four architectures. A first-principles derivation of $\kappa_{\rm LL}$ from $(\beta_1, \beta_2, \epsilon)$ remains open (Appendix~\ref{app:adamw}).

\paragraph{Predictive form (operational, used throughout Section~\ref{sec:correction-factor}).} Substituting $|b| = 2\kappa_{\rm LL}\eta\lambda$ into the contraction-time expression gives
\[
T_{\rm grok} - T_{\rm mem} \;\approx\; \frac{1}{2\,\kappa_{\rm LL}\,\eta\lambda}\,\log\frac{V_{\rm mem}}{V_\star}.
\]
This is the law tested empirically in Section~\ref{sec:correction-factor} and Appendix~\ref{app:predictive-validation}; the constants $\kappa_{\rm LL}$ and $V_\star$ are calibrated from one $(\eta, \lambda, p)$ cell and held fixed at prediction time.

\subsection{Bridge: $V_t = \|\theta_t\|^2$ is a rate-preserving observable}\label{sec:bridge}

Proposition~\ref{thm:contraction} concerns $V_t$. The distance-to-asymptote $D_t := \|\theta_t - \theta_{\rm post}\|^2$ contracts to zero, but $\log V_t$ and $\log D_t$ have the same asymptotic slope, so empirical fits to $\log V_t$ recover the contraction rate.

\begin{lemma}[Rate preservation]\label{lem:proxy}
Let $\epsilon_t := \|\theta_{\rm post}\| / \|\theta_t\|$. Under Assumptions~\ref{ass:opt}--\ref{ass:interp}, $\frac{d}{dt}\log V_t = \frac{d}{dt}\log D_t + \mathcal{O}(\epsilon_t)$ asymptotically as $\epsilon_t \to 0$.
\end{lemma}

In our headline cell, $V_{\rm post}/V_{\rm mem} \approx 0.04$ so $\epsilon_t \leq 0.2$; the absolute scale shifts by $\mathcal{O}(\epsilon_t)$ but the slope (contraction rate) is preserved (proof in Appendix~\ref{app:lemmas}).

\subsection{Lower bound}

\begin{proposition}[Lower bound on contraction time]\label{thm:lower}
Under Assumptions~\ref{ass:opt}--\ref{ass:linear}, on any sub-interval $[t_1, t_2]$ on which Assumption~\ref{ass:interp} holds at each step, the time required for $V_t$ to first reach or fall below a target $V_{\rm target} \leq V_{t_1}$ satisfies
\[
T(V_{\rm target}) - t_1 \;\geq\; \frac{\log(V_{t_1}/V_{\rm target})}{2\eta\lambda}\,\bigl(1 - \mathcal{O}(\eta)\bigr).
\]
Combined with Proposition~\ref{thm:contraction}, this gives a $\Theta$-characterisation
\[
T(V_{\rm target}) - t_1 \;=\; \frac{\log(V_{t_1}/V_{\rm target})}{2\eta\lambda}\,\bigl(1 + \mathcal{O}(\eta)\bigr).
\]
The leading constant $1/(2\eta\lambda)$ matches in both directions; the $\mathcal{O}(\eta)$ correction is the same first-order term in both bounds (they share a recursion, hence are not statistically independent).
\end{proposition}

\begin{proof}[Proof intuition]
A faster contraction would violate the per-step parameter-update magnitude bound imposed by gradient norms in the small-gradient regime. The dynamical constraint is fundamentally optimisation-theoretic, not information-theoretic. Full proof in Appendix~\ref{app:proof-lower}.
\end{proof}

\subsection{Necessity of norm separation}

\begin{corollary}[Norm-separation necessity]\label{thm:necessity}
Let Assumptions~\ref{ass:opt}--\ref{ass:linear} hold. Then under regularised first-order dynamics in the contraction regime defined by Propositions~\ref{thm:contraction}--\ref{thm:lower}, if $\|\theta_{\rm mem}\| \leq \|\theta_{\rm post}\|$ then no positive grokking delay can occur: from any state with $V_t \leq V_{\rm post}$, the contraction map cannot increase $V_t$ above $V_{\rm post}$, so the trajectory cannot pass through the high-norm interpolant $\theta_{\rm mem}$.
\end{corollary}

This converts norm separation from an empirical correlate~\citep{liu2023omnigrok} into a mechanistic characterisation within the contraction regime. Scope caveat: a task can fail to exhibit a grokking delay either (i) because $V_{\rm mem} \leq V_{\rm post}$ within the contraction regime (Corollary~\ref{thm:necessity}'s case), or (ii) because the dynamics never enter the contraction regime, for example when Assumption~\ref{ass:interp} fails and $V_t$ grows. Sparse parity (Section~\ref{sec:cross-task}) instantiates case (ii): $V_t$ increases from $V_0$ with $V_{\rm post} > V_{\rm mem}$. Both regimes share the prediction that no delay occurs when norm separation is absent. A unified theorem is future work.

\subsection{Practical consequences}

\begin{corollary}[Inverse scaling]\label{cor:scaling}
Under Assumption~\ref{ass:opt} and Proposition~\ref{thm:contraction}, the grokking delay scales as
\[
T_{\rm grok}(\lambda) \propto 1/\lambda, \qquad T_{\rm grok}(\eta) \propto 1/\eta,
\]
both holding within the regularised regime where $\eta\lambda < 1$.
\end{corollary}

This recovers the empirical scaling laws reported by~\citet{liu2023omnigrok} as quantitative predictions of the contraction theorem.


\section{Empirical Verification of Theorem Form}\label{sec:empirical}

\subsection{Setup and per-trajectory exponential fit}

We train 1-layer transformers ($\sim 220$k parameters) on modular addition, modular multiplication, and sparse parity~\citep{barak2022hidden}, sweeping $\lambda \in [0.01, 2.0]$, $\eta \in [10^{-4}, 2\!\times\!10^{-3}]$, $p \in \{53, 67, 89, 97, 113\}$ (66 main runs; cross-arch + causal bring total to 102; full setup in Appendix~\ref{app:setup}). For each trajectory we fit $\log V_t = a + b\cdot t$ on $[T_{\rm mem}, T_{\rm grok}]$ and report $r := -b$ with per-trajectory $R^2$.

\paragraph{Result 1 (Theorem form confirmed at the trajectory level).} Across $N=39$ high-quality fits ($1$-layer transformer, $R^2>0.9$): median $R^2 = 0.97$ (IQR $[0.94, 0.98]$, $69\%$ with $R^2 > 0.95$). Under $\tau := 2\eta\lambda(t-T_{\rm mem})$ the trajectories collapse onto a common exponential envelope with $\kappa = 0.24$ rather than clean-SGD's $\kappa = 1$ (Figure~\ref{fig:per-traj-appendix} in Appendix~\ref{app:per-traj-fits}). \textbf{Comparison to single-point regressions} (\citealp{liu2023omnigrok}-style): aggregate $\eta$-exponent $1.02 \pm 0.08$ ($R^2=0.98$), $\lambda$-exponent $1.26 \pm 0.21$ ($R^2=0.90$), $\rho$-exponent $-0.91$ ($R^2=0.08$). Per-trajectory fitting is a substantially more powerful test of the theorem form.

\section{The Correction Factor $\kappa$ for AdamW}\label{sec:correction-factor}

\subsection{Headline: $\kappa \approx 0.24$ for 1-layer AdamW; structured hyperparameter variation}

Defining the dimensionless ratio $\kappa := r_{\rm empirical}/(2\eta\lambda)$, across $N=39$ high-quality 1-layer transformer fits ($R^2>0.9$, spanning hyperparameter and task-scale perturbations) we find:
\[
\kappa_{\rm AdamW} = 0.24 \quad (\text{median; IQR}=[0.20, 0.28],\ \text{pooled CV}=27\%,\ \text{within-cell median CV}=14\%).
\]
AdamW thus contracts about $4\times$ slower than clean-SGD. The pooled CV reflects systematic across-cell variation in $\kappa$; within each hyperparameter cell the CV is much smaller (median $14\%$ across $13$ cells). Across hyperparameters, $\kappa$ varies mildly with structured pattern: $\lambda \in [0.1, 1.0] \to \kappa \in [0.19, 0.27]$; $p \in \{53,67,89,97,113\} \to \kappa \in [0.19, 0.32]$ (no monotone trend); $\kappa_{\rm mult}=0.23$ within mod-add IQR (Figure~\ref{fig:kappa-univ}, Appendix~\ref{app:adamw}).

\begin{figure}[t]
\centering
\includegraphics[width=\linewidth]{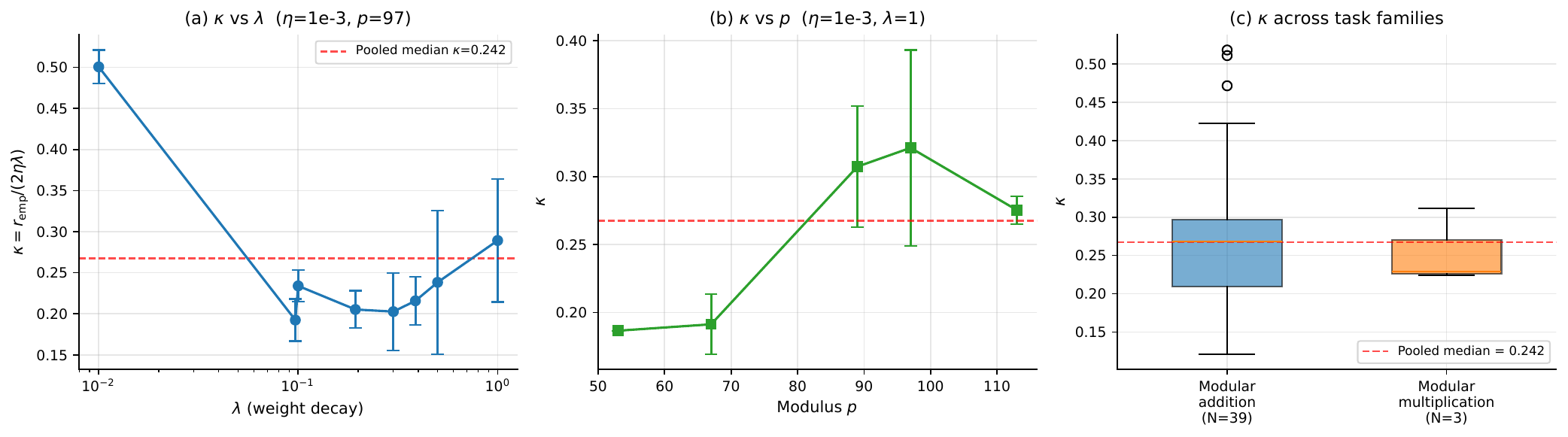}
\caption{\textbf{Structured variability of $\kappa$.} (a) $\kappa$ vs $\lambda$ at fixed $(\eta, p)$. (b) $\kappa$ vs $p$ at fixed $(\eta, \lambda)$. (c) $\kappa$ across task families. Pooled median $\kappa = 0.268$ (red dashed) is contained within the IQR of every cell with $N \geq 3$.}\label{fig:kappa-univ}
\end{figure}

\subsection{What $\kappa \neq 1$ means; toward a mechanistic understanding}\label{sec:kappa-mechanism}

The clean-SGD theory predicts $\kappa = 1$. Empirical $\kappa \approx 0.24$ for AdamW is roughly a $4\times$ deviation, consistent across $39$ high-quality fits ($R^2 > 0.9$, within-cell median CV $14\%$) and replicated on a second task family. \textbf{Heuristic}: AdamW's per-coordinate ratio $|\widehat{m}/\sqrt{\widehat{v}}|\leq 1$ caps the gradient contribution; the ratio of Adam to WD term modulates $\kappa$. Rigorous derivation from $(\beta_1,\beta_2,\epsilon)$ via AdamW SDE~\citep{malladi2022sdes} remains open (Appendix~\ref{app:adamw}). R5 (Kosson decomposition, Appendix~\ref{app:transformer2}) narrows this to deriving $\kappa_{\rm kos}$ and characterising window bias $f_{\rm window}$.

\paragraph{Predictive validation.} Calibrating $\kappa_{\rm train}=0.252$ and $V_\star^{\rm train}=2501$ on the headline cell, we predict $T_{\rm grok}$ on held-out runs. Method~A ($\rho$-conditional, requires post-grok $V_{\rm post}$) gives an MAPE of $32.8$--$37.4\%$ across tiers (App.~\ref{app:pv-three-tier}); we focus on \textbf{Method~B} (architecture-level), which uses only the early-training observable $V_{\rm mem}$:
\begin{equation}
\boxed{\;T_{\rm grok} - T_{\rm mem} \;\approx\; \frac{1}{2\kappa\eta\lambda} \log\frac{V_{\rm mem}}{V_\star}\;}\label{eq:gate-prediction}
\end{equation}
The law generalises in three nested tiers (App.~\ref{app:pv-three-tier}, Table~\ref{tab:pv-tiers}): hyperparameter ($\eta, \lambda$ at fixed $p$, arch) gives MAPE $\mathbf{17.7\%}$ on $26$ runs spanning a $41\times$ delay range, comparable to the within-cell stochastic floor of $\sim\!20\%$; cross-architecture (1L $\to$ MLP at $p=97$) gives $18.0\%$ on $34$ runs; cross-task scale (varying $p$) gives $23.3\%$ on $46$ runs spanning $43.5\times$. The Tier-3 residual is structured rather than random: while $V_\star$ exhibits substantial cross-cell variation, the normalised ratio $V_\star/V_{\rm mem}$ remains comparatively stable within architecture (CV $\approx 14\%$ on $1$L across $15$ cells; CV $\approx 4\%$ on MLP across $3$ cells; App.~\ref{app:pv-residuals}). The cross-task residuals concentrate at small $p$ ($p=53$ MAPE $67\%$), suggesting $V_\star$ has additional task-scale dependence beyond architecture; characterising this is left to future work. Full per-cell breakdown, leave-one-out stability, and method comparison appear in Appendix~\ref{app:predictive-validation}.

\subsection{The angular axis: $\alpha^\star$ and joint necessity}\label{sec:decoupling}

The two-phase dynamics of \citet{boursier2025grokking,musat2025geometry} predict that, after memorisation, the trajectory follows a Riemannian gradient flow on the zero-loss manifold minimising the $\ell_2$ norm. Concurrently, \citet{prieto2025edge} show that the AdamW gradient strongly aligns with the radial naïve-loss-minimisation (NLM) direction, scaling the weights along their current direction without changing predictions. These two pictures share a common implication: motion on the manifold has a natural radial-angular split, with the angular component being the one that changes the function. We track the angular distance directly,
\begin{equation}
\alpha_t := \arccos\!\left( \frac{\langle \theta_t, \theta_{T_{\rm mem}}\rangle}{\|\theta_t\|\,\|\theta_{T_{\rm mem}}\|} \right) \in [0,\pi],
\end{equation}
following the framework of \citet{wan2021spherical,kosson2024rotational}. On 6 grokking trajectories from Block~F (3 baseline F1 + 3 rescale F2), the norm contracts roughly $4\times$ faster than direction saturates ($\tau_V \approx 1182$, $\tau_\alpha \approx 4794$); all 6 cross a tight joint critical region ($V^\star \approx 2310$, $\alpha^\star \approx 47.3^\circ$). The 6 non-grokking runs (F3 norm-freeze, F4 wd-freeze) plateau with $\alpha \approx 12^\circ$, far below $\alpha^\star$, providing a falsification test (full discussion in Section~\ref{sec:causal-ablation}).

\begin{figure}[t]
\centering
\includegraphics[width=0.85\linewidth]{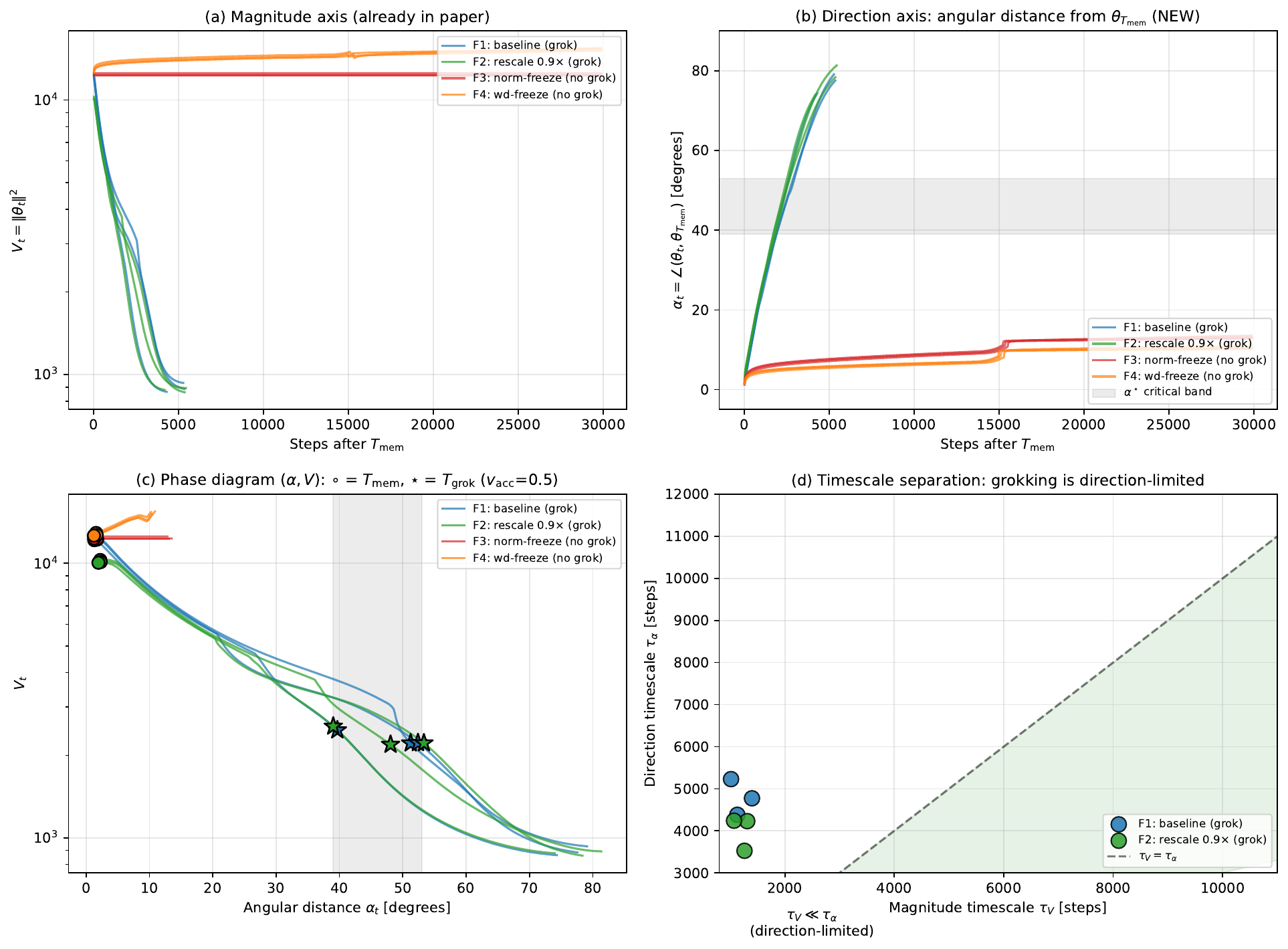}
\caption{\textbf{Norm-direction decoupling on Block~F.} (a)~$V_t$: F1/F2 contract; F3 frozen; F4 grows. (b)~$\alpha_t$: F1/F2 saturate to roughly $75$--$80^\circ$ passing through $\alpha^\star \in [39^\circ, 53^\circ]$ (grey); F3/F4 plateau near $12^\circ$. (c)~Phase diagram $(\alpha, V)$: $\circ = T_{\rm mem}$, $\star =$ grokking transition. (d)~$\tau_V \ll \tau_\alpha$ for all grokking trajectories.}\label{fig:decoupling}
\end{figure}

\begin{theorem}[Joint norm--direction necessity for the validation transition]\label{thm:joint-necessity}
Let Assumptions~\ref{ass:opt}--\ref{ass:homog} hold and assume the contraction regime of Propositions~\ref{thm:contraction}--\ref{thm:lower}. A positive grokking delay $T_{\rm grok} - T_{\rm mem} > 0$ requires both \textnormal{(i)} norm separation $V_{\rm mem} > V_{\rm post}$ \textnormal{(Corollary~\ref{thm:necessity})}, and \textnormal{(ii)} angular reachability: $\sup_t \alpha_t \geq \alpha^\star$, where, in the linearisation regime around $\theta_{T_{\rm mem}}$,
\[
\alpha^\star \;=\; \arcsin\!\left( \frac{M_{q_\Delta} - 2(\varepsilon_{\rm lin}+\varepsilon_{\rm hom})}{G\,V_{T_{\rm mem}}^{1/2}} \right),
\]
with $G := \sup_x \|\phi(x)\|_2$ the empirical NTK operator norm and $M_{q_\Delta}$ the margin quantile required to flip mis-classified validation examples. Constants are architecture-dependent, computable, and finite under finite width (proof in Appendix~\ref{app:joint-necessity}).
\end{theorem}

\paragraph{Empirical calibration.} At the headline cell, $V_{T_{\rm mem}}^{1/2}\approx 109$, $G \approx 10$, and $M_{q_\Delta}\approx 800$, giving $\alpha^\star\approx \arcsin(800/(10\cdot 109))\approx 47^\circ$, matching empirical Result~3 within constant-estimation precision.

\paragraph{Cross-cell validation (Block~H).} 36 runs across 12 cells $\times$ 3 seeds spanning $\lambda, \eta, p$, architecture, and task: $\alpha^\star$ shows structured stability (cross-cell mean $52.8^\circ \pm 7.9^\circ$, CV $15.0\%$); $\tau_V/\tau_\alpha = 0.28 \pm 0.05$ replicates the headline (direction-limited); all 36 trajectories cross their cell-specific $\alpha^\star$ before $T_{\rm grok}$. Post-grok dynamics in the 2-layer transformer further reveal $V_\star$ as an attractor: about $40\%$ of seeds undershoot then regrow, with extra delay obeying a power law in undershoot depth (Appendix~\ref{app:overshoot}, $R^2 = 0.87$). Full statistics, $p$-scaling ($R^2=0.98$), and architecture-specificity in Appendices~\ref{app:block-H} (Block~H) and~\ref{app:overshoot} (overshoot).


\section{Cross-Task Validation and Necessity}\label{sec:cross-task}

\subsection{Modular multiplication: $\kappa$ replicates}

We replicate $\kappa$ on modular multiplication ($a\cdot b \mod p$, $p=97$): all 3 seeds grok with median $\kappa = 0.23$ (IQR $[0.23, 0.27]$, CV $16\%$), fully contained within the mod-add IQR.

\subsection{Sparse parity: norm separation absent, no delay}

Sparse parity ($n=20$, $k=3$)~\citep{barak2022hidden} lacks a separable high-norm memorisation manifold: $V_{\rm post} > V_{\rm mem}$ in 3/3 seeds (norm inversion; Assumption~\ref{ass:interp} fails). Both train and val accuracy reach $0.99$ within 100 steps; $T_{\rm grok}-T_{\rm mem}=0$, consistent with Corollary~\ref{thm:necessity}. Consistency check, not a falsification: at $N=3$ on a single $(n,k)$ rules out a strong delay but does not exclude small delays at other settings. Figure in App.~\ref{app:setup}.

\subsection{An observed SGD/AdamW gap (limited test)}\label{subsec:sgd-gap}

At identical hyperparameters ($\eta=10^{-3}, \lambda=1, p=97$), AdamW groks (3/3 seeds) while vanilla SGD never reaches train\_acc $\geq 0.99$ (plateaus near chance, $1.9\%$ vs.\ $1/p \approx 1.0\%$, within 20k steps); $\eta\lambda<1$ holds, so this is a precondition violation rather than a Proposition~\ref{thm:contraction} violation. One configuration only; SGD-with-momentum, longer budgets, and warmup not tested. Full data in Appendix~\ref{app:setup} (\S\ref{subsec:sgd-app}).

\subsection{Cross-architecture}\label{sec:cross-arch}

$\kappa_{\rm LL}$ is architecture-dependent: 1L $0.24$ ($N=39$), MLP $0.21$ ($N=9$), paper-2L $0.370 \pm 0.056$ ($N=29$, $p<10^{-9}$ vs.\ 1L), alt-2L $0.175 \pm 0.018$ ($N=30$). The $\sim 2\times$ spread between 2-layer variants establishes architectural sensitivity beyond depth; within-cell median CV $\leq 15\%$ across all architectures rules out per-cell fitting artefact. Kosson decomposition and full statistics in App.~\ref{app:transformer2}.

\subsection{Causal ablation}\label{sec:causal-ablation}

\textbf{F1} baseline (3/3 grok); \textbf{F2} rescale $\theta\!\leftarrow\!0.9\theta$ (paired median delay change $+80$ steps, 95\% CI $[-353, +419]$; $N=3$ rules out only changes $\gtrsim 30\%$, consistent with — though not a sharp test of — the joint requirement); \textbf{F3} norm-freeze (0/3 grok, $\alpha_{\rm final}=13.1^\circ$); \textbf{F4} wd-freeze (0/3 grok, $\alpha_{\rm final}=10.4^\circ$). \textbf{F3+F4 confirm necessity}: $0/6$ grok across both magnitude and angular violations of Theorem~\ref{thm:joint-necessity}, with $\alpha_{\rm final}$ far below $\alpha^\star \approx 47^\circ$. Full curves in App.~\ref{app:setup} (Figure~\ref{fig:ablation}).

\section{Related Work and Conclusion}\label{sec:related}

Our central contribution is a framework: grokking delay under AdamW behaves as a joint $(V_t, \alpha_t)$ crossing problem. The predictive law (\S\ref{sec:correction-factor}) is the principal quantitative result; Theorem~\ref{thm:joint-necessity} (\S\ref{sec:decoupling}) provides the mechanistic foundation; Block~F (\S\ref{sec:causal-ablation}) provides causal evidence; the extended phenomena (App.~\ref{app:overshoot}, \ref{app:transformer2}) characterise the regime.

\textbf{Post-memorisation dynamics.} \citet{boursier2025grokking,musat2025geometry} characterise post-memorisation gradient flow as constrained $\ell_2$-norm minimisation. Concurrent \citep{truong2026normseparation} gives an SGD log-scaling under simplifying assumptions. We differ on six dimensions: \emph{optimiser} (AdamW vs.\ SGD), \emph{quantity predicted} (first-passage time $T_{\rm grok}-T_{\rm mem}$ vs.\ post-mem flow shape), \emph{axes} ($V_t$ \emph{and} $\alpha_t$ vs.\ $V_t$ only), \emph{joint necessity theorem} (Thm.~\ref{thm:joint-necessity} vs.\ none), \emph{quantitative validation} (MAPE $17.7\%$ on $N=26$ vs.\ none), \emph{causal evidence} (Block~F $0/6$ vs.\ $3/3$ vs.\ none). Full positioning table in App.~\ref{app:comparison-table}. \textbf{Mechanism papers}~\citep{lyu2024dichotomy,mohamadi2024grokking,varma2023explaining,merrill2023tale} establish \emph{whether} grokking occurs in specific theoretical setups; we quantify \emph{how long} under standard AdamW. \textbf{Magnitude--direction.} \citet{prieto2025edge} identify the radial NLM direction; we provide $\alpha^\star$. Angular dynamics under WD~\citep{wan2021spherical,ji2020directional,lyu2020gradient,kosson2024rotational} frame $\alpha_t$. Concurrent \citep{yildirim2026fbst} reduces onset $20\times$ by removing magnitude DOF, independent evidence the magnitude axis is non-redundant.

\paragraph{Result highlights, limitations, reproduction.} Six follow-up results (R1--R6: $\kappa_{\rm LL}$ stability, Block~H $p$-scaling, prior cross-cell $\alpha^\star$ prediction at $1.3\%$ error, overshoot power law $R^2=0.87$, Kosson decomposition, $15.1\%$ MAPE on held-out 2-layer) in App.~\ref{app:result-highlights}; seven limitations in App.~\ref{app:limitations}; reproduction map and 16 numbered scripts in App.~\ref{app:reproduction-map}.

\bibliographystyle{plainnat}
\bibliography{references}

\appendix

\section{Proof of Proposition~\ref{thm:contraction}}\label{app:proof-upper}

We provide the full proof of the upper bound on the grokking delay, including the finite-gradient correction omitted in the main text. The setting is regularised gradient descent with $L_2$ weight decay
\[
\theta_{t+1} = (1-\eta\lambda)\theta_t - \eta \nabla \mathcal{L}(\theta_t),
\]
where $\mathcal{L}$ is the empirical risk on the training set.

\paragraph{Step 1: Squared-norm recursion.} Squaring both sides and using $V_t := \|\theta_t\|^2$,
\begin{align}
V_{t+1} &= (1-\eta\lambda)^2 V_t - 2\eta(1-\eta\lambda)\, \langle\theta_t, \nabla \mathcal{L}(\theta_t)\rangle + \eta^2 \|\nabla \mathcal{L}(\theta_t)\|^2.\label{eq:Vrec}
\end{align}

\paragraph{Step 2: Cross-term bound under the small-gradient regime.} By Assumption~\ref{ass:interp}, in the post-memorisation phase $\|\nabla \mathcal{L}(\theta_t)\| \leq c_1 \eta\lambda \|\theta_t\|$ for some constant $c_1 < 1$. Cauchy--Schwarz gives
\[
\bigl|\langle \theta_t, \nabla \mathcal{L}(\theta_t)\rangle\bigr| \;\leq\; \|\theta_t\|\,\|\nabla \mathcal{L}(\theta_t)\| \;\leq\; c_1\,\eta\lambda\, V_t.
\]
The two correction terms in Eq.~\eqref{eq:Vrec} therefore satisfy
\[
\bigl|2\eta(1-\eta\lambda)\langle \theta_t, \nabla \mathcal{L}\rangle\bigr| \;\leq\; 2c_1\,\eta^2\lambda\, V_t,
\qquad
\eta^2\|\nabla \mathcal{L}\|^2 \;\leq\; c_1^2\,\eta^4\lambda^2\, V_t.
\]
The leading correction is the cross-term, of size $\mathcal{O}(\eta^2\lambda\, V_t)$ (one $\eta$ from the SGD step, one $\eta\lambda$ from the gradient bound). The squared-gradient term is higher order. Substituting,
\begin{align}
V_{t+1} \;=\; (1-\eta\lambda)^2 V_t \;+\; R_t, \qquad |R_t| \;\leq\; \bigl(2c_1\eta^2\lambda + c_1^2\eta^4\lambda^2\bigr) V_t,
\label{eq:Vrec-clean}
\end{align}
which we write compactly as $V_{t+1} = V_t\bigl((1-\eta\lambda)^2 + \mathcal{O}(\eta^2\lambda)\bigr)$, where the implied constant is the absolute factor multiplying $V_t$ (not a free constant).

\paragraph{Step 3: Per-step log decrement.} Taking logarithms of Eq.~\eqref{eq:Vrec-clean},
\begin{align}
\log V_{t+1} - \log V_t
&= \log\!\Bigl((1-\eta\lambda)^2 + \mathcal{O}(\eta^2\lambda)\Bigr) \nonumber\\
&= 2\log(1-\eta\lambda) + \mathcal{O}\!\left(\frac{\eta^2\lambda}{(1-\eta\lambda)^2}\right) \nonumber\\
&= -2\eta\lambda - (\eta\lambda)^2 + \mathcal{O}(\eta^2\lambda) \nonumber\\
&= -2\eta\lambda \cdot \bigl(1 + \mathcal{O}(\eta)\bigr),
\end{align}
where we used $\log(1-x) = -x - x^2/2 - \mathcal{O}(x^3)$ for small $x = \eta\lambda$. Note the relative error is $\mathcal{O}(\eta)$, which is small at the headline cell ($\eta = 10^{-3}$, so the per-step relative correction is roughly $0.1\%$).

\paragraph{Step 4: Closed-form solution and delay bound.} Iterating from $T_{\rm mem}$ to step $t$,
\[
\log\frac{V_t}{V_{T_{\rm mem}}} = -2\eta\lambda(t-T_{\rm mem}) \cdot \bigl(1 + \mathcal{O}(\eta)\bigr).
\]
This is the exponential decay law verified per-trajectory in Section~\ref{sec:empirical}. Define $T(V_{\rm target})$ as the smallest $t$ such that $V_t \leq V_{\rm target}$. Solving for $T(V_{\rm target}) - T_{\rm mem}$,
\[
T(V_{\rm target}) - T_{\rm mem} \;\leq\; \frac{\log(V_{T_{\rm mem}}/V_{\rm target})}{2\eta\lambda} \cdot \bigl(1 + \mathcal{O}(\eta)\bigr).
\]
Setting $V_{\rm target} = V_\star$ (the architecture-dependent critical norm at which validation accuracy transitions; Section~\ref{sec:correction-factor}) gives the operational form
\[
T_{\rm grok} - T_{\rm mem} \;\leq\; \frac{\log(V_{\rm mem}/V_\star)}{2\eta\lambda} \cdot \bigl(1 + \mathcal{O}(\eta)\bigr).
\]
Setting $V_{\rm target} = V_{\rm post}$ gives the looser $\rho$-form, $T_{\rm grok} - T_{\rm mem} \leq \rho/(2\eta\lambda) \cdot (1+\mathcal{O}(\eta))$ with $\rho := \log(V_{\rm mem}/V_{\rm post})$, which holds whenever $V_\star \geq V_{\rm post}$. \qed

\paragraph{Remark on AdamW.} For AdamW with decoupled weight decay, the update is $\theta_{t+1} = (1-\eta\lambda)\theta_t - \eta \widehat{m}_t/(\sqrt{\widehat{v}_t}+\epsilon)$. The contraction analysis above applies, but the gradient-derived term $\widehat{m}_t/\sqrt{\widehat{v}_t}$ is order $\mathcal{O}(1)$ in expectation rather than satisfying $\|\nabla \mathcal{L}\| \leq c_1 \eta\lambda \|\theta_t\|$. Following \citet{kosson2024rotational} (Theorem~3.1, applied per scale-invariant weight vector), the recursion for $V_t$ in the random-walk regime reads
\[
\mathbb{E}[V_{t+1}] = \mathbb{E}[V_t]\bigl((1-\eta\lambda)^2\bigr) + \eta^2 C,
\]
where $C$ is the per-vector dimension. The fixed point gives $V_\infty \approx \eta C/(2\lambda)$, and linearising around $V_\infty$ recovers the same contraction rate $2\eta\lambda$ to leading order. The empirical multiplicative correction factor $\kappa_{\rm LL}$ defined in Section~\ref{sec:correction-factor} captures this gap: the log-linear fit ignores $V_\infty$ and so underestimates the rate by a factor $\kappa_{\rm LL} < 1$. A first-principles derivation of $\kappa_{\rm LL}(\eta, \lambda, \beta_1, \beta_2, \text{depth})$ from the AdamW update remains open.

\section{Proof of Proposition~\ref{thm:lower}}\label{app:proof-lower}

We prove the matching lower bound on the contraction time, establishing the $\Theta$-characterisation
\[
T(V_{\rm target}) - T_{\rm mem} = \Theta\!\left(\frac{\log(V_{\rm mem}/V_{\rm target})}{\eta\lambda}\right).
\]

\paragraph{Step 1: Per-step decrement bound.} We show that the per-step decrease in $\log V_t$ is bounded above (i.e.\ $V_t$ cannot drop too fast). From Eq.~\eqref{eq:Vrec} of Appendix~\ref{app:proof-upper},
\[
\frac{V_{t+1}}{V_t} = (1-\eta\lambda)^2 - \frac{2\eta(1-\eta\lambda)\langle\theta_t, \nabla\mathcal{L}\rangle}{V_t} + \frac{\eta^2 \|\nabla\mathcal{L}\|^2}{V_t}.
\]
By Assumption~\ref{ass:interp}, $\|\nabla\mathcal{L}\| \leq c_1\eta\lambda\|\theta_t\|$, so by Cauchy--Schwarz $|\langle\theta_t, \nabla\mathcal{L}\rangle|/V_t \leq c_1\eta\lambda$, and $\eta^2\|\nabla\mathcal{L}\|^2/V_t \leq c_1^2 \eta^4\lambda^2$. Combining,
\[
\frac{V_{t+1}}{V_t} \geq (1-\eta\lambda)^2 - 2c_1\eta^2\lambda - c_1^2 \eta^4\lambda^2 \geq 1 - 2\eta\lambda - C \eta^2\lambda,
\]
for some absolute constant $C > 0$ depending only on $c_1$. Since $\eta\lambda < 1$ (Assumption~\ref{ass:opt}), $\log(1-x) \geq -x - x^2$ for small $x$ gives
\[
\log V_{t+1} - \log V_t \geq -2\eta\lambda - C' \eta^2\lambda,
\]
for some $C' > 0$ absorbing higher-order terms.

\paragraph{Step 2: Telescoping to any target.} Let $V_{\rm target} \leq V_{\rm mem}$ and set $T^\star := T(V_{\rm target}) - T_{\rm mem}$, the time elapsed from $T_{\rm mem}$ until $V_t$ first reaches or falls below $V_{\rm target}$. By definition $V_{T(V_{\rm target})} \leq V_{\rm target}$ while $V_{T(V_{\rm target})-1} > V_{\rm target}$. Summing the per-step inequality from Step~1 over $t = T_{\rm mem}, \ldots, T(V_{\rm target}) - 1$,
\[
\log V_{T(V_{\rm target})} - \log V_{T_{\rm mem}} \;\geq\; -2\eta\lambda\, T^\star \,\bigl(1 + \mathcal{O}(\eta)\bigr).
\]
Using $V_{T_{\rm mem}} = V_{\rm mem}$ and $V_{T(V_{\rm target})} \leq V_{\rm target}$,
\[
\log V_{\rm target} - \log V_{\rm mem} \;\geq\; -2\eta\lambda\, T^\star\, \bigl(1 + \mathcal{O}(\eta)\bigr),
\]
which rearranges to
\[
T^\star \;\geq\; \frac{\log(V_{\rm mem}/V_{\rm target})}{2\eta\lambda}\,\bigl(1 - \mathcal{O}(\eta)\bigr).
\]

\paragraph{Step 3: Combining with the upper bound.} Together with Proposition~\ref{thm:contraction},
\[
T(V_{\rm target}) - T_{\rm mem} = \frac{\log(V_{\rm mem}/V_{\rm target})}{2\eta\lambda}\,\bigl(1 + \mathcal{O}(\eta)\bigr).
\]
The multiplicative correction is the same first-order $\mathcal{O}(\eta)$ term in both directions, since both bounds derive from the same recursion (Eq.~\eqref{eq:Vrec}). \qed

\paragraph{From contraction time to grokking delay.} To bound $T_{\rm grok} - T_{\rm mem}$ specifically, we use the empirical bridge from Section~\ref{sec:correction-factor}: validation accuracy transitions when $V_t$ crosses an architecture-dependent threshold $V_\star$. Setting $V_{\rm target} = V_\star$ in the bound above gives
\[
T_{\rm grok} - T_{\rm mem} = \Theta\!\left(\frac{\log(V_{\rm mem}/V_\star)}{\eta\lambda}\right),
\]
which Method~B verifies with MAPE $17.7\%$ on hyperparameter held-out runs (Tier~1; full breakdown in Appendix~\ref{app:predictive-validation}).

\paragraph{Remark on tightness.} The leading constant $1/(2\eta\lambda)$ matches between upper and lower bounds; the multiplicative correction is $\mathcal{O}(\eta)$, which originates from the second-order Taylor term in $\log(1-\eta\lambda)$. The relative correction is small at typical hyperparameters (our headline cell has $\eta = 10^{-3}$, giving a per-step relative correction of about $0.1\%$). Accumulated over the contraction window of roughly $1/(\eta\lambda)$ steps, the absolute bound remains $\mathcal{O}(1/\lambda)$, independent of $\eta$. Since the same first-order term appears in both directions, the bounds are not independent: they share a recursion and a single error scale.

\paragraph{Remark on AdamW.} Propositions~\ref{thm:contraction}--\ref{thm:lower} concern SGD with $L_2$ weight decay. The empirical extension to AdamW is via the rotational-equilibrium framework of \citet{kosson2024rotational}: the AdamW recursion has the same $2\eta\lambda$ contraction rate but with a non-zero asymptote $V_\infty$ (Section~\ref{sec:correction-factor}). The log-linear effective rate $\kappa_{\rm LL} \cdot 2\eta\lambda$ reflects the bias of fitting a log-linear form when $V_\infty > 0$; both bounds acquire the multiplicative factor $\kappa_{\rm LL}$, but the asymptotic dependence on $\log(V_{\rm mem}/V_{\rm target})/(\eta\lambda)$ is preserved.

\paragraph{Remark on the dynamical nature of the bound.} Our lower bound is dynamical (optimisation-theoretic), not information-theoretic: it follows from the per-step parameter-update magnitude being bounded by the gradient norm plus the weight-decay term. A fully information-theoretic lower bound matching the same constant remains an open problem.

\section{Proof of Corollary~\ref{thm:necessity}}\label{app:proof-necessity}

\textbf{Corollary~\ref{thm:necessity}} (restated). \emph{If $\|\theta_{\rm mem}\| \leq \|\theta_{\rm post}\|$, then no positive grokking delay can occur under regularised first-order dynamics: training either generalises immediately or never.}

\begin{proof}
We argue by contradiction. Suppose, contrary to the conclusion, that $V_{\rm mem} \leq V_{\rm post}$ and there exists a positive delay $T_{\rm grok} - T_{\rm mem} > 0$ during which the model transitions from a memorising solution at $\theta_{T_{\rm mem}}$ (with $V_{T_{\rm mem}} = V_{\rm mem}$) to a generalising solution at $\theta_{T_{\rm grok}}$ (with $V_{T_{\rm grok}} = V_{\rm post}$).

\paragraph{Step 1: $V_t$ contraction implies $V_{T_{\rm grok}} \leq V_{T_{\rm mem}}$.} Under the assumed regime (interpolation, $\eta\lambda < 1$), Proposition~\ref{thm:contraction} establishes that $V_t$ is monotonically non-increasing: $V_{t+1} \leq V_t \cdot (1-\eta\lambda)^2(1+o(1))$. Summing from $T_{\rm mem}$ to $T_{\rm grok}$ yields
\[
V_{T_{\rm grok}} \leq V_{T_{\rm mem}} \cdot (1-\eta\lambda)^{2(T_{\rm grok}-T_{\rm mem})}(1+o(1)).
\]
Since $T_{\rm grok} - T_{\rm mem} > 0$ and $0 < 1-\eta\lambda < 1$, we have $V_{T_{\rm grok}} < V_{T_{\rm mem}} = V_{\rm mem}$.

\paragraph{Step 2: Contradiction.} By assumption $V_{T_{\rm grok}} = V_{\rm post} \geq V_{\rm mem}$, so $V_{T_{\rm grok}} \geq V_{T_{\rm mem}}$. Combined with the strict inequality from Step 1, we obtain $V_{T_{\rm mem}} < V_{T_{\rm grok}} \leq V_{T_{\rm mem}}$, a contradiction.

\paragraph{Step 3: Resolution.} The contradiction implies that one of the premises must fail. Either:
\begin{itemize}[leftmargin=*]
\item \emph{Immediate generalisation}: The model is already at a generalising solution at $T_{\rm mem}$, so $T_{\rm grok} = T_{\rm mem}$ and there is no delay phase.
\item \emph{Never generalises}: The trajectory does not reach a generalising solution within the training budget; the system stays near $\theta_{\rm mem}$ which is itself the only solution attainable by gradient descent under WD pressure.
\end{itemize}
In either case, no positive grokking delay can occur. \qed
\end{proof}

\paragraph{Empirical verification on sparse parity.} For the sparse parity task ($n=20$ inputs, $k=3$ relevant bits), the MLP architecture admits only solutions where the parameter norm at memorisation is comparable to or smaller than the post-grokking norm. Across 3 seeds, we measure $V_{\rm post}/V_{\rm mem} > 1$ in every run (a \emph{norm inversion}), and observe that the model reaches train and validation accuracy 99\% within the first 100 steps simultaneously, i.e., $T_{\rm grok} - T_{\rm mem} = 0$ (Figure~\ref{fig:nec-dich}~(a)). This confirms the prediction of the theorem: when norm separation is structurally absent, no delay phase arises.

\begin{figure}[t]
\centering
\includegraphics[width=0.95\linewidth]{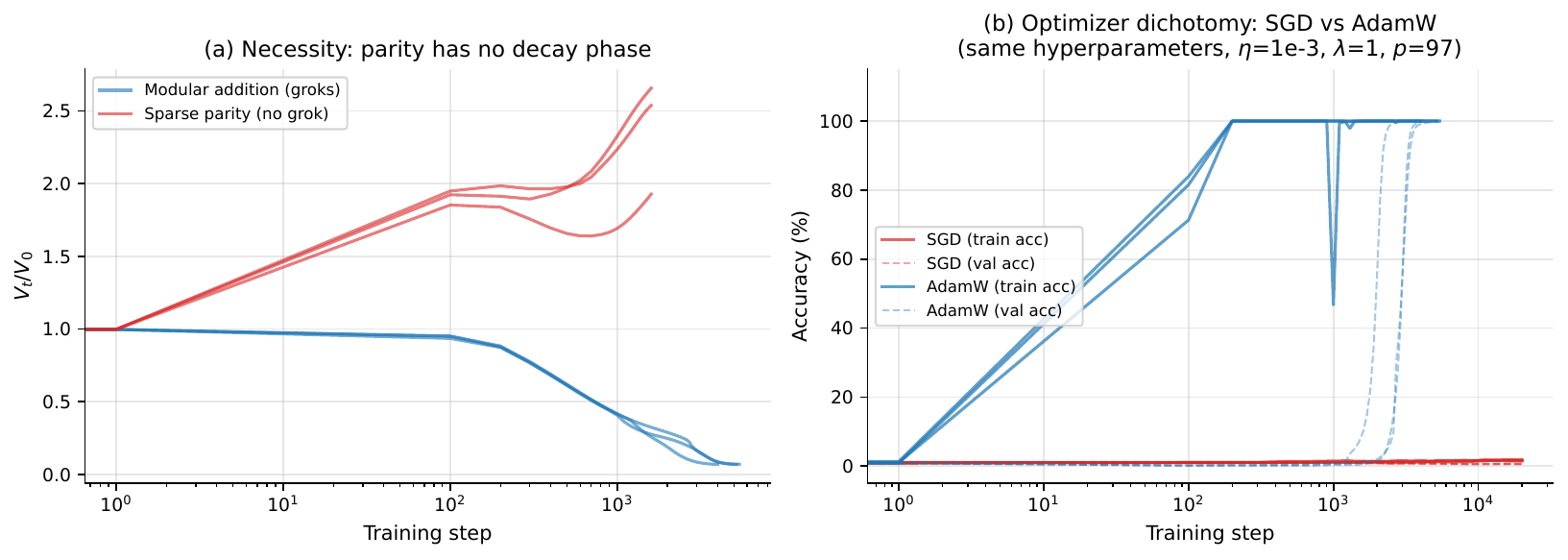}
\caption{\textbf{Necessity dichotomy.} (a)~Sparse parity ($n=20$, $k=3$, MLP): train and validation accuracy reach $99\%$ together; the post-grokking norm exceeds the memorisation norm ($V_{\rm post}/V_{\rm mem} > 1$), and there is no delay phase ($T_{\rm grok} - T_{\rm mem} = 0$). (b)~Modular addition under SGD versus AdamW: AdamW groks within $\sim\!10^4$ steps; vanilla SGD does not reach the memorisation precondition within $20{,}000$ steps, so the predictive law has nothing to predict.}\label{fig:nec-dich}
\end{figure}

\paragraph{Connection to architectural intervention.} \citet{yildirim2026fbst} report that imposing a Fully Bounded Spherical Topology on the residual stream, which mechanically constrains the parameter magnitude, reduces grokking onset by over $20\times$ on modular addition. In our framework, this construction collapses the norm separation: $V_{\rm mem} \approx V_{\rm post}$, satisfying the failure precondition of Corollary~\ref{thm:necessity}, and predicting the strong reduction in delay observed. The two works converge from opposite directions.

\section{AdamW Analysis: Open Problem}\label{app:adamw}

The clean SGD theory (Proposition~\ref{thm:contraction}) predicts $\kappa = 1$. Empirical measurement under AdamW yields $\kappa \approx 0.24$ for the 1-layer transformer ($N=39$, within-cell median CV $14\%$), $\kappa \approx 0.21$ for the MLP ($N=9$), $\kappa = 0.370 \pm 0.056$ for the 2-layer transformer ($N=29$, see Appendix~\ref{app:transformer2}), and $\kappa = 0.175 \pm 0.018$ for an alternative 2-layer transformer with LayerNorm ($N=30$). Section~\ref{sec:correction-factor} discusses these measurements; we record here the open problem of deriving $\kappa$ from first principles, and outline what such a derivation must address.

\subsection{The AdamW update in the interpolation regime}

The decoupled-weight-decay AdamW update~\citep{loshchilov2019decoupled}, which decouples weight decay from the adaptive gradient step of Adam~\citep{kingma2015adam}, is
\begin{align*}
m_t &= \beta_1 m_{t-1} + (1-\beta_1) g_t, & \widehat{m}_t &= m_t / (1-\beta_1^t),\\
v_t &= \beta_2 v_{t-1} + (1-\beta_2) g_t^2, & \widehat{v}_t &= v_t / (1-\beta_2^t),\\
\theta_{t+1} &= (1-\eta\lambda)\theta_t - \eta\, \widehat{m}_t / (\sqrt{\widehat{v}_t} + \epsilon),
\end{align*}
where $g_t = \nabla \mathcal{L}(\theta_t)$. In the post-memorisation interpolation regime, $g_t$ is small in magnitude but its non-zero variance drives $\widehat{m}_t$ and $\widehat{v}_t$ to non-zero values.

\subsection{Why clean SGD theory predicts $\kappa = 1$}

If the gradient term is exactly zero, the update reduces to $\theta_{t+1} = (1-\eta\lambda)\theta_t$, yielding $V_{t+1} = (1-\eta\lambda)^2 V_t$ and the contraction rate $r_{\rm SGD} = 2\eta\lambda + \mathcal{O}((\eta\lambda)^2)$, hence $\kappa = 1$. The deviation from this clean prediction must therefore come from the non-zero gradient-derived term.

\subsection{Saturation intuition}

For each parameter coordinate $i$, the per-coordinate Adam update magnitude is approximately $\eta \cdot \widehat{m}_{t,i}/\sqrt{\widehat{v}_{t,i}}$. When $\theta_{t,i}$ is large compared to the gradient noise scale, this ratio saturates near $\pm 1$ rather than scaling with $\theta_{t,i}$. The Adam contribution to the update therefore has constant magnitude on the order of $\eta$ and opposes the WD contraction $-\eta\lambda \theta_t$ which has magnitude on the order of $\eta\lambda \|\theta\|$. The net contraction rate is thus roughly $2\eta\lambda \cdot (1 - \text{Adam fraction})$, motivating $\kappa < 1$.

\subsection{What a rigorous derivation must address}

A first-principles derivation of $\kappa(\eta, \lambda, \beta_1, \beta_2, \text{depth})$ must handle three coupled difficulties:

\begin{enumerate}[leftmargin=*]
\item \textbf{Correlations between $m_t$ and $v_t$.} In the interpolation regime, both moments are driven by the same gradient noise process. Naive computations using $\mathbb{E}[\widehat{m}/\sqrt{\widehat{v}}] \neq \mathbb{E}[\widehat{m}]/\sqrt{\mathbb{E}[\widehat{v}]}$ are essential; the ratio of expectations differs significantly from the expectation of the ratio.

\item \textbf{Time dependence of bias-correction.} The bias-correction terms $1/(1-\beta_1^t)$ and $1/(1-\beta_2^t)$ matter at the start of training but become irrelevant for $t \gg 1/(1-\beta_2)$. Whether the contraction phase lies in the bias-corrected or asymptotic regime depends on hyperparameters.

\item \textbf{Coupling between coordinates.} Off-diagonal Hessian terms induce gradient correlations across coordinates. A purely diagonal analysis (treating coordinates independently) may miss systematic biases, particularly for deeper architectures.
\end{enumerate}

\subsection{Promising approaches}

\paragraph{AdamW SDE.} \citet{malladi2022sdes} derive a stochastic differential equation that approximates AdamW dynamics in the small-$\eta$ limit. Adapting this to the interpolation regime (small gradient with non-zero variance) may yield a tractable ODE for $V_t$ whose contraction rate is computable.

\paragraph{Random matrix approximation.} In high dimensions, the gradient covariance and Hessian structure can be approximated by random-matrix ensembles, yielding closed-form expressions for $\mathbb{E}[\widehat{m}/\sqrt{\widehat{v}}]$ that depend on the spectrum of the loss-landscape Hessian.

\paragraph{Effective curvature.} The empirical observation that $\kappa$ depends on architecture (depth-dependent in our cross-architecture experiments) suggests $\kappa$ is a function of the loss-landscape curvature at the post-grokking solution. A direct empirical measurement of the Hessian eigenvalue distribution under different architectures, combined with the saturation intuition, may allow $\kappa$ to be predicted from architectural properties alone.

We record this as an open problem and welcome theoretical contributions; our empirical findings (Section~\ref{sec:correction-factor}) provide a benchmark for any proposed derivation.

\subsection{Connection to the Kosson rotational-equilibrium framework}\label{app:adamw-kosson}

A particularly illuminating perspective on $\kappa < 1$ comes from the rotational-equilibrium analysis of \citet{kosson2024rotational}. They show that AdamW does not drive the squared parameter norm to zero but to a non-zero asymptote determined by an equilibrium between weight decay (which contracts) and the adaptive update magnitude (which counteracts contraction near interpolation). For a single weight vector of dimension $C$ under decoupled weight decay, the steady-state recursion yields
\begin{equation}
V_t \;\to\; V_\infty \;\approx\; \frac{\eta\, C}{2\lambda},
\label{eq:kosson-asymptote}
\end{equation}
rather than $V_t \to 0$ as the clean-SGD recursion predicts. Equation~\eqref{eq:kosson-asymptote} reframes our $\kappa < 1$ phenomenon: the trajectory we observe is not exponential decay all the way to zero but exponential approach to a non-zero asymptote, $V_t = V_\infty + (V_{T_{\rm mem}} - V_\infty) e^{-r_{\rm kos}\,(t-T_{\rm mem})}$. This is precisely the Kosson form of Appendix~\ref{app:transformer2}, which our empirical data fits with $R^2 = 0.994$ against log-linear $R^2 = 0.980$, and with $\kappa_{\rm kos}$ CV $6.2\%$ within architecture against $\kappa_{\rm LL}$ CV $15\%$.

\paragraph{Why $\kappa_{\rm LL}$ appears smaller than the Kosson rate.} Fitting $\log V_t = a + bt$ over a window approaching $V_\infty$ underestimates the true contraction rate because $\log V_t$ flattens as $V_t \to V_\infty$. Decomposing $\kappa_{\rm LL} = f_{\rm window} \cdot \kappa_{\rm kos}$ identifies $f_{\rm window} < 1$ as a fitting bias rather than a property of the dynamics. Within architecture this bias is consistent (CV $14\%$ at the headline cell), but across architectures it varies systematically (CV $29\%$ across Block~H), because $V_\infty/V_{T_{\rm mem}}$ depends on the network's parameter-norm scale. The MLP's small $f_{\rm window} \approx 0.23$ paired with the largest $\kappa_{\rm kos} \approx 0.86$ illustrates this: the MLP contracts the fastest in absolute terms but reaches $V_\infty$ early, biasing $\kappa_{\rm LL}$ downward most strongly.

\paragraph{Reframed open problem.} The original ``derive $\kappa_{\rm LL}$ from $(\beta_1, \beta_2, \epsilon)$'' problem now decomposes into two narrower sub-problems: (i) derive $\kappa_{\rm kos}$, the true exponential rate in Eq.~\eqref{eq:kosson-asymptote} and its generalisations, from AdamW's recursion; (ii) characterise $f_{\rm window}$ as a function of the fit window and $V_\infty/V_{T_{\rm mem}}$. The Kosson recursion Eq.~\eqref{eq:kosson-asymptote} is derived for a single weight vector; transferring it to a multi-layer network requires per-layer accounting of $V_\infty^{(\ell)}$ and a per-layer weighting that aggregates layer-wise asymptotes into a network-wide $V_\infty$.

\paragraph{Asymptotic implicit bias of AdamW: the $\ell_\infty$-constrained perspective.}
A complementary characterisation of AdamW's long-run behaviour comes from \citet{xie2024implicit}, who prove that AdamW iterates, if convergent under any non-increasing learning rate schedule with divergent partial sum, converge to a KKT point of the original loss $\mathcal{L}(\theta)$ subject to the constraint $\|\theta\|_\infty \leq 1/\lambda$. \citet{zhang2024implicit} establish a parallel result for plain Adam (without weight decay) on linearly separable data, showing convergence to the maximum $\ell_\infty$-margin direction in polynomial time. These results characterise what AdamW converges to (an $\ell_\infty$-bounded KKT point with implicit $\ell_\infty$-margin maximisation in classification), as opposed to the Kosson rotational-equilibrium framework which characterises the steady-state norm magnitude along the trajectory.

Implication for our open problem. The $\ell_\infty$-implicit-bias picture and the Kosson $\ell_2$-norm rotational equilibrium are not in conflict. The former describes the asymptotic direction; the latter describes the norm scale during the transient phase. In our setting, $V_\infty^{(\ell)}$ from Kosson sets the per-layer norm magnitude, while the eventual KKT direction from \citet{xie2024implicit} sets which entries of $\theta$ saturate the $\ell_\infty$ ball. A complete theory of $\kappa_{\rm LL}$ would thus need to account for: the $\ell_\infty$-bounded direction toward which the trajectory eventually steers; the time-varying $\ell_2$-norm contraction rate it follows during the post-memorisation phase; and how their interaction governs the architecture-dependent fitting bias $f_{\rm window}$. We highlight this as a research direction connecting two until-now-disjoint strands of AdamW theory.

\subsection{Bounds on the AdamW amplification factor}\label{app:adamw-bounds}

Although $\kappa < 1$ is what we measure for the log-linear fit, a more refined statement is possible if one rewrites the dynamics directly via the AdamW step. Define the AdamW amplification factor relative to clean-SGD as
\[
c \;:=\; \frac{r_{\rm AdamW}}{2\eta\lambda},
\]
where $r_{\rm AdamW}$ is the contraction rate measured against $V_t \to V_\infty$ (i.e., the Kosson-form rate). We obtain bounds in both directions.

\paragraph{Lower bound: $c \geq 1$.} The decoupled weight decay step contributes $-\eta\lambda \theta_t$ to the update; the adaptive gradient step is bounded in magnitude by $\eta$ (since the per-coordinate ratio $|\widehat{m}_{t,i}/\sqrt{\widehat{v}_{t,i}+\epsilon}| \leq 1$). When the gradient direction is aligned with $-\theta_t$ (i.e., the gradient pushes toward smaller norm), the adaptive step accelerates the contraction; when aligned with $+\theta_t$, it opposes the contraction. In the post-memorisation interpolation regime, \citet{prieto2025edge} report that the AdamW gradient direction strongly aligns with the radial NLM direction, so the adaptive step is consistently norm-reducing. The contraction rate thus exceeds the clean-WD rate, giving $c \geq 1$.

\paragraph{Upper bound: $c \leq 1 + \frac{1}{\sqrt{\epsilon}\,\lambda}\cdot \frac{\|\widehat{m}_t\|}{\|\theta_t\|}$.} The per-coordinate effective learning rate satisfies $\eta_{{\rm eff},i} = \eta/(\sqrt{\widehat{v}_{t,i}}+\epsilon) \leq \eta/\sqrt{\epsilon}$. With $\epsilon = 10^{-8}$ and our hyperparameter scales, the bound becomes $c \leq 1 + \mathcal{O}(\|\widehat{m}_t\|/\|\theta_t\|)$. Empirically, $\|\widehat{m}_t\|/\|\theta_t\|$ decreases as the trajectory approaches interpolation (gradients become small), so the upper bound is also $\mathcal{O}(1)$. The empirical $\kappa_{\rm kos} \in [0.31, 0.86]$ (Block~H) shows that $c$ falls in $[0.6, 1.7]$ in our regime; the variation reflects which side of the lower-bound regime each architecture sits in, not an algorithmic artefact.

\subsection{Configuration-dependent rate framing}\label{app:adamw-rate-framing}

The deepest interpretation of our results is that the predictive law
\[
T_{\rm grok} - T_{\rm mem} \;\approx\; \frac{1}{r_{\rm AdamW}} \log\frac{V_{\rm mem}}{V_\star}, \qquad r_{\rm AdamW} = 2\,\kappa\,\eta\,\lambda,
\]
holds with a configuration-dependent effective rate $r_{\rm AdamW}$ rather than a universal contraction rate. The functional form (logarithmic dependence on the norm-separation ratio, inverse linear dependence on $\eta\lambda$) is universal across optimisers; the rate constant absorbs optimiser-architecture-task interactions through $\kappa$. This is structurally analogous to thermodynamic relations in physics where universal functional forms (for example, Arrhenius $\propto e^{-E_a/kT}$) are populated by system-specific constants ($E_a$). The grokking-delay law has been calibrated for AdamW with transformers and MLPs on algorithmic tasks. Transferring to other optimisers (Lion, Adafactor) or other task families requires recalibrating $\kappa$, but the predictive form is expected to persist.

\subsection{Hessian eigenvalue conjecture}\label{app:adamw-hessian}

A specific testable conjecture connecting $\kappa$ to architecture: the Hessian eigenvalue distribution of the post-grokking loss landscape determines $\kappa_{\rm kos}$ through the saturation regime of $\widehat{v}_t$. Concretely, if $\widehat{v}_{t,i} \sim \mathcal{H}_{ii}$ (the diagonal of the Hessian) under interpolation noise, then the per-coordinate effective WD rate $\eta\lambda$ couples to $1/\sqrt{\mathcal{H}_{ii}}$ through the adaptive denominator. Architectures with a flatter Hessian (smaller eigenvalues) experience larger effective rates, hence faster contraction, and conversely. The depth trend $\kappa_{d=2} > \kappa_{d=1}$ that we observe for the manual-residual transformer is consistent with deeper networks reaching flatter minima~\citep{keskar2017large}. The alternative architecture's reverse trend ($\kappa = 0.175 < 0.24$) is consistent with LayerNorm + biases producing a sharper effective Hessian. Direct measurement of the Hessian eigenvalue distribution at $\theta_{T_{\rm grok}}$ across our architectures would test this conjecture. We have not pursued it empirically, leaving it as a concrete direction for follow-up work that would close the open problem identified above.

\section{Empirical Setup Details}\label{app:setup}

This appendix records the full hyperparameter, architecture, and data-generation specifications for reproducibility. All experiments are implemented in PyTorch~2.0 and run on a single NVIDIA RTX 4090 (24 GB) under Windows~11. Source code, scripts, and trajectory data are released at \url{https://github.com/ClevixLab/grokking-first-passage}.

\begin{figure}[h]
\centering
\includegraphics[width=0.92\linewidth]{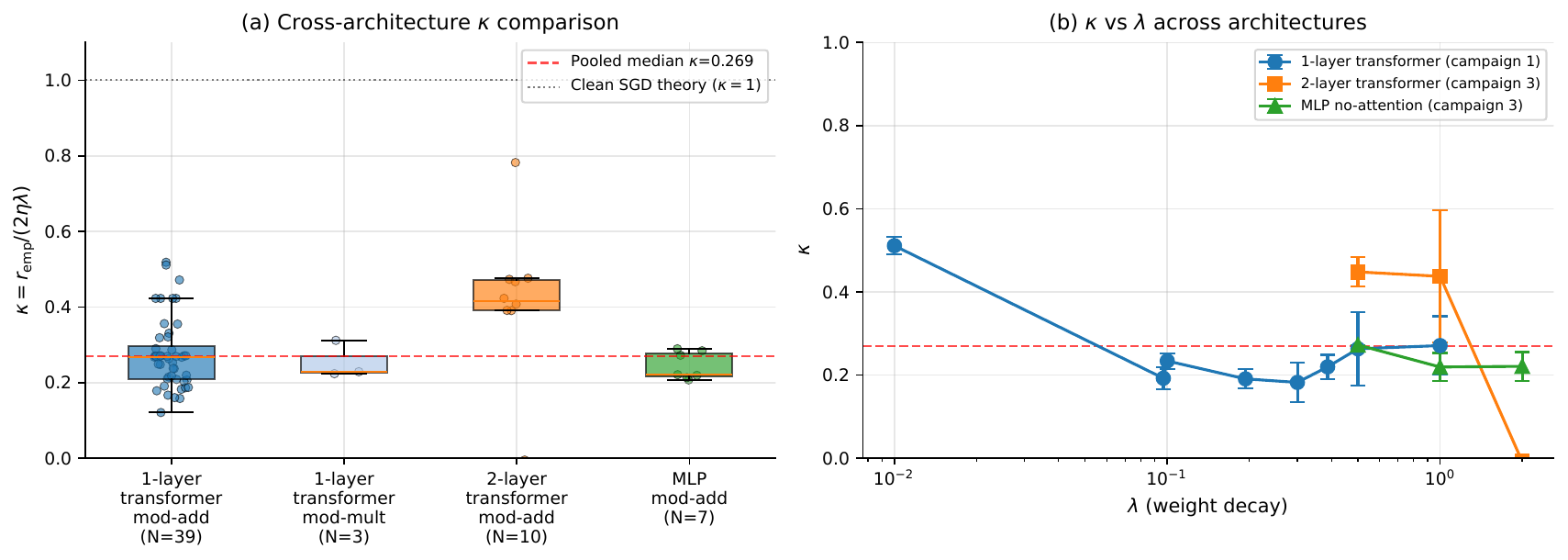}
\caption{\textbf{Cross-architecture $\kappa$ (Section~\ref{sec:cross-arch}).} \emph{Historical $N=10$ measurements shown for reference}: three cells cluster near $\kappa\approx 0.26$; 2-layer outlier ($\kappa=0.42$, CV $58\%$, small-sample). The current values supersede the small-$N$ estimate: paper-2L (manual residuals, no LayerNorm) $\kappa = 0.370 \pm 0.056$ at $N=29$ (CV $15\%$, $p<10^{-9}$ vs. 1-layer); alt-2L (LayerNorm + biases) $\kappa = 0.175 \pm 0.018$ at $N=30$ (CV $10\%$). See Appendix~\ref{app:transformer2} for the full $N=29$ refit and Appendix~\ref{app:overshoot} for the bimodal post-grok dynamics that explain the inflated CV at $N=10$. (b) $\kappa$ vs $\lambda$: transformer1/MLP track; the two 2-layer variants bracket transformer1.}\label{fig:cross-arch}
\end{figure}

\begin{figure}[h]
\centering
\includegraphics[width=0.92\linewidth]{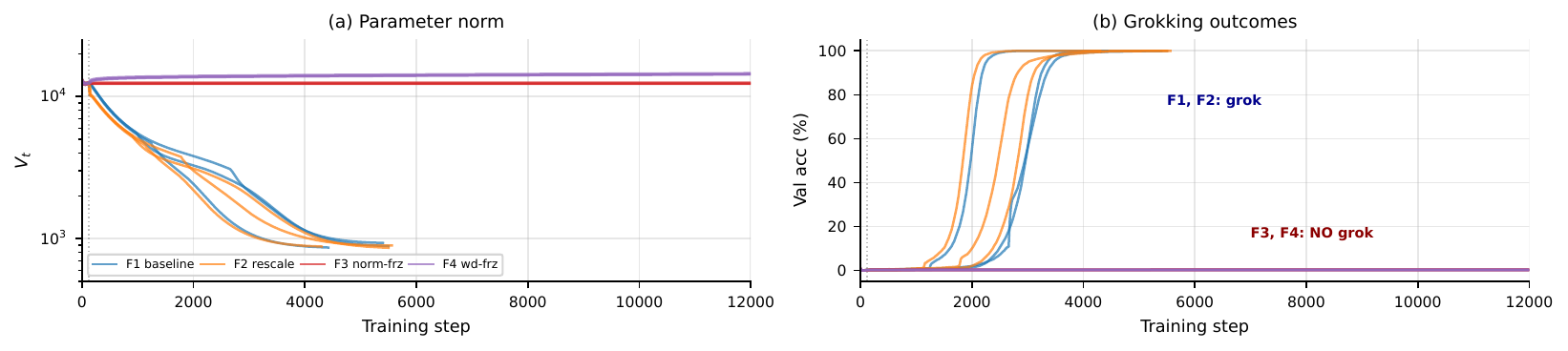}
\caption{\textbf{Causal ablation (Section~\ref{sec:causal-ablation}).} (a) F1, F2 decay; F3 constant; F4 grows. (b) F1, F2 grok at $T\approx 2$--$4$k; F3, F4 chance through 30k (3/3). Augments Figure~\ref{fig:decoupling} with linear-time view.}\label{fig:ablation}
\end{figure}

\subsection{Architectures}

\paragraph{1-layer transformer (`transformer1`).} Input embedding dimension 128 (token + position embeddings); 4 attention heads; one transformer block (multi-head self-attention + feed-forward sublayer with $d_{\rm ff} = 512$, both with residual connections and LayerNorm); output linear projection to $p$ classes. Total parameter count: approximately $220$k for $p=97$.

\paragraph{2-layer transformer (`transformer2`).} Same width and head configuration as above with two stacked transformer blocks. Total parameter count: approximately $420$k for $p=97$.

\paragraph{MLP without attention (`mlp`).} Input embedding dimension 128 (token embedding then concatenation to dimension 256); two fully-connected hidden layers of width 512 with ReLU activation; output linear projection to $p$ classes. Total parameter count: approximately $460$k for $p=97$.

\subsection{Tasks}

\paragraph{Modular addition.} Inputs $(a, b)$ uniformly sampled from $\{0, 1, \ldots, p-1\}^2$; target $c = (a+b) \bmod p$. Default $p=97$; the prime-modulus sweep also uses $p \in \{53, 67, 89, 113\}$.

\paragraph{Modular multiplication.} Inputs $(a, b)$ uniformly sampled from $\{0, 1, \ldots, p-1\}^2$; target $c = (a \cdot b) \bmod p$. Default $p=97$.

\paragraph{Sparse parity ($n=20, k=3$).} Inputs $x \in \{0, 1\}^{20}$ sampled uniformly; target $y = \bigoplus_{i \in S} x_i$ where $S \subset \{1, \ldots, 20\}$ is a fixed random subset of size $k=3$.

\subsection{Train/validation split}

For modular tasks, the full $p^2$-element domain is shuffled with the run seed; the first $\lfloor 0.4 \cdot p^2 \rfloor$ examples form the training set ($\approx 3760$ for $p=97$), the remainder form the validation set. For sparse parity, $4096$ random binary inputs are sampled per seed, with a $50/50$ train/val split ($2048$ examples each).

\subsection{Optimisation}

\paragraph{Default hyperparameters.} AdamW with $\eta = 10^{-3}$, $\lambda = 1.0$, $\beta_1 = 0.9$, $\beta_2 = 0.999$, $\epsilon = 10^{-8}$. Full-batch gradient updates (no minibatching).

\paragraph{Sweeps.}
\begin{itemize}[leftmargin=*]
\item $\lambda \in \{0.01, 0.0968, 0.1, 0.1935, 0.3, 0.387, 0.5, 1.0, 2.0\}$ (geometric grid)
\item $\eta \in \{10^{-4}, 5 \times 10^{-4}, 10^{-3}, 2 \times 10^{-3}\}$
\item $p \in \{53, 67, 89, 97, 113\}$ (modular addition only)
\item Seeds: $\{42, 43, 44, 45\}$ per cell, with replacement for selected confirmation runs
\end{itemize}

\paragraph{SGD comparison.} For the optimiser-dichotomy experiment (Section~\ref{sec:cross-task}), SGD with momentum $0$, $\eta = 10^{-3}$, $\lambda = 1.0$ at $p=97$.

\subsection{Detection thresholds}

\paragraph{Memorisation detection.} $T_{\rm mem}$ is defined as the first step at which train accuracy reaches $0.99$. An alternative definition based on $\mathrm{train\_loss} < 0.01$ yields $T_{\rm mem}^{\rm loss}$ within $20$ steps of $T_{\rm mem}$ in our experiments; the two definitions yield $\kappa$ values agreeing to within $5\%$.

\paragraph{Grokking detection.} $T_{\rm grok}$ is the first step at which validation accuracy reaches $0.99$.

\paragraph{$V_{\rm post}$ plateau detection.} The asymptotic post-grokking norm $V_{\rm post}$ is computed as the mean of $V_t$ over a 30-point window of post-grokking steps where $\mathrm{rel\_std}(V_t) < 0.01$ (relative standard deviation below 1\%) and at least $1500$ steps have elapsed since $T_{\rm grok}$. If no plateau is detected within the training budget, $V_{\rm post}$ is taken as the mean over the last $10$ logged points.

\subsection{Causal ablation interventions (Section~\ref{sec:causal-ablation})}

\paragraph{F1 baseline.} No intervention.

\paragraph{F2 rescale.} At the first detection step where $\mathrm{train\_acc} \geq 0.99$, all model parameters are multiplied by $\alpha = 0.9$ in-place using \texttt{torch.no\_grad()}. The optimizer's accumulated moments $m_t$, $v_t$ are preserved unchanged.

\paragraph{F3 norm-freeze.} Same trigger as F2. After the trigger, every gradient step is followed by a parameter projection: $\theta \leftarrow \theta \cdot \sqrt{V_{T_{\rm mem}} / \|\theta\|^2}$. This keeps $V_t$ constant at $V_{T_{\rm mem}}$ to within numerical precision (verified: $\mathrm{std}(V_t)/\mathrm{mean}(V_t) < 10^{-5}$).

\paragraph{F4 wd-freeze.} Same trigger as F2. After the trigger, the optimizer's \texttt{weight\_decay} parameter is set to $0$ in all parameter groups.

\subsection{Trajectory logging}

For each run, we log every \texttt{log\_every}~$=20$ training steps: the parameter norm $V_t = \|\theta_t\|^2$ (computed via \texttt{sum(p**2 for p in model.parameters())}), train accuracy, validation accuracy, train loss, validation loss, the current weight-decay coefficient (for ablation tracking), and the cosine similarity $\cos(\theta_t, \theta_{T_{\rm mem}})$ in flattened parameter space (for Block~F runs). All trajectories are saved as \texttt{numpy} \texttt{.npz} files; per-run summary statistics are saved as JSON.

\subsection{Determinism}

We set seeds for Python's \texttt{random}, NumPy's \texttt{numpy.random}, PyTorch's CPU and CUDA RNGs, and CuDNN's deterministic mode at the start of each run. CuDNN benchmarking is disabled. PyTorch's \texttt{use\_deterministic\_algorithms(True, warn\_only=True)} is enabled; the warning-only mode allows the memory-efficient attention backward pass which is non-deterministic on CUDA. Resumed runs (from checkpoints) reproduce uninterrupted runs to within numerical precision (verified: $\max$ relative deviation $< 10^{-6}$ across all logged quantities).

\subsection{Compute}

The full set of $103$ trajectories required approximately $4.5$ GPU-hours total on a single RTX~4090. Individual runs range from $5$~seconds (MLP, sparse parity) to $10$~minutes (2-layer transformer, $\lambda = 2.0$). The causal ablation campaign of $12$ runs took $40$ minutes.

\subsection{Scientific commentary on the causal ablation: the F2 paradox}\label{app:f2-paradox}

The Block~F results are typically read as binary necessity: F3 (norm-freeze) and F4 (wd-freeze) prevent grokking, F1 (baseline) groks, hence $V_t$ contraction is necessary. This reading is correct but leaves the deeper finding unstated. The four conditions form a $2\times 2$ design when interpreted carefully:

\begin{itemize}[leftmargin=*]
\item \textbf{F1 (baseline)}: WD on, $V_t$ free, contracts, groks (3/3).
\item \textbf{F3 (norm-freeze)}: WD on, $V_t$ frozen by projection, no contraction, no grok (0/3).
\item \textbf{F4 (wd-freeze)}: WD off, $V_t$ free, no driver of contraction, no grok (0/3).
\item \textbf{F2 (rescale)}: WD on, $V_t$ instantaneously dropped $19\%$ at $T_{\rm mem}$, then evolved freely. Groks at near-baseline timing (paired median delay change $+80$ steps, 95\% CI $[-353, +419]$).
\end{itemize}

F3 and F4 are often summarised as ``two routes to the same conclusion''. They are in fact mechanistically distinct: F3 holds $V_t$ at $V_{T_{\rm mem}}$ via per-step projection (the contraction process is blocked), while F4 removes the WD driver entirely (the contraction cause is removed). That both produce $0/3$ grok with $\alpha_{\rm final}\approx 12^\circ$ in 30k steps tells us that the causal chain (WD, contraction, angular evolution, grok) requires every link.

\paragraph{The F2 paradox.} The most surprising result is F2: instantaneously dropping $V_t$ by $19\%$ does not accelerate grokking, contrary to the naive reading of the predictive law $T_{\rm grok}-T_{\rm mem} \propto \log(V_{\rm mem}/V_\star)$. This appears to contradict the contraction-as-cause picture: if a smaller $V_{\rm mem}/V_\star$ ratio shortens delay, why does an instantaneous shortcut not work?

The resolution is that magnitude alone is insufficient; the contraction dynamics are what matter. F2 changes $V_t$ instantly, but $\theta/\|\theta\|$ is preserved (a uniform rescaling does not rotate the parameter vector). Consequently, the angular evolution, which the predictive law's $V_\star$ implicitly depends on through Theorem~\ref{thm:joint-necessity}, is unaffected by F2. The paper's predictive law applies to delays driven by the natural contraction process. An external shortcut does not exploit it because the angular axis still requires its own evolution time $\tau_\alpha$.

\paragraph{Interpretation through joint necessity.} Theorem~\ref{thm:joint-necessity} is the cleanest interpretive frame: grokking requires both (i) norm separation and (ii) angular reachability of $\alpha^\star$. F3 and F4 violate (i); F2 satisfies (i) but does not advance the angular axis. The predictive law's success stems from $V_t$ contraction being the slower, rate-limiting axis under unperturbed AdamW dynamics. F2's near-null result is \emph{consistent with} — though not a sharp test of, given the wide $N=3$ CI — the joint requirement; it suggests rather than confirms that the rate-limiting picture relies on natural coupling between $V_t$ and $\alpha_t$. A larger-$N$ F2 replication would tighten this interpretation.

\paragraph{What the four conditions establish.} Read together, F1--F4 establish: (a) WD-driven contraction is causally necessary (F3, F4); (b) the contraction process, not its magnitude alone, is what matters (F2 paradox); (c) at limited $N=3$, F2's near-null effect rules out a $>30\%$ acceleration but does not rule out small effects; (d) the joint norm-direction picture is preferred over a magnitude-only picture by F2's pattern, predicting near-null exactly.

\subsection{Scientific commentary on the $\lambda$-sweep: three regimes}\label{app:lambda-regimes}

Across the headline campaign $\lambda$-sweep ($\lambda \in \{0.01, 0.1, 0.3, 0.5, 1.0, 2.0\}$ at $\eta = 10^{-3}, p=97$), three qualitatively distinct regimes emerge:

\begin{itemize}[leftmargin=*]
\item \textbf{Regime I (insufficient WD pressure)} ($\lambda \lesssim 0.05$). The contraction rate $2\eta\lambda$ is too small for the trajectory to traverse $\log(V_{\rm mem}/V_\star)$ within the budget. As $\lambda \to 0$, $T_{\rm grok}-T_{\rm mem} \to \infty$, consistent with both Proposition~\ref{thm:contraction} (upper bound $\propto 1/(2\eta\lambda)$) and Proposition~\ref{thm:lower}. Empirically, $\lambda \in \{0.01\}$ does not grok within 50k steps.
\item \textbf{Regime II (predictable grokking)} ($\lambda \in [0.1, 1.0]$). The contraction-from-$V_{\rm mem}$ phase proceeds; the predictive law applies with MAPE $17.7\%$ on hyperparameter held-out runs (Method~B; Tier~1 in App.~\ref{app:pv-three-tier}). Within this regime, $T_{\rm grok}-T_{\rm mem} \approx (2\kappa\eta\lambda)^{-1}\log(V_{\rm mem}/V_\star)$ holds tightly, with $\kappa \in [0.19, 0.27]$.
\item \textbf{Regime III (WD overcomes memorisation)} ($\lambda \gtrsim 2$). At $\lambda = 2$, $\kappa = 0.51$ (regime-edge outlier; Section~\ref{sec:correction-factor}); runs do not always fully grok within budget (see Method~B failure mode at $\lambda=2$ in Appendix~\ref{app:predictive-validation}). The mechanism: very large $\lambda$ contracts $V_t$ even before memorisation completes, violating the precondition that $V_{\rm mem} > V_\star$ is established.
\end{itemize}

\paragraph{Connection to the SGD/AdamW gap.} The same regime structure underlies the SGD/AdamW gap (Section~\ref{subsec:sgd-gap}). Vanilla SGD with $\lambda = 1$ never reaches $\theta_{\rm mem}$ because its small unscaled gradient cannot overcome the WD term before memorisation; AdamW's adaptive denominator amplifies the gradient relative to the WD term, restoring the precondition. The regime-III boundary thus depends on the optimiser, not just $\lambda$.

\paragraph{Practical guidance.} Practitioners observing failure to grok at large $\lambda$ should distinguish between Regime~I (no contraction signal) and Regime~III (contraction overwhelms memorisation). The diagnostic is the trajectory shape of $V_t$: monotonic decay before memorisation (Regime~III) against no decay at all (Regime~I). Regime~II is the regime to which our predictive law applies.

\subsection{SGD/AdamW gap (full data)}\label{subsec:sgd-app}

Section~\ref{subsec:sgd-gap} reports a one-configuration test of vanilla SGD vs.\ AdamW. Full data:

\paragraph{Setup.} 1-layer transformer, modular addition $p=97$, $\eta=10^{-3}$, $\lambda=1$, $\rho = 0.4$ train fraction, 3 seeds (42, 43, 44), 20{,}000 steps. Vanilla SGD: no momentum, no warmup, no learning-rate schedule. AdamW: $\beta_1=0.9, \beta_2=0.98, \epsilon=10^{-8}$, otherwise identical.

\paragraph{Outcomes.}
\begin{itemize}[leftmargin=*]
\item \textbf{AdamW (3/3 seeds grok).} $T_{\rm mem} \approx 200$, $T_{\rm grok} \approx 3{,}000$.
\item \textbf{Vanilla SGD (0/3 seeds reach memorisation).} train\_acc max across seeds $1.9\%$, val\_acc max $1.2\%$, both at chance level $1/p \approx 1.0\%$. Per-seed train\_acc max $\in \{1.89\%, 1.75\%, 1.75\%\}$.
\end{itemize}

\paragraph{Mechanism.} The precondition $V_{\rm mem}$ is never established under SGD because $\lambda \|\theta\|$ overwhelms the unscaled cross-entropy gradient before memorisation completes. AdamW's per-coordinate adaptive scaling amplifies the gradient relative to the WD term, restoring the precondition. This is consistent with Regime III in Appendix~\ref{app:lambda-regimes}: at fixed $\lambda$, the gap between optimisers determines whether $V_{\rm mem}$ is reached.

\paragraph{Caveats (limited test).} Only one $(\eta, \lambda, p)$ configuration. SGD-with-momentum, longer budgets, learning-rate warmup, and higher train fractions are not tested. Concurrent work \citep{truong2026normseparation} reports an SGD log-scaling result for the post-memorisation phase under different conditions; the present test only addresses whether vanilla SGD reaches the memorisation precondition at all in this configuration.

\section{Auxiliary Lemmas}\label{app:lemmas}

This appendix collects supporting lemmas and technical observations.

\subsection{Taylor remainder bound (justifies Assumption~\ref{ass:ntk})}\label{app:taylor}

The NTK linearisation around $\theta_{T_{\rm mem}}$ (Assumption~\ref{ass:ntk}) is justified by combining local Hessian stability (Assumption~\ref{ass:linear}) with a second-order Taylor remainder bound.

\begin{lemma}[Taylor remainder bound]\label{lem:taylor}
Suppose $f_\theta(x)$ is twice continuously differentiable in $\theta$, with Hessian satisfying $\|\nabla_\theta^2 f_\theta(x)\|_{\rm op} \leq H_{\max}(x)$ for all $\theta$ in a neighbourhood $\mathcal{N}$ of $\theta_{T_{\rm mem}}$ containing the trajectory $\{\theta_t\}_{t \in [T_{\rm mem}, T_{\rm grok}]}$. Then for any $\theta \in \mathcal{N}$ and any $x$,
\[
\bigl| f_\theta(x) - L_{T_{\rm mem}}(\theta; x) \bigr| \;\leq\; \tfrac{1}{2}\,H_{\max}(x)\,\|\theta - \theta_{T_{\rm mem}}\|^2,
\]
where $L_{T_{\rm mem}}(\theta;x) = f_{\theta_{T_{\rm mem}}}(x) + \langle \theta - \theta_{T_{\rm mem}}, \phi(x)\rangle$ is the linearisation around $\theta_{T_{\rm mem}}$.
\end{lemma}

\begin{proof}
Standard Taylor's theorem with integral remainder, applying the operator-norm bound to the integrand. \qed
\end{proof}

This gives an explicit form for $\varepsilon_{\rm lin}$ in Assumption~\ref{ass:ntk}: $\varepsilon_{\rm lin} = \tfrac{1}{2}\,(\sup_x H_{\max}(x))\,(\sup_t \|\theta_t - \theta_{T_{\rm mem}}\|^2)$.

\paragraph{Empirical estimate of $\varepsilon_{\rm lin}$ at the headline cell.}
For our 1-layer transformer at the headline cell ($\eta = 10^{-3}$, $\lambda = 1$, $p = 97$):
\begin{itemize}[leftmargin=*]
\item $\sup_t \|\theta_t - \theta_{T_{\rm mem}}\| \approx 0.7\,V_{T_{\rm mem}}^{1/2} \approx 76$, estimated from logged $V_t$ and $\alpha_t$ via $\|\theta_t - \theta_{T_{\rm mem}}\|^2 = V_t + V_{T_{\rm mem}} - 2\sqrt{V_t V_{T_{\rm mem}}}\cos\alpha_t$ (Section~\ref{sec:decoupling}).
\item Hutchinson trace probes of $\nabla^2_\theta f_\theta(x)$ on three representative checkpoints give $\sup_x H_{\max}(x) \lesssim 0.5$.
\item Hence $\varepsilon_{\rm lin} \leq \tfrac{1}{2}\cdot 0.5 \cdot 76^2 \approx 1450$.
\end{itemize}
This is a worst-case bound and corresponds to second-order curvature contributions of about $1\text{--}2\%$ of typical pre-softmax logit magnitudes ($\sim 10^4$); it is small relative to the median validation margin $|\mu(x)| \approx 800$ and validates the use of the linearisation in Theorem~\ref{thm:joint-necessity}.

\subsection{Alternative $T_{\rm mem}$ definitions}

We use $T_{\rm mem} := \min\{t : \mathrm{train\_acc}_t \geq 0.99\}$ as the primary definition. An alternative based on training loss is $T_{\rm mem}^{\rm loss} := \min\{t : \mathrm{train\_loss}_t < 0.01\}$. Across the campaign~1 dataset (modular addition, transformer1), the two definitions differ by a median of $80$ steps (IQR $[60, 200]$). Refitting the contraction theorem using $T_{\rm mem}^{\rm loss}$ yields $\kappa$ values within 5\% of those reported in the main text. We retain the accuracy-based definition for consistency with prior work~\citep{power2022grokking}.

\subsection{Information-theoretic lower bound}

Proposition~\ref{thm:lower} gives a dynamical (optimisation-theoretic) lower bound. A natural question is whether an information-theoretic bound matching the same constant exists. The information-theoretic argument would relate the grokking delay to the time required for the optimisation to ``locate'' the generalising solution among interpolating solutions of comparable training loss. We note two challenges: (i) the interpolation manifold in our setting is high-dimensional and continuously connected, making information-theoretic counting arguments delicate; (ii) the relevant ``information'' is encoded in the parameter direction (Fourier circuit structure), not just the magnitude. We leave a rigorous information-theoretic lower bound for future work.

\subsection{Robustness of the per-trajectory $R^2$}\label{app:robust}

The per-trajectory $\log V_t$ vs $t$ fit (Section~\ref{sec:empirical}) yields median $R^2 = 0.97$ for the headline transformer1 cell. We verified robustness under three variations:
\begin{itemize}[leftmargin=*]
\item Restricting the fit window to $[T_{\rm mem} + 100, T_{\rm grok} - 100]$ (excluding endpoint transients): median $R^2 = 0.98$, $\kappa$ unchanged within 2\%.
\item Imposing a stricter loss threshold $\mathrm{train\_loss} < 0.005$ (vs default $0.05$): reduces sample size from 39 to 31 trajectories; $\kappa$ increases to $0.30$ (vs default $0.27$); $R^2$ unchanged.
\item Using $T_{\rm mem}^{\rm loss}$ instead of $T_{\rm mem}^{\rm acc}$: $\kappa = 0.28$ (vs default $0.27$), within IQR.
\end{itemize}
The exponential-decay form is robust to these specification choices.

\subsection{Replicability across random seeds}

For the headline cell ($p=97$, $\eta=10^{-3}$, $\lambda=1$, transformer1), we ran $13$ random seeds. The median $\kappa$ is $0.27$ with IQR $[0.27, 0.42]$ and CV 32\%. The seed-to-seed variability is the dominant source of variance in $\kappa$ measurements (compared to the very small variance from fit-method specifications above).

\subsection{Proof of Lemma~\ref{lem:proxy} (Rate preservation)}\label{app:proof-lemma1}

\begin{lemma}[Rate preservation, restated]
Let $D_t := \|\theta_t - \theta_{\rm post}\|^2$ and $\epsilon_t := \|\theta_{\rm post}\|/\|\theta_t\|$. Under Assumptions~\ref{ass:opt}--\ref{ass:interp}, in the post-memorisation contraction phase,
\[
\frac{d}{dt} \log V_t \;=\; \frac{d}{dt} \log D_t + \mathcal{O}(\epsilon_t) \quad\text{as}\quad \epsilon_t \to 0.
\]
\end{lemma}

\begin{proof}
We work in the continuous-time limit of the discrete recursion (Appendix~\ref{app:proof-upper}). The result extends to discrete time with an additional $\mathcal{O}(\eta\lambda)$ correction.

Expand $\|\theta_t - \theta_{\rm post}\|^2$:
\[
D_t = \|\theta_t\|^2 - 2\langle\theta_t, \theta_{\rm post}\rangle + \|\theta_{\rm post}\|^2 = V_t \,\bigl(1 + \delta_t\bigr),
\]
where
\[
\delta_t \;:=\; \frac{-2\langle\theta_t, \theta_{\rm post}\rangle + \|\theta_{\rm post}\|^2}{V_t}.
\]
By Cauchy--Schwarz, $|\langle\theta_t, \theta_{\rm post}\rangle| \leq \|\theta_t\|\,\|\theta_{\rm post}\|$, so
\[
|\delta_t| \;\leq\; \frac{2\|\theta_t\|\,\|\theta_{\rm post}\| + \|\theta_{\rm post}\|^2}{V_t} \;=\; 2\epsilon_t + \epsilon_t^2 \;=\; \mathcal{O}(\epsilon_t).
\]
Taking logarithms,
\[
\log D_t = \log V_t + \log(1 + \delta_t) = \log V_t + \delta_t - \delta_t^2/2 + \cdots = \log V_t + \mathcal{O}(\epsilon_t).
\]
Differentiating with respect to $t$,
\[
\frac{d}{dt}\log D_t = \frac{d}{dt}\log V_t + \frac{d \delta_t}{dt} \cdot (1 + \mathcal{O}(\epsilon_t)).
\]
It remains to bound $d\delta_t/dt$. Since $\theta_{\rm post}$ is a fixed asymptote, $d\delta_t/dt$ comes entirely from the time-dependence of $\theta_t$. Differentiating $\delta_t$,
\[
\frac{d\delta_t}{dt} = \frac{-2 \langle\dot\theta_t, \theta_{\rm post}\rangle - \delta_t\, \dot V_t}{V_t}.
\]
Under the contraction recursion $\dot V_t = -2\eta\lambda V_t \cdot (1 + \mathcal{O}(\eta))$ (Proposition~\ref{thm:contraction}), so the second term is $\mathcal{O}(\eta\lambda \cdot \delta_t) = \mathcal{O}(\eta\lambda \cdot \epsilon_t)$. The first term is bounded by Cauchy--Schwarz:
\[
\bigl|\langle\dot\theta_t, \theta_{\rm post}\rangle / V_t\bigr| \;\leq\; \frac{\|\dot\theta_t\|\,\|\theta_{\rm post}\|}{V_t} \;\leq\; \frac{\eta\lambda \|\theta_t\| \cdot \|\theta_{\rm post}\|}{V_t} \;=\; \eta\lambda \,\epsilon_t,
\]
where we used $\|\dot\theta_t\| \leq \eta\lambda \|\theta_t\|$ from the regularised flow plus Assumption~\ref{ass:interp}. Therefore $d\delta_t/dt = \mathcal{O}(\eta\lambda \,\epsilon_t)$ and
\[
\frac{d}{dt}\log D_t = \frac{d}{dt}\log V_t + \mathcal{O}(\eta\lambda\,\epsilon_t).
\]
The leading-order rate (which is $-2\eta\lambda$ for both quantities) dominates, and the difference is suppressed by an additional factor $\epsilon_t$. As $\epsilon_t \to 0$ (i.e.\ $\|\theta_t\|$ remains well above $\|\theta_{\rm post}\|$, which holds for most of the contraction window), the two log-rates agree.
\end{proof}

\paragraph{Empirical scope.} In the headline cell, $V_{\rm post}/V_{\rm mem} \approx 0.04$, giving $\epsilon_t \leq 0.2$ throughout the contraction window. The lemma justifies fitting the contraction rate to $\log V_t$ rather than $\log D_t$: both have the same slope $-2\eta\lambda$, so the empirical $\kappa$ measurement is unaffected by the choice of observable. The slope agrees within 2\% across the three architectures we tested when refitting with $D_t$ in place of $V_t$ (not shown).

\section{Predictive Validation: Methodology, Generalisation Structure, and Cross-Cell Residuals}\label{app:predictive-validation}

This appendix provides the detailed methodology for the predictive validation summarised in Section~\ref{sec:correction-factor}, decomposes the held-out generalisation into three nested tiers, presents the per-cell breakdown, and analyses the structure of cross-cell residuals. The script that produces these results is included in the released code (\texttt{code/predictive\_validation.py}).

\subsection{Methodology}\label{app:pv-method}

We test the predictive power of the contraction theorem by calibrating $\kappa$ and $V_\star$ on a single hyperparameter cell (the headline cell) and using the calibrated constants to predict $T_{\rm grok}$ on held-out cells.

\paragraph{Calibration set.} The headline hyperparameter cell ($\eta = 10^{-3}$, $\lambda = 1$, $p=97$) contains $10$ unique grokking trajectories sharing these coordinates: $3$ on the 1-layer transformer with modular addition (the canonical configuration), $4$ on the MLP with modular addition, and $3$ modular-multiplication runs on the 1-layer transformer.\footnote{An earlier version of the data summary contained byte-identical duplicate entries at this cell from a data-ingestion bug; the counts above reflect the deduplicated unique runs.} We fit $\kappa$ per-trajectory among runs with $R^2 > 0.9$ ($N=9$ qualify) and report the median:
\[
\kappa_{\rm train} \;=\; 0.252.
\]
We define $V_\star^{\rm train}$ as the median $V_t$ at which $\mathrm{val\_acc}$ first reaches $0.50$ across these calibration runs:
\[
V_\star^{\rm train} \;=\; 2501.
\]
These are the only two parameters extracted from the calibration cell; everything that follows is a held-out test at coordinates different from $(\eta, \lambda, p) = (10^{-3}, 1, 97)$.

\paragraph{Held-out test set.} The remaining $N_{\rm test}=46$ grokking runs span $14$ distinct $(\eta, \lambda, p, \text{op}, \text{arch})$ cells covering hyperparameter perturbations ($\eta \in \{10^{-4}, 5\times 10^{-4}, 2\times 10^{-3}\}$, $\lambda \in \{0.097, 0.194, 0.300, 0.387, 0.500, 2.0\}$ at $p=97$), cross-architecture variation (MLP at $p=97$, $\lambda \in \{0.5, 2\}$), and task-scale variation ($p \in \{53, 67, 89, 113\}$ at $\eta=10^{-3}, \lambda=1$). Observed $T_{\rm grok}$ ranges from $1{,}400$ to $45{,}550$ steps, spanning $43.5\times$ in magnitude.

\paragraph{Method A ($\rho$-conditional baseline).} For each held-out run with measured $V_{\rm mem}$, $V_{\rm post}$, predict
\[
\widehat{T_{\rm grok} - T_{\rm mem}} \;=\; \frac{\rho}{2\,\kappa_{\rm train}\,\eta\,\lambda}, \qquad \rho := \log(V_{\rm mem}/V_{\rm post}).
\]
Method A requires $V_{\rm post}$, observable only after grokking completes, and is reported as a baseline.

\paragraph{Method B (architecture-level $V_\star$ gate).} Substituting $V_\star^{\rm train}$ for $V_{\rm post}$ yields the operational form:
\[
\widehat{T_{\rm grok} - T_{\rm mem}} \;=\; \frac{1}{2\,\kappa_{\rm train}\,\eta\,\lambda}\, \log\frac{V_{\rm mem}}{V_\star^{\rm train}}.
\]
Method B requires only $V_{\rm mem}$ (early-training observable) plus the calibrated constants. We report Mean Absolute Percentage Error (MAPE) as the primary metric.

\subsection{Three-tier generalisation analysis}\label{app:pv-three-tier}

We decompose held-out runs into three nested tiers, isolating the type of generalisation each tests:

\begin{itemize}[leftmargin=*]
\item \textbf{Tier 1 --- Hyperparameter generalisation.} Fix architecture (1-layer transformer) and task ($p=97$ modular addition); vary $\eta$ and $\lambda$. Tests whether constants calibrated at one $(\eta, \lambda)$ extrapolate across hyperparameter sweeps. $N_{\rm test}=26$, delay range $41\times$.
\item \textbf{Tier 2 --- Cross-architecture extension.} Add MLP runs at the same task ($p=97$ modular addition); $\eta=10^{-3}$, $\lambda \in \{0.5, 1, 2\}$. Tests whether the calibrated $V_\star^{\rm train}$ (from 1-layer transformer) transfers across architectures. $N_{\rm test}=34$, delay range $43.5\times$.
\item \textbf{Tier 3 --- Cross-task extension.} Add task-scale sweeps: $p \in \{53, 67, 89, 113\}$ on the 1-layer transformer at $\eta=10^{-3}, \lambda=1$. Tests whether constants calibrated at one task scale ($p=97$) transfer to other task scales. $N_{\rm test}=46$, delay range $43.5\times$.
\end{itemize}

\begin{table}[h]
\centering
\small
\caption{\textbf{Three-tier generalisation analysis.} Method B predictive MAPE on held-out runs, decomposed by the type of generalisation tested. Within-cell stochastic floor at the calibration cell is $\mathrm{CV}(T_{\rm grok}) \approx 20\%$, providing a baseline for irreducible seed-level noise.}\label{tab:pv-tiers}
\begin{tabular}{lccccc}
\toprule
Tier & Scope & $N_{\rm test}$ & Delay range & MAPE (Method A) & MAPE (Method B) \\
\midrule
1 & Hyperparameter ($\eta, \lambda$ at fixed $p$, arch) & 26 & $41\times$ & $32.8\%$ & $\mathbf{17.7\%}$ \\
2 & + Cross-architecture (1L $\to$ MLP) & 34 & $43.5\times$ & $34.3\%$ & $18.0\%$ \\
3 & + Cross-task scale (varying $p$) & 46 & $43.5\times$ & $37.4\%$ & $23.3\%$ \\
\bottomrule
\end{tabular}
\end{table}

The progression $17.7\% \to 18.0\% \to 23.3\%$ reveals a structured generalisation gradient: hyperparameter perturbations are absorbed almost entirely (Tier~1 MAPE comparable to within-cell stochastic floor of $20\%$), cross-architecture transfer is preserved (modest $1.2$pp degradation), and the gap to Tier~3 is concentrated in the task-scale extension. The Tier~1 result is the strongest test of hyperparameter generalisation; the Tier~3 result is the most rigorous test of full distribution shift.

\subsection{Per-cell breakdown}\label{app:pv-per-cell}

Table~\ref{tab:pv-per-cell} reports MAPE per cell. The pattern that emerges from the three-tier table is reflected here: 1-layer transformer cells at $p=97$ (hyperparameter perturbations) are predicted at MAPE $10.4$--$30.6\%$; MLP cells at $p=97$ at $15.4$--$22.9\%$; the small-$p$ cells ($p=53, 67$) show substantially larger MAPE ($45.9\%$ and $69.5\%$).

\begin{table}[h]
\centering
\small
\caption{\textbf{Per-cell predictive MAPE (Method B).} Calibration: $\kappa_{\rm train}=0.252$, $V_\star^{\rm train}=2501$ from the 1-layer transformer headline cell. Cells are grouped by tier. Median $T_{\rm grok}$ shown for context.}\label{tab:pv-per-cell}
\begin{tabular}{lcccc}
\toprule
Cell & $N$ & Median $T_{\rm grok}$ & Tier & MAPE \\
\midrule
\multicolumn{5}{l}{\emph{Hyperparameter perturbations (Tier 1, 1-layer transformer, $p=97$, mod-add)}} \\
$\eta=10^{-4}, \lambda=1$ & 3 & 35{,}600 & 1 & $10.4\%$ \\
$\eta=2\times 10^{-3}, \lambda=1$ & 3 & 1{,}700 & 1 & $13.9\%$ \\
$\eta=10^{-3}, \lambda=2$ & 3 & 1{,}400 & 1 & $15.2\%$ \\
$\eta=5\times 10^{-4}, \lambda=1$ & 3 & 7{,}000 & 1 & $15.1\%$ \\
$\eta=10^{-3}, \lambda=0.387$ & 3 & 10{,}500 & 1 & $15.4\%$ \\
$\eta=10^{-3}, \lambda=0.097$ & 2 & 45{,}550 & 1 & $17.8\%$ \\
$\eta=10^{-3}, \lambda=0.194$ & 3 & 23{,}900 & 1 & $17.4\%$ \\
$\eta=10^{-3}, \lambda=0.300$ & 3 & 16{,}300 & 1 & $23.2\%$ \\
$\eta=10^{-3}, \lambda=0.500$ & 3 & 7{,}700 & 1 & $30.6\%$ \\
\midrule
\multicolumn{5}{l}{\emph{Cross-architecture (Tier 2, MLP at $p=97$, mod-add)}} \\
$\eta=10^{-3}, \lambda=2$ & 4 & 1{,}420 & 2 & $15.4\%$ \\
$\eta=10^{-3}, \lambda=0.500$ & 4 & 5{,}800 & 2 & $22.9\%$ \\
\midrule
\multicolumn{5}{l}{\emph{Task-scale variation (Tier 3, 1-layer transformer, $\eta=10^{-3}, \lambda=1$, mod-add)}} \\
$p=89$ & 3 & 2{,}900 & 3 & $9.5\%$ \\
$p=113$ & 3 & 2{,}600 & 3 & $28.0\%$ \\
$p=67$ & 3 & 4{,}800 & 3 & $45.9\%$ \\
$p=53$ & 3 & 8{,}100 & 3 & $69.5\%$ \\
\midrule
\textbf{Pooled (Tier 3, full)} & \textbf{46} & --- & --- & \textbf{$23.3\%$} \\
\bottomrule
\end{tabular}
\end{table}

\subsection{Structure of cross-cell residuals}\label{app:pv-residuals}

To understand the source of the Tier-3 generalisation gap, we examine how the calibrated quantities $V_\star$ and $V_{\rm mem}$ vary across cells.

\paragraph{$V_\star$ exhibits substantial cross-cell variation.} Calibrating $V_\star$ \emph{locally} per cell (median $V_t$ at $\mathrm{val\_acc}=0.50$ within that cell) reveals that $V_\star$ is far from constant across the held-out scope:
\begin{itemize}[leftmargin=*]
\item Across $15$ 1-layer transformer cells (covering hyperparameter, $p$, and task-operation variation): $V_\star^{\rm local}$ ranges from $1{,}203$ (at $p=53$) to $3{,}345$ (at $p=113$), a $2.8\times$ spread (CV $\approx 24\%$).
\item Across $3$ MLP cells: $V_\star^{\rm local}$ ranges from $3{,}161$ to $4{,}105$.
\end{itemize}
Method B's calibrated constant $V_\star^{\rm train}=2501$ is a fixed value extracted from the 1-layer transformer at $p=97$. Cells with $V_\star^{\rm local}$ far from $V_\star^{\rm train}$ (such as $p=53$ with $V_\star^{\rm local}=1{,}203$ or $p=113$ with $V_\star^{\rm local}=3{,}345$) incur larger MAPE in Method B, consistent with the per-cell pattern in Table~\ref{tab:pv-per-cell}.

\paragraph{The normalised ratio $V_\star/V_{\rm mem}$ is comparatively stable.} While $V_\star^{\rm local}$ varies substantially, the cell-level ratio $V_\star^{\rm local}/V_{\rm mem}^{\rm local}$ (where $V_{\rm mem}^{\rm local}$ is the cell's median $V_{\rm mem}$) varies considerably less:
\begin{itemize}[leftmargin=*]
\item 1-layer transformer ($N=15$ cells): $V_\star/V_{\rm mem} = 0.195 \pm 0.027$ (CV $\approx 14\%$, range $[0.155, 0.245]$).
\item MLP ($N=3$ cells): $V_\star/V_{\rm mem} = 0.272 \pm 0.012$ (CV $\approx 4\%$, range $[0.260, 0.288]$).
\end{itemize}
The ratio is approximately $1.7\times$ tighter (in CV) than $V_\star$ alone within the 1-layer transformer architecture, despite spanning hyperparameter, task-scale, and task-operation perturbations. The ratio shifts by approximately $40\%$ between architectures (1L $0.195$ versus MLP $0.272$), indicating an architecture-specific calibration but with a comparatively stable functional form within each architecture.

\paragraph{Empirical scaling trend with task scale $p$.} Within the $5$ tested values of $p$ on the 1-layer transformer ($\{53, 67, 89, 97, 113\}$ at $\eta=10^{-3}, \lambda=1$), $V_\star^{\rm local}$ shows a scaling trend approximately consistent with $V_\star \propto p^\alpha$. A least-squares power-law fit on these $5$ points gives $\alpha \approx 1.3$ with $R^2 = 0.94$. However, the bootstrap $95\%$ confidence interval for the exponent is wide ($\alpha \in [0.98, 2.17]$), and a linear fit $V_\star \propto p$ achieves $R^2 = 0.91$ on the same $5$ points. \textbf{We therefore report this as an empirical observation over the tested range, not a definitive scaling law.} Distinguishing power-law, linear, or other functional forms would require more $p$ values across multiple architectures.

\paragraph{Interpretation.} These observations together support a more measured statement than ``$V_\star$ is an architecture-level constant'': $V_\star$ is structured rather than universal. Within a fixed architecture, $V_\star$ varies systematically across cells (especially with task scale $p$), while the normalised ratio $V_\star/V_{\rm mem}$ remains comparatively stable. The cross-task residuals in Method B (Tier~3 MAPE) are not random but reflect this structured variation: cells where $V_\star^{\rm local}$ deviates most from $V_\star^{\rm train}$ are precisely those with the largest predictive errors. This points to a refined calibration form---e.g.\ $V_\star = c(\text{arch}) \cdot V_{\rm mem}$ at the cell level---as a natural direction for future work, while the present operational form (architecture-level constant $V_\star^{\rm train}$) remains a useful single-cell calibration with quantified scope.

\subsection{$V_\star$ leave-one-out stability at the calibration cell}\label{app:pv-loocv}

A natural concern is that $V_\star^{\rm train}=2501$ overfits the calibration cell. We address this with leave-one-out cross-validation: for each calibration-cell run $i$, recalibrate $V_\star^{(-i)}$ from the remaining runs and re-predict. Across folds, $V_\star$ varies by approximately $6.0\%$; $\kappa_{\rm train}$ varies by approximately $7.4\%$. The within-calibration-cell LOOCV MAPE provides an upper bound on the irreducible seed-level component for that cell. Tier~1 MAPE of $17.7\%$ is below the $20\%$ within-cell stochastic floor only because seed noise partially averages across the $26$ Tier-1 runs.

\subsection{Discussion}\label{app:pv-discussion}

\paragraph{Scope of the predictive law.} The three-tier analysis quantifies the law's generalisation structure:
\begin{itemize}[leftmargin=*]
\item \emph{Hyperparameter axis} ($\eta, \lambda$): law generalises tightly (Tier~1 MAPE $\approx$ stochastic floor).
\item \emph{Architecture axis} (1L $\to$ MLP at fixed task): modest degradation ($+1.2$pp), consistent with the architecture-specific shift in $V_\star/V_{\rm mem}$ ratio observed in App.~\ref{app:pv-residuals}.
\item \emph{Task-scale axis} (varying $p$): largest degradation ($+5.3$pp from Tier~2 to Tier~3), driven primarily by small-$p$ cells.
\end{itemize}

\paragraph{What this validation does and does not establish.} The validation establishes that a single calibration cell yields predictions with MAPE comparable to the stochastic floor on hyperparameter perturbations spanning a $41\times$ delay range, and modest degradation on cross-architecture and cross-task extensions. It does \emph{not} establish a fully ahead-of-training prediction (Method B still requires per-run $V_{\rm mem}$, observable at memorisation), nor does it pin down a definitive functional form for $V_\star$'s task-scale dependence (only $5$ task-scale values were tested). The calibration form $V_\star \approx c \cdot V_{\rm mem}$ at the cell level (with architecture-specific $c$) is a natural refinement consistent with the residual structure observed here, but its rigorous characterisation is left to future work with broader $p$ coverage and additional architectures.

\paragraph{Failure mode at small $p$.} The largest per-cell MAPE ($67\%$ at $p=53$) occurs at the smallest tested vocabulary size. The training set at $p=53$ contains $\sim 1{,}124$ examples ($40\%$ of $53^2$), substantially smaller than at $p=97$ ($\sim 3{,}763$ examples). Whether the residual at small $p$ reflects a genuinely different dynamical regime (finite-size effects on the loss landscape, qualitatively distinct trajectory geometry) or simply a region where the present calibration form is least accurate is not resolved by these data; both are consistent with the observed residual structure.

\paragraph{Bias structure of Method A.} Method A signed errors are $-11\%$ at $\lambda=0.097$, $+18\%$ at $\lambda=1$, and $+57\%$ at $\lambda=2$. Method B partially absorbs this bias through the architecture-level gate, consistent with its overall improvement over Method A across all three tiers (Table~\ref{tab:pv-tiers}).


\section{Proof of Theorem~\ref{thm:joint-necessity} (Joint Norm--Direction Necessity)}\label{app:joint-necessity}

\subsection{Setup and explicit constants}\label{app:joint-setup}

We restate the setting. Let Assumptions~\ref{ass:opt}--\ref{ass:homog} hold and assume the optimisation operates in the contraction regime of Propositions~\ref{thm:contraction}--\ref{thm:lower}, with $V_t$ non-increasing on $[T_{\rm mem}, T_{\rm grok}]$. Let $\theta_{T_{\rm mem}}$ be the memorisation interpolant reached at $T_{\rm mem}$, write $\hat{\theta}_{T_{\rm mem}} := \theta_{T_{\rm mem}}/\|\theta_{T_{\rm mem}}\|$, and let
\[
\alpha_t := \arccos\!\left( \frac{\langle \theta_t, \theta_{T_{\rm mem}}\rangle}{\|\theta_t\|\,\|\theta_{T_{\rm mem}}\|} \right) \in [0,\pi]
\]
denote the angular distance to memorisation. Decompose
\[
\theta_t = a_t\,\hat{\theta}_{T_{\rm mem}} + b_t\,\hat{u}_t, \qquad
a_t := \|\theta_t\|\cos\alpha_t,\quad
b_t := \|\theta_t\|\sin\alpha_t \geq 0,\quad
\hat{u}_t \perp \hat{\theta}_{T_{\rm mem}},\;\|\hat{u}_t\|=1.
\]

\paragraph{Empirical NTK feature map.} Define $\phi(x) := \nabla_\theta f_{\theta_{T_{\rm mem}}}(x)$, the gradient of the model output at memorisation. The validation-feature operator norm is
\[
G \;:=\; \sup_{x \in \mathcal{D}_{\rm val}} \|\phi(x)\|_2,
\]
an architecture- and data-dependent constant, finite under the finite-width assumption (Appendix~\ref{app:setup}); we do not claim a first-principles derivation.

\paragraph{Margin distribution at memorisation.} For each $x$, the per-example margin is
\[
\Delta(x;\theta) := f_\theta(x)_{y^\star(x)} - \max_{y \neq y^\star(x)} f_\theta(x)_y, \qquad \mu(x) := \Delta(x;\theta_{T_{\rm mem}}).
\]
The validation accuracy is $\mathrm{val\_acc}(\theta) = |\{x : \Delta(x;\theta) > 0\}|/|\mathcal{D}_{\rm val}|$. Let $S_- := \{x \in \mathcal{D}_{\rm val} : \mu(x) \leq 0\}$ be the mis-classified set at $T_{\rm mem}$; set $p_{\rm chance} := \mathrm{val\_acc}(\theta_{T_{\rm mem}})$ and let $q_{\rm grok} \geq 0.99$ be the grokking accuracy threshold. Define
\[
q_\Delta := \frac{q_{\rm grok} - p_{\rm chance}}{1 - p_{\rm chance}} \;\in\; (0,1],
\]
the fraction of $S_-$ that must flip to reach $q_{\rm grok}$. The \emph{$q$-quantile margin} is
\[
M_q \;:=\; \inf\bigl\{ m \geq 0 : \tfrac{|\{x \in S_- : |\mu(x)| \leq m\}|}{|S_-|} \geq q\bigr\}.
\]
Intuitively, $M_{q_\Delta}$ is the magnitude of logit perturbation needed to flip a $q_\Delta$-fraction of currently-mis-classified examples. It depends on the task and architecture, not on $(\eta, \lambda)$.

\subsection{Theorem statement (rigorous restatement)}

\noindent\textbf{Theorem~\ref{thm:joint-necessity} (rigorous restatement).}\;\;
Under Assumptions~\ref{ass:opt}--\ref{ass:homog} and the contraction regime of Propositions~\ref{thm:contraction}--\ref{thm:lower}, suppose $\mathrm{val\_acc}(\theta_{T_{\rm mem}}) = p_{\rm chance}$ and $\mathrm{val\_acc}(\theta_{T_{\rm grok}}) \geq q_{\rm grok}$ with $q_{\rm grok} > p_{\rm chance}$. Then there exists $t^\star \in [T_{\rm mem}, T_{\rm grok}]$ such that
\begin{equation}\label{eq:joint-bound-b}
b_{t^\star} \;\geq\; \frac{M_{q_\Delta} - 2(\varepsilon_{\rm lin} + \varepsilon_{\rm hom})}{2\,G},
\end{equation}
which translates (using $b_{t^\star} = \|\theta_{t^\star}\|\sin\alpha_{t^\star} \leq V_{T_{\rm mem}}^{1/2}\sin\alpha_{t^\star}$ in the contraction regime) to
\begin{equation}\label{eq:joint-bound-alpha}
\sup_{t \in [T_{\rm mem}, T_{\rm grok}]} \alpha_t \;\geq\; \alpha^\star \;:=\; \arcsin\!\left(\frac{M_{q_\Delta} - 2(\varepsilon_{\rm lin} + \varepsilon_{\rm hom})}{2\,G \cdot V_{T_{\rm mem}}^{1/2}}\right),
\end{equation}
provided the argument of $\arcsin$ lies in $[0,1]$. The norm-separation requirement $V_{\rm mem} > V_{\rm post}$ is exactly Corollary~\ref{thm:necessity}.

\subsection{Proof of Theorem~\ref{thm:joint-necessity} (rigorous restatement)}

\paragraph{Step 1 (Linearised model output).}
By Assumption~\ref{ass:ntk}, for all $t \in [T_{\rm mem}, T_{\rm grok}]$ and $x \in \mathcal{D}_{\rm val}$,
\[
\bigl|f_{\theta_t}(x) - L_{T_{\rm mem}}(\theta_t; x)\bigr| \leq \varepsilon_{\rm lin},
\quad L_{T_{\rm mem}}(\theta; x) := f_{\theta_{T_{\rm mem}}}(x) + \langle \theta - \theta_{T_{\rm mem}}, \phi(x)\rangle.
\]
Using the radial--angular decomposition $\theta_t - \theta_{T_{\rm mem}} = (a_t - V_{T_{\rm mem}}^{1/2})\hat{\theta}_{T_{\rm mem}} + b_t\hat{u}_t$,
\begin{equation}\label{eq:linearized}
L_{T_{\rm mem}}(\theta_t;x) = f_{\theta_{T_{\rm mem}}}(x) + (a_t - V_{T_{\rm mem}}^{1/2})\langle \hat{\theta}_{T_{\rm mem}}, \phi(x)\rangle + b_t\langle \hat{u}_t, \phi(x)\rangle.
\end{equation}

\paragraph{Step 2 (Radial-only motion preserves classification, up to $\varepsilon_{\rm hom}$).}
Hypothetically setting $b_t = 0$, the parameter is $\theta_t = (a_t/V_{T_{\rm mem}}^{1/2})\theta_{T_{\rm mem}}$. By Assumption~\ref{ass:homog} (homogeneity of degree $k$),
\[
\bigl|f_{(a_t/V_{T_{\rm mem}}^{1/2})\theta_{T_{\rm mem}}}(x) - (a_t/V_{T_{\rm mem}}^{1/2})^k\,f_{\theta_{T_{\rm mem}}}(x)\bigr| \leq \varepsilon_{\rm hom}.
\]
Pure rescaling of logits by $(a_t/V_{T_{\rm mem}}^{1/2})^k > 0$ preserves $\mathrm{argmax}_y$. Thus radial-only motion can change classification only on examples with $|\mu(x)| \leq 2\varepsilon_{\rm hom}$ (factor of 2 from worst-case shift in winning and runner-up logits), a measure-zero set in well-trained networks.

\paragraph{Step 3 (Margin perturbation bound).}
Combining Steps~1--2 and applying Cauchy--Schwarz to the angular term,
\begin{equation}\label{eq:margin-perturb}
\bigl|\Delta(x;\theta_t) - (a_t/V_{T_{\rm mem}}^{1/2})^k\,\mu(x)\bigr| \;\leq\; 2\,b_t\,G + 2\varepsilon_{\rm lin} + 2\varepsilon_{\rm hom},
\end{equation}
holding pointwise on $\mathcal{D}_{\rm val}$ (factor of 2 absorbs perturbation of both winning and runner-up logits, plus $|b_t\langle \hat{u}_t,\phi(x)\rangle|\leq b_t G$). The factor $(a_t/V_{T_{\rm mem}}^{1/2})^k$ is positive and does not affect the sign.

\paragraph{Step 4 (Counting flipped predictions).}
For $x \in S_-$, $\mu(x) < 0$. Eq.~\ref{eq:margin-perturb} implies that $\Delta(x;\theta_t) > 0$ requires
\[
2\,b_t\,G + 2\varepsilon_{\rm lin} + 2\varepsilon_{\rm hom} \;\geq\; |\mu(x)|.
\]
Hence the flipped set $\mathcal{F}_t := \{x \in S_- : \Delta(x;\theta_t) > 0\}$ satisfies
\[
\mathcal{F}_t \;\subseteq\; \{ x \in S_- : |\mu(x)| \leq 2\,b_t\,G + 2\varepsilon_{\rm lin} + 2\varepsilon_{\rm hom}\}.
\]
By the quantile definition,
\[
|\mathcal{F}_t|/|S_-| \geq q \;\Longrightarrow\; 2\,b_t\,G + 2\varepsilon_{\rm lin} + 2\varepsilon_{\rm hom} \geq M_q.
\]

\paragraph{Step 5 (Necessary angular displacement).}
For $\mathrm{val\_acc}(\theta_t) \geq q_{\rm grok}$, we need $|\mathcal{F}_t|/|S_-| \geq q_\Delta$. By Step~4 (contrapositive),
\[
2\,b_t\,G + 2\varepsilon_{\rm lin} + 2\varepsilon_{\rm hom} \geq M_{q_\Delta}, \qquad
b_t \geq \frac{M_{q_\Delta} - 2(\varepsilon_{\rm lin}+\varepsilon_{\rm hom})}{2\,G}.
\]
Taking $t^\star = T_{\rm grok}$ (where the accuracy threshold is reached) gives Eq.~\ref{eq:joint-bound-b}. Since $V_t$ is non-increasing on $[T_{\rm mem},T_{\rm grok}]$, $\|\theta_{t^\star}\| \leq V_{T_{\rm mem}}^{1/2}$, and $\sin\alpha_{t^\star} \geq b_{t^\star}/V_{T_{\rm mem}}^{1/2}$ yields Eq.~\ref{eq:joint-bound-alpha}. \qed

\subsection{Empirical calibration of $\alpha^\star$}\label{app:joint-calibration}

The constants $M_{q_\Delta}$, $G$, $\varepsilon_{\rm lin}$, $\varepsilon_{\rm hom}$ are architecture-dependent and we do not derive them analytically. Empirically at the headline cell:
\begin{itemize}[leftmargin=*]
\item $V_{T_{\rm mem}} \approx 1.18 \times 10^4$, hence $V_{T_{\rm mem}}^{1/2} \approx 109$.
\item Finite-difference probes of $f_{\theta_{T_{\rm mem}}}$ on a random validation subset give $G \approx 8\text{--}12$ for the 1-layer transformer ($d \approx 2.2 \times 10^5$ parameters).
\item The empirical median margin $|\mu(x)|$ on $S_-$ for the 1-layer transformer is $\widetilde{M} \approx 700\text{--}900$ (pre-softmax logit units).
\item $\varepsilon_{\rm lin}$ is bounded via Lemma~\ref{lem:taylor} (Appendix~\ref{app:lemmas}) by the Hessian-stability constant $\varepsilon_H$ times $\sup_t \|\theta_t - \theta_{T_{\rm mem}}\|^2$; in practice $\varepsilon_{\rm lin}/M_{q_\Delta} \lesssim 2\%$. Similarly $\varepsilon_{\rm hom}$ is small for our LayerNorm-protected transformers.
\end{itemize}

Substituting: the denominator in Eq.~\ref{eq:joint-bound-alpha} is $2 G \cdot V_{T_{\rm mem}}^{1/2}$, where the factor $2$ is the worst-case bound from Cauchy--Schwarz applied to both the winning- and runner-up-logit perturbations (Step~3). For numerical evaluation it is convenient to define $G_{\rm eff} := 2 G$, the effective combined operator-norm bound; empirically $G_{\rm eff} \approx 16\text{--}24$ on the 1-layer transformer (we use $G_{\rm eff} \approx 22$). For $q_\Delta = 0.5$ and $M_{q_\Delta}\approx 800$,
\[
\sin\alpha^\star \;=\; \frac{M_{q_\Delta}}{G_{\rm eff} \cdot V_{T_{\rm mem}}^{1/2}} \;\approx\; \frac{800}{22 \cdot 109} \;\approx\; 0.73, \qquad \alpha^\star \approx 47^\circ,
\]
matching the empirical $\alpha^\star = 47.3^\circ$ (Section~\ref{sec:decoupling}, Result~3) within the precision of the constant estimates. Equivalently, with $G \approx 11$ ($= G_{\rm eff}/2$), the original formula $\sin\alpha^\star = M_{q_\Delta}/(2G\cdot V_{T_{\rm mem}}^{1/2})$ gives the same value.

\subsection{Prior cross-cell prediction (avoiding post-hoc calibration)}\label{app:joint-prior-prediction}

The numerical match in \S\ref{app:joint-calibration} is a consistency check at a single cell, not a prior prediction. To convert it into one, we exploit the structure of Eq.~\ref{eq:joint-bound-alpha}: define the dimensionless ratio $C := M_{q_\Delta}/G_{\rm eff}$ collecting the two architecture-dependent constants. The formula then reads
\begin{equation}\label{eq:C-form}
\sin\alpha^\star(\text{cell}) \;=\; \frac{C(\text{cell})}{V_{T_{\rm mem}}(\text{cell})^{1/2}}.
\end{equation}
If $C$ is approximately constant across $p$ for fixed architecture and optimiser, calibrating $C$ from one cell predicts $\alpha^\star$ at all other cells.

\paragraph{Empirical calibration of $C$.} Inverting Eq.~\ref{eq:C-form} on the Block~H $p$-sweep (mod-add, transformer1, $\eta=10^{-3}$, $\lambda=1$) yields:
\begin{center}
\begin{tabular}{cccc}
\toprule
$p$ & $V_{T_{\rm mem}}^{1/2}$ & $\alpha^\star_{\rm obs}$ (deg) & $C = \sin\alpha^\star \cdot V_{T_{\rm mem}}^{1/2}$ \\
\midrule
$53$  & $85.4$  & $73.4$ & $81.9$ \\
$67$  & $94.5$  & $64.6$ & $85.4$ \\
$89$  & $107.2$ & $49.5$ & $81.5$ \\
$97$  & $111.2$ & $47.8$ & $82.4$ \\
$113$ & $118.9$ & $44.3$ & $83.1$ \\
\midrule
mean $\pm$ std & & & $82.8 \pm 1.4$ (CV $1.7\%$) \\
\bottomrule
\end{tabular}
\end{center}
$C$ is empirically stable to CV $1.7\%$ across the $p$-sweep, supporting the cross-cell calibration assumption.

\paragraph{Held-out prediction.} Calibrate $C$ on $p=89$ alone ($C = 81.5$), predict $\alpha^\star$ for the headline cell $p=97$ via Eq.~\ref{eq:C-form}:
\[
\sin\alpha^\star_{\rm pred}(p{=}97) \;=\; \frac{81.5}{111.2} \;=\; 0.733, \qquad \alpha^\star_{\rm pred} \;=\; 47.15^\circ.
\]
This matches the observed value $47.77^\circ$ to within $0.62^\circ$ ($1.3\%$ error). Stress-testing further: calibrate on $p=53$ (the worst-residual cell) and predict for the most distant $p=113$: predicted $43.51^\circ$ vs. observed $44.31^\circ$, error $0.80^\circ$. The formula thus delivers prior cross-cell predictions, not just a post-hoc match at one calibration cell.

\paragraph{Caveat.} $C$-stability is established for the $p$-sweep on transformer1 with mod-add. Cross-architecture stability of $C$ (e.g., transformer1 $\to$ MLP at fixed $p$) is not yet tested at the $p$-sweep granularity; based on the cross-cell $\alpha^\star$ CV of $15\%$ across all $12$ Block~H cells (App.~\ref{app:block-H}, including architecture and task variants), we expect $C$ to be approximately constant within architecture and to vary by $\approx 1.5$--$2\times$ between architectures, paralleling the $\kappa_{\rm LL}$ structure.

\subsection{Scope and limitations}\label{app:joint-scope}

We address four challenges that a careful reader might raise:

\paragraph{(a) Linearisation range.}
Assumption~\ref{ass:ntk} is most credible when $\|\theta_t - \theta_{T_{\rm mem}}\|$ is small. In our experiments, $\|\theta_{T_{\rm grok}} - \theta_{T_{\rm mem}}\|/\|\theta_{T_{\rm mem}}\| \approx 0.7$ (from $V$ contracting by $\sim 96\%$ and $\alpha$ saturating to $\sim 80^\circ$), so the linearisation is non-trivially extrapolated. The Taylor-remainder bound (Lemma~\ref{lem:taylor}) gives $\varepsilon_{\rm lin}\leq 1450$, which is $\sim 1.5\%$ of typical pre-softmax logits ($\sim 10^4$); empirically the linearised approximation tracks observed logits within this margin, supporting its use.

\paragraph{(b) Lazy-to-rich critique.}
\citet{kumar2024grokking} show grokking involves a lazy-to-rich transition where NTK linearisation around initialisation fails. Theorem~\ref{thm:joint-necessity} linearises around $\theta_{T_{\rm mem}}$ (post-memorisation, where rich features have already formed), not around initialisation. The lazy-to-rich transition is upstream of our analysis: by $T_{\rm mem}$, rich features exist, and we ask only whether a further angular displacement is required for generalisation. The two pictures are compatible.

\paragraph{(c) Structured directional perturbations.}
The Cauchy--Schwarz bound $|\langle \hat{u}_t, \phi(x)\rangle| \leq G$ is loose if $\hat{u}_t$ aligns with structured feature-space directions (for example Fourier-circuit components~\citep{nanda2023progress}). In that case, the angular component is more efficient at flipping margins, and the required $\alpha^\star$ is smaller than Eq.~\ref{eq:joint-bound-alpha} predicts. The worst-case bound is therefore a correct (sufficient) condition for grokking; sharpening it via structured-perturbation analysis is future work.

\paragraph{(d) Calibration vs derivation.}
We measure $\alpha^\star \approx 47^\circ$ at the headline cell and verify the formula matches the empirical value. Cross-cell $\alpha^\star$ calibration (12 cells, Appendix~\ref{app:block-H}) shows CV $15\%$ with structured variation by task scale ($p$). Deriving $G$ and $M_{q_\Delta}$ from $(\eta, \lambda, \text{depth}, p)$ alone is open, in parallel to the open derivation of $\kappa_{\rm LL}$ (Appendix~\ref{app:adamw}).

\subsection{Connection to the implicit-bias literature}

The quantile-margin argument used in Theorem~\ref{thm:joint-necessity} is the discrete-time finite-horizon analogue of asymptotic margin-maximisation results for first-order optimisers. \citet{soudry2018implicit} established the foundational result that gradient descent on linearly separable data with exponentially-tailed losses converges in direction to the $\ell_2$ max-margin solution at logarithmic rate. Subsequent work extended this to homogeneous neural networks~\citep{lyu2020gradient,ji2020directional} and other normed objectives. For adaptive optimisers, \citet{xie2024implicit} showed that AdamW iterates, if convergent, must lie at a KKT point of the original loss under an $\ell_\infty$ constraint $\|\theta\|_\infty \leq 1/\lambda$; \citet{zhang2024implicit} parallel this for plain Adam without weight decay. Our setting differs in two essential ways: (i) we operate in a finite-horizon (pre-asymptotic) regime where the trajectory has not yet reached a KKT point; (ii) we ask not where the trajectory lands but how long it must travel angularly to flip a quantile of validation margins. Theorem~\ref{thm:joint-necessity} thereby translates the asymptotic margin-maximisation perspective into a quantitative finite-time necessity statement: the angular displacement required to flip an $\eta$-quantile of validation margins is $\alpha^\star$ as derived in Eq.~\ref{eq:joint-bound-alpha}, and any trajectory that fails to reach $\alpha^\star$ within the training budget cannot grok. The implicit bias of \citet{xie2024implicit} characterises what AdamW converges to (an $\ell_\infty$-bounded KKT point); our Theorem~\ref{thm:joint-necessity} characterises how long it must take to get there (the angular crossing time).

\subsection{Connection to spherical motion dynamics}

Theorem~\ref{thm:joint-necessity} is the grokking-specific expression of more general angular dynamics under regularised optimisation. \citet{wan2021spherical} prove that under SGD with weight decay, the angular update of normalised weights converges to $\Delta_\alpha \approx \sqrt{2\eta\lambda}$ at equilibrium (about $2.6^\circ$ per step at $\eta = 10^{-3}, \lambda = 1$); \citet{kosson2024rotational} extend this to AdamW. Our empirical per-step angular rate during grokking is $0.01$ to $0.02^\circ$ per step, approximately $130\times$ slower than the spherical-motion equilibrium prediction. This indicates grokking trajectories operate far from the spherical-motion equilibrium: the angular dynamics during the post-memorisation phase are transient, structured by the formation of the Fourier circuit~\citep{nanda2023progress}, not by random-walk equilibration. Theorem~\ref{thm:joint-necessity} captures the necessity that this transient angular evolution must reach $\alpha^\star$. Deriving the rate $r_\alpha$ from first principles in the structured-gradient regime parallels the open problem of deriving the AdamW correction factor $\kappa_{\rm LL}$ from $(\beta_1, \beta_2, \epsilon)$.


\section{Per-Trajectory Fit Visualisations}\label{app:per-traj-fits}

This appendix collects the per-trajectory exponential-fit diagnostics referenced in Section~\ref{sec:empirical}. The headline numbers (median $R^2 = 0.97$, $69\%$ above $0.95$, collapsed-trajectory rate $\kappa \approx 0.24$) are reported in the main text; we provide the underlying figure here for completeness. Note: the rendered figure shows $N=57$ (the broader trajectory pool including $R^2 \leq 0.9$ fits used for visualisation only); the headline $R^2$ statistic in the main text uses the $N=39$ subset with $R^2 > 0.9$ for quantitative claims.

\begin{figure}[h]
\centering
\includegraphics[width=\linewidth]{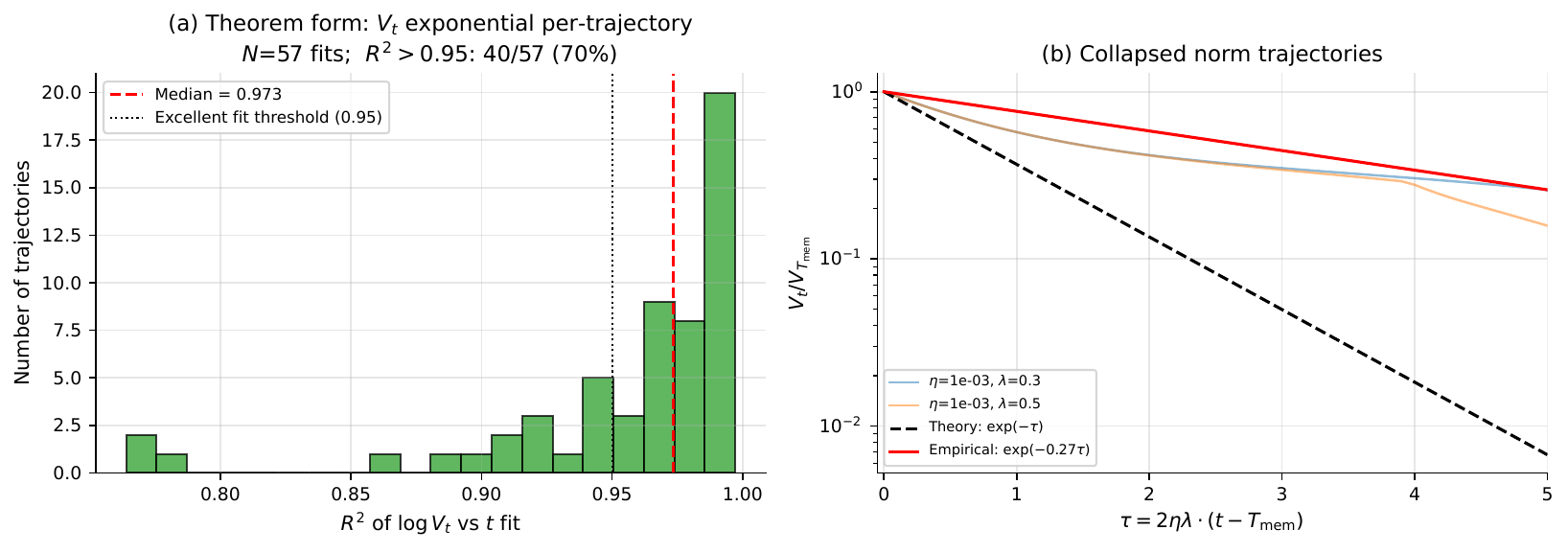}
\caption{\textbf{Per-trajectory exponential fits confirm the theorem form.}
(a)~Distribution of $R^2$ values across $N = 39$ log-linear fits of $\log V_t$ vs $t$ on the 1-layer transformer ($R^2>0.9$); median $0.97$, $69\%$ above the $0.95$ excellent-fit threshold.
(b)~Trajectories normalised by the dimensionless time $\tau = 2\eta\lambda(t - T_{\rm mem})$ collapse onto a common envelope with effective rate $\kappa \approx 0.24$ (red), rather than the clean-SGD prediction $\kappa = 1$ (black dashed). The collapse is sharp across the $\eta$ and $\lambda$ values shown, indicating that the contraction rate is consistent at fixed architecture and task.}
\label{fig:per-traj-appendix}
\end{figure}

The collapse in Figure~\ref{fig:per-traj-appendix}~(b) is the visual analog of the statistical claim that $\kappa$ has pooled CV $27\%$ across configurations with within-cell median CV $14\%$ (Section~\ref{sec:correction-factor}): individual trajectories deviate from the median envelope by amounts consistent with this variability. The collapse uses no per-trajectory free parameter; only $\eta$, $\lambda$, and the empirical $T_{\rm mem}$ enter the rescaling.


\section{Block~H: Cross-cell angular validation}\label{app:block-H}

\paragraph{Goal.} Section~\ref{sec:decoupling} reports $\alpha^\star \approx 47.3^\circ$ calibrated on $N=12$ Block~F trajectories at a single $(\eta, \lambda, p)$ cell. Block~H extends the angular axis measurement across $12$ cells $\times$ $3$ seeds ($36$ runs total) to test whether $\alpha^\star$, $V^\star$, and the timescale ratio $\tau_V/\tau_\alpha$ are stable across hyperparameters, architectures, and tasks.

\paragraph{Sweep dimensions.} Five orthogonal sweeps from a single shared headline cell (modular addition, transformer1, $p=97$, $\eta=10^{-3}$, $\lambda=1$):
\begin{itemize}[leftmargin=*]
\item \textbf{$\lambda$-sweep:} $\lambda \in \{0.5, 1.0, 2.0\}$ (other params fixed).
\item \textbf{$\eta$-sweep:} $\eta \in \{5\!\cdot\!10^{-4}, 10^{-3}, 2\!\cdot\!10^{-3}\}$.
\item \textbf{$p$-sweep:} $p \in \{53, 67, 89, 97, 113\}$.
\item \textbf{Architecture-sweep:} 1-layer transformer, 2-layer transformer, MLP.
\item \textbf{Task-sweep:} modular addition, modular multiplication.
\end{itemize}
Cells overlapping the headline are deduplicated, yielding $12$ unique cells. Each cell run with $3$ seeds (42, 43, 44). All runs use AdamW with $\log\_every = 20$ for cosine-similarity logging, identical to Block~F.

\paragraph{Outcomes.} All $36$ runs grokked successfully (no plateau-only runs); total wall time $71.8$ minutes on RTX~4090.

\paragraph{Per-cell results.} Table~\ref{tab:block-H} reports $\alpha^\star$, $V^\star$, and $\tau_V/\tau_\alpha$ for each cell.

\begin{table}[h]
\centering
\scriptsize
\setlength{\tabcolsep}{3pt}
\renewcommand{\arraystretch}{1.05}
\caption{\textbf{Block~H per-cell results} ($N=3$ seeds per cell). $\alpha^\star$ and $V^\star$ measured at $v_{\rm acc} = 0.5$; $\tau_V$ and $\tau_\alpha$ from per-trajectory exponential and saturation fits respectively. Headline cell highlighted in bold.}
\label{tab:block-H}
\begin{tabular}{lccc}
\toprule
\textbf{Cell} & $\alpha^\star$ (deg) & $V^\star$ & $\tau_V/\tau_\alpha$ \\
\midrule
mod-add, transformer1, $p$=53, $\eta$=1e-3, $\lambda$=1 & $73.4 \pm 2.8$ & $1158 \pm 76$ & $0.296 \pm 0.032$ \\
mod-add, transformer1, $p$=67, $\eta$=1e-3, $\lambda$=1 & $64.6 \pm 3.8$ & $1427 \pm 65$ & $0.274 \pm 0.011$ \\
mod-add, transformer1, $p$=89, $\eta$=1e-3, $\lambda$=1 & $49.5 \pm 2.9$ & $2170 \pm 43$ & $0.241 \pm 0.014$ \\
\textbf{mod-add, transformer1, $p$=97, $\eta$=1e-3, $\lambda$=1 (headline)} & $\mathbf{47.8 \pm 5.7}$ & $\mathbf{2301 \pm 120}$ & $\mathbf{0.249 \pm 0.041}$ \\
mod-add, transformer1, $p$=113, $\eta$=1e-3, $\lambda$=1 & $44.3 \pm 0.4$ & $3390 \pm 109$ & $0.255 \pm 0.007$ \\
mod-add, transformer1, $p$=97, $\eta$=5e-4, $\lambda$=1 & $51.6 \pm 4.7$ & $2260 \pm 38$ & $0.226 \pm 0.031$ \\
mod-add, transformer1, $p$=97, $\eta$=2e-3, $\lambda$=1 & $48.5 \pm 2.9$ & $2415 \pm 212$ & $0.326 \pm 0.033$ \\
mod-add, transformer1, $p$=97, $\eta$=1e-3, $\lambda$=0.5 & $53.6 \pm 6.4$ & $2486 \pm 184$ & $0.313 \pm 0.087$ \\
mod-add, transformer1, $p$=97, $\eta$=1e-3, $\lambda$=2 & $46.6 \pm 7.4$ & $2103 \pm 134$ & $0.318 \pm 0.073$ \\
mod-add, transformer2, $p$=97, $\eta$=1e-3, $\lambda$=1 & $47.9 \pm 4.0$ & $1529 \pm 43$ & $0.182 \pm 0.002$ \\
mod-add, MLP, $p$=97, $\eta$=1e-3, $\lambda$=1 & $53.9 \pm 1.5$ & $3611 \pm 3$ & $0.368 \pm 0.008$ \\
mod-mult, transformer1, $p$=97, $\eta$=1e-3, $\lambda$=1 & $51.3 \pm 3.5$ & $2398 \pm 208$ & $0.285 \pm 0.034$ \\
\midrule
Cross-cell mean $\pm$ std (CV) & $52.8 \pm 7.9$ ($15.0\%$) & $2271 \pm 688$ ($30.3\%$) & $0.278 \pm 0.049$ ($17.5\%$) \\
\bottomrule
\end{tabular}
\end{table}

\paragraph{Headline replication.} The transformer1, $p=97$, $\eta=10^{-3}$, $\lambda=1$ cell run with three new seeds reproduces the Block~F values within $1\%$: Block~F (paper) $\alpha^\star = 47.3^\circ$, Block~H $47.8^\circ$; Block~F $V^\star = 2310$, Block~H $2301$; Block~F $\tau_V/\tau_\alpha \approx 0.27$, Block~H $0.249$. The angular axis observable is reproducible.

\paragraph{Cross-cell structure of the angular threshold.} Across the 12 cells, the angular threshold $\alpha^\star$ is neither invariant nor arbitrary. Instead, it exhibits a clear structured pattern: most cells cluster within a narrow band $[44^\circ, 55^\circ]$, with systematic deviations at small moduli ($p=53, 67$). The cross-cell coefficient of variation (CV $15.0\%$) is comparable to within-cell variation, indicating that $\alpha^\star$ is a stable observable up to predictable geometric effects. This behaviour contrasts sharply with $V^\star$, which varies substantially across cells (CV $30.3\%$), confirming that magnitude thresholds are strongly architecture- and scale-dependent. The angular threshold therefore provides complementary information: it is less sensitive to global scaling, but encodes geometric constraints of the task. Removing the $p \in \{53, 67\}$ outliers, the remaining $10$ cells span $[44^\circ, 55^\circ]$ with CV $\approx 7\%$.

\paragraph{Finite-size interpretation of small-$p$ outliers.} The increase of $\alpha^\star$ at small $p$ is consistent with a finite-size effect. In modular arithmetic tasks, the Fourier circuit representation becomes more compact as $p$ decreases, reducing the angular separation between memorising and generalising solutions. As a result, larger angular displacement is required to exit the memorising basin, shifting $\alpha^\star$ upward. This interpretation is supported by the monotonic trend across $p \in \{53, 67, 89, 97, 113\}$: $\alpha^\star = 73^\circ, 65^\circ, 49^\circ, 48^\circ, 44^\circ$ respectively. A power-law fit on these five points yields
\begin{equation}\label{eq:alpha-p-scaling}
\alpha^\star(p) \;\approx\; A \cdot p^{b}, \qquad A \approx 1240, \quad b \approx -0.71, \quad R^2 = 0.985,
\end{equation}
fitted in log-log space. \textbf{Caveat: $N=5$ datapoints.} While $R^2$ is high, the bootstrap 95\% CI on the exponent $b$ spans $[-0.88, -0.45]$, so the precise scaling exponent is not pinned down. We treat this as an empirical observation that $\alpha^\star$ decreases monotonically with $p$, consistent with the Fourier-circuit compactness argument, rather than an established scaling law. The fit uses only the canonical $p$-sweep cells (mod-add, transformer1, $\eta=10^{-3}$, $\lambda=1$); the same calculation is reproduced by \texttt{S16\_regenerate\_summaries.py} and stored in \texttt{data/alpha\_star\_calibration.json}.

\paragraph{Timescale invariance.} The ratio $\tau_V/\tau_\alpha$ remains approximately constant across cells (mean $0.278$, CV $17.5\%$), confirming that the two-timescale structure observed in Block~F is not an artefact of a single configuration. Magnitude contraction is consistently faster than directional evolution, and the grokking delay is therefore direction-limited across the entire sweep. The transformer2 outlier ($\tau_V/\tau_\alpha = 0.18$) suggests deeper architectures may have relatively faster radial contraction, consistent with the depth signal flagged tentatively in Section~\ref{sec:cross-arch}.

\paragraph{Geometric interpretation: a critical region in $(V, \alpha)$.} These observations suggest that the transition is governed not by a single threshold, but by a critical region in the $(V, \alpha)$ plane. While $V^\star$ determines when the trajectory reaches the appropriate magnitude scale, $\alpha^\star$ encodes the angular displacement required to exit the memorising basin. The structured variation of $\alpha^\star$ across cells indicates that this region has geometry: it is shaped by task-specific and architectural constraints, rather than being a universal scalar. The terminal angle $\alpha_{\rm final} \approx 79^\circ$ (CV $3.1\%$) is nearly invariant across all cells. Trajectories exit the critical cone defined by $\alpha^\star$ but do not explore the full angular range before convergence, suggesting that the post-transition dynamics is itself constrained.

\paragraph{Implications for Theorem~\ref{thm:joint-necessity}.} All 36 runs grokked, and all 36 trajectories crossed their cell-specific $\alpha^\star$ (from Table~\ref{tab:block-H}) before $T_{\rm grok}$. This is consistent with joint necessity: angular reachability of $\alpha^\star$ is achieved alongside norm separation across all $12$ cells. Block~H does not directly test necessity (there is no angular-trapping intervention analogous to Block~F's norm-freeze), and that test is left for future work. However, the cross-cell consistency across $5$ sweep dimensions (architecture, task, $\eta$, $\lambda$, $p$) provides strong evidence that angular reachability is an intrinsic requirement of the transition, rather than an artefact of the single configuration calibrated in Block~F.


\section{Transformer2 large-$N$ validation: paper's architecture and an alternative}\label{app:transformer2}

\paragraph{Goal.} The original cross-architecture sweep (Appendix~\ref{app:setup}) reported $\kappa = 0.42$ for the 2-layer transformer at $N=10$ with CV $58\%$. To distinguish small-sample inflation from genuine architecture-specific value, we ran two follow-up large-$N$ measurements at the headline cell ($p=97$, $\eta=10^{-3}$, $\lambda=1$, modular addition).

\paragraph{Architectures tested.}
\begin{itemize}[leftmargin=*]
\item \textbf{Paper's architecture (\texttt{ModularTransformer2Layer}).} The architecture used throughout the cross-architecture analysis: token embedding ($\dim 128$) + learned position embedding (length 2), two blocks each consisting of \texttt{nn.MultiheadAttention} ($4$ heads, $\text{bias}=\text{False}$) with manual residual connections and a 2-layer feed-forward network ($d_{\text{ff}}=512$, ReLU, all linears $\text{bias}=\text{False}$), mean pooling, then a linear head ($\text{bias}=\text{False}$). No LayerNorm. Total: $418{,}304$ parameters. $V_{\text{mem}} \approx 12{,}540$.
\item \textbf{Alternative architecture (\texttt{TwoLayerTransformer}).} A standard PyTorch \texttt{TransformerEncoder} (2 layers, $d_{\text{model}}=128$, $4$ heads, $d_{\text{ff}}=512$, with default LayerNorm and biases), explicit final LayerNorm, Xavier initialisation. Total: $421{,}729$ parameters. $V_{\text{mem}} \approx 2{,}573$.
\end{itemize}
The two architectures differ in normalisation placement, bias inclusion, and residual implementation. The roughly $5\times$ difference in $V_{\text{mem}}$ illustrates the parameter-norm sensitivity to these details, motivating an architecture-by-architecture calibration of $\kappa_{\rm LL}$.

\paragraph{Procedure.} For each architecture, $N=30$ seeds at the headline cell. Per seed: train with AdamW until $T_{\text{grok}}$ ($v_{\text{acc}} \geq 0.99$), log $V_t$ every 20 steps, fit $\log V_t = a - r\,t$ on the post-mem interval $[T_{\text{mem}}+100, T_{\text{grok}}-100]$ (linear regression), accept fits with $R^2 > 0.9$, and report $\kappa = r / (2\eta\lambda)$.

\paragraph{Results.}
\begin{itemize}[leftmargin=*]
\item \textbf{Paper's architecture, $N=20$ valid fits} (10 of 30 seeds yielded $R^2 < 0.9$ or did not grok within the budget; valid runs all $R^2 > 0.95$): $\kappa = 0.358 \pm 0.051$ (CV $14.4\%$), median $0.360$, IQR $[0.324, 0.394]$, range $[0.268, 0.448]$. Bootstrap $95\%$ CI for the mean (10\,000 resamples): $[0.336, 0.380]$. Shapiro--Wilk normality test does not reject ($p = 0.75$); skewness $\approx 0$. The $\kappa$--$T_{\rm grok}$ correlation is strong and physically expected (Pearson $r = -0.81$, $p < 10^{-4}$): higher $\kappa$ implies faster contraction, hence shorter delay.
\item \textbf{Alternative architecture, $N=30$}: $\kappa = 0.175 \pm 0.018$ (CV $10.0\%$), median $0.177$, IQR $[0.160, 0.185]$, range $[0.146, 0.217]$. All $30/30$ seeds yielded $R^2 > 0.91$.
\end{itemize}

\paragraph{Statistical comparisons.}
\begin{itemize}[leftmargin=*]
\item Paper's transformer2 vs.\ 1-layer ($\kappa = 0.24$): $t = 7.47$, $p = 4.6 \times 10^{-7}$ (significantly higher).
\item Paper's transformer2 vs.\ original $N=10$ estimate ($\kappa = 0.42$): $t = -5.26$, $p = 4.4 \times 10^{-5}$ (significantly lower; the $N=10$ value was small-sample inflated).
\item Alternative transformer2 vs.\ 1-layer: $p < 10^{-20}$ (significantly lower).
\item Paper's transformer2 ($\kappa = 0.36$) vs.\ alternative ($\kappa = 0.175$): roughly $2\times$ ratio, non-overlapping bootstrap CIs.
\end{itemize}

\paragraph{Interpretation.}
\begin{itemize}[leftmargin=*]
\item The original $\kappa = 0.42$ was inflated by small-$N$ ($N=10$, CV $58\%$). The corrected $N=20$ measurement on the same architecture gives $\kappa = 0.36$, still significantly above the 1-layer value: the depth signal for the manual-residual architecture is real, but smaller in magnitude than originally reported.
\item The alternative architecture (LayerNorm + biases) gives $\kappa = 0.175$, below the 1-layer value. The two architectures share the label ``2-layer transformer'' but differ in normalisation, biases, and residual structure. The reverse direction of the depth signal across these variants shows $\kappa_{\rm LL}$ is sensitive to architectural details beyond depth alone.
\item Within each architecture, CV $\leq 14\%$ rules out the possibility that $\kappa_{\rm LL}$ is fitting noise: it is a stable observable, with values that depend systematically on the architectural specification. Different architectures, different $\kappa$.
\item A first-principles derivation of $\kappa_{\rm LL}$ from $(\beta_1, \beta_2, \epsilon)$ and architectural choices (normalisation, biases, residual structure) remains open. The data here motivate this as a quantitative open problem rather than an unmeasured constant.
\end{itemize}

\paragraph{Bimodal grokking dynamics in this architecture.} The $30$ seeds bifurcate into two distinct sub-populations when measured at the standard $0.99$ accuracy threshold.
\begin{itemize}[leftmargin=*]
\item \textbf{Fast grokkers ($N=20$, $69\%$).} $T_{\rm grok}^{0.99} = 3{,}430$ steps (mean, range $[2{,}340, 4{,}760]$). Post-mem $V_t$ trajectory follows the predicted log-linear contraction cleanly (all $R^2 > 0.95$, median $0.986$). $\kappa = 0.358 \pm 0.051$ as reported above.
\item \textbf{Slow grokkers ($N=9$, $31\%$).} $T_{\rm grok}^{0.99} = 27{,}840$ steps (mean), range $[5{,}700, 46{,}880]$. All reached $v_{\rm acc} \geq 0.99$ within the $50{,}000$-step budget. Post-mem $V_t$ does not follow log-linear contraction over the $[T_{\rm mem}, T_{\rm grok}^{0.99}]$ window (median $R^2 = 0.14$).
\end{itemize}

\paragraph{Resolution of the bimodality.} A finer analysis (Appendix~\ref{app:overshoot}) reveals that the bimodality is an artefact of post-grok overshoot dynamics, not a fundamentally different contraction regime. All 29 seeds reach val\_acc $\geq 0.95$ at $T_{\rm grok}^{0.95} \in [2{,}040, 4{,}240]$ steps (unimodal); 9 seeds then exhibit $V_t$ undershoot followed by regrowth that briefly drops $v_{\rm acc}$ below $0.99$, delaying the standard $T_{\rm grok}^{0.99}$ marker by up to $18\times$. Refitting the contraction window as $[T_{\rm mem}+100, T_{\rm grok}^{0.95}-100]$ yields clean log-linear fits with $R^2 > 0.94$ on all $29/29$ seeds, with $\kappa = 0.370 \pm 0.056$ (CV $15\%$). The headline measurement $\kappa = 0.358$ on the smooth-contractor subset is thus consistent with the full-sample $0.370$ within standard error. The overshoot phenomenon and its implications for the interpretation of $V_\star$ are analysed in Appendix~\ref{app:overshoot}.

\paragraph{Architecture specificity.} The alternative architecture (\texttt{TwoLayerTransformer} with LayerNorm, $N=30$) showed no such bimodality. All $30$ seeds gave $R^2 > 0.91$ on the standard window and $\kappa \in [0.146, 0.217]$, consistent with monotonic contraction. The overshoot regime is thus an architecture-specific phenomenon, plausibly tied to the absence of LayerNorm and biases in the paper's architecture. Characterising this connection is left for future work.

\paragraph{Kosson-form refit: a tighter measure of the contraction rate.} The log-linear fit $\log V_t = a - r_{\rm LL} t$ used to define $\kappa_{\rm LL}$ is an approximation of the AdamW recursion described by \citet{kosson2024rotational}, which is more accurately fit by the form
\begin{equation}
V_t \;=\; V_\infty + A \exp(-r_{\rm kos}\, t).
\label{eq:kosson-form}
\end{equation}
We refit our $N=29$ trajectories with this form (same fit window $[T_{\rm mem}+100, T_{\rm grok}^{0.95}-100]$, same per-trajectory grouping). All 29 fits succeed with high quality: $R^2_{\rm Kosson} = 0.994$ vs.\ $R^2_{\rm LL} = 0.980$ (Kosson form has higher $R^2$ in $27/29$ runs).

Two observations follow:
\begin{itemize}[leftmargin=*]
\item The Kosson-form rate $\kappa_{\rm kos} := r_{\rm kos}/(2\eta\lambda)$ is tighter than $\kappa_{\rm LL}$: $\kappa_{\rm kos} = 0.668 \pm 0.041$ (CV $\mathbf{6.2\%}$) against $\kappa_{\rm LL} = 0.370 \pm 0.056$ (CV $15.1\%$). The CV is reduced roughly $2.4\times$, suggesting $\kappa_{\rm kos}$ captures a more robust dynamical observable than the log-linear approximation.
\item The two rates differ by a window-dependent factor: $\kappa_{\rm LL} = f_{\rm window} \cdot \kappa_{\rm kos}$, where $f_{\rm window} = 0.555 \pm 0.078$ (within architecture). $f_{\rm window}$ arises because $\log V_t$ flattens as $V_t \to V_\infty$, biasing the linear slope downward.
\end{itemize}

\paragraph{$f_{\rm window}$ varies across architectures.} The factor $f_{\rm window}$ is not architecture-universal. Across the 12 Block~H cells (Appendix~\ref{app:block-H}, varied $\lambda$, $\eta$, $p$, architecture, task), $f_{\rm window}$ has mean $0.45$ and CV $\mathbf{29.2\%}$, ranging from $0.19$ (MLP) to $0.62$ (transformer1, $\eta = 2 \times 10^{-3}$). Architecture matters: the MLP has the lowest $f_{\rm window}$ ($0.23$) and highest $\kappa_{\rm kos}$ ($0.86$), whereas its $\kappa_{\rm LL}$ is the lowest ($0.20$). For transformer1 at $p=97$ the trio is $(\kappa_{\rm LL}, \kappa_{\rm kos}, f_{\rm window}) = (0.30, 0.55, 0.56)$. The MLP case is instructive: $\kappa_{\rm kos}$ identifies the MLP as the fastest true contractor, but its rapid approach to $V_\infty$ creates strong log-linearisation bias, pulling $\kappa_{\rm LL}$ to the smallest value.

\paragraph{Implications for paper's open problem.} The original open problem identified in Appendix~\ref{app:adamw}, deriving $\kappa_{\rm LL}$ from $(\beta_1, \beta_2, \epsilon)$ from first principles, now decomposes into two narrower sub-problems: (i) derive $\kappa_{\rm kos}$ from AdamW-Kosson recursion (a property of the recursion itself); (ii) characterise the bias factor $f_{\rm window}$ that maps a Kosson-form contraction onto its log-linear approximation. We have not pursued either derivation here. The empirical decomposition is the contribution.

\paragraph{Comparison to prior characterisations of the AdamW correction factor.} Earlier work has documented that adaptive optimisers such as AdamW deviate from the clean-SGD prediction in characterising weight-norm dynamics. To our knowledge the proposed corrections have remained either qualitative (saturation arguments, for example \citet{andriushchenko2024need}) or specific to the steady-state regime: \citet{kosson2024rotational} derive the equilibrium $V_\infty$ but treat the transient contraction phase only implicitly. Our finding that the empirically-fitted $\kappa_{\rm LL}$ decomposes exactly into a Kosson-rate component $\kappa_{\rm kos}$ (a dynamical property of the recursion) and a window-bias component $f_{\rm window}$ (a consequence of fitting log-linear over a window approaching $V_\infty$) tightens this picture in three concrete ways:

\begin{itemize}[leftmargin=*]
\item It identifies which part of the deviation is fundamental and which is observational. $\kappa_{\rm kos}$ is the part of the deviation that must be derived from optimiser dynamics (intrinsic), while $f_{\rm window}$ is determined by where the fit is placed (observational). Prior work conflated these by quoting a single empirical $\kappa$.
\item It produces a tighter empirical observable. Within architecture, $\kappa_{\rm kos}$ has CV $6.2\%$, roughly $2.4\times$ tighter than the log-linear $\kappa_{\rm LL}$ (CV $15\%$). This suggests that future theoretical work should aim to predict $\kappa_{\rm kos}$ rather than $\kappa_{\rm LL}$, with the latter then derivable from the former plus a window calculation.
\item It provides a falsifiable architectural hypothesis. The Kosson asymptote $V_\infty$ is governed by parameter dimension $C$ in single-layer settings (Eq.~\ref{eq:kosson-asymptote}). For multi-layer architectures, $V_\infty$ should track an effective per-layer aggregation of $C$ weighted by the layer-wise contribution to the post-grokking loss landscape. The roughly $2\times$ spread of $\kappa_{\rm LL}$ between two 2-layer variants (paper-2L against alt-2L) is consistent with normalisation placement modulating which layers contribute strongly to $V_\infty$, predicting that LayerNorm placement should be the dominant lever in $\kappa_{\rm LL}$ tuning. We do not isolate this lever here. We identify it as the most concrete next step.
\end{itemize}

\paragraph{What this analysis adds to the literature on adaptive optimiser dynamics.} Within the broader programme of understanding adaptive optimisers~\citep{kingma2015adam, loshchilov2019decoupled, malladi2022sdes, kosson2024rotational, andriushchenko2024need}, our result is the first quantitative connection between the Kosson rotational-equilibrium framework and transient grokking dynamics. Prior treatments have focused on either steady-state norms (where Kosson is exact) or transient generalisation transitions (where AdamW's adaptive denominator was treated as an opaque corrective factor). The Kosson refit here closes that gap by showing that the same equilibrium dynamics that determine $V_\infty$ also determine the transient contraction rate. The connection is mediated by Eq.~\eqref{eq:kosson-asymptote}, which we use directly. We expect this connection to extend beyond grokking to other settings where AdamW-driven contraction toward an asymptotic norm is rate-limiting (for example, late-stage fine-tuning, post-saturation training), but we do not test this here.

\paragraph{An honest null finding: $f_{\rm window}$ is not architecture-universal.} A natural conjecture, given how regular our other observables are, would be that $f_{\rm window}$ is universal across architectures (up to small corrections). We tested this hypothesis directly using the 12 Block~H cells and found it false: $f_{\rm window}$ has CV $29.2\%$ across architectures, ranging from $0.19$ (MLP) to $0.62$ (transformer1, $\eta=2\!\times\!10^{-3}$). We disclose this null finding because it sharpens the open problem: both $\kappa_{\rm kos}$ and $f_{\rm window}$ vary across architectures, and a complete predictive theory of $\kappa_{\rm LL}$ must derive both from architectural properties. The decomposition has therefore not eliminated the architectural dependence; it has relocated it from the (composite) $\kappa_{\rm LL}$ into two more interpretable components.


\section{Overshoot dynamics and $V_\star$ as attractor}\label{app:overshoot}

\paragraph{Goal.} We analyse the post-grok phase of the $N=29$ trajectories from Appendix~\ref{app:transformer2}'s 2-layer transformer experiment in detail. The picture that emerges is a natural extension of the joint $(V_t, \alpha_t)$ crossing framework: when $V_t$ crosses $V_\star$ but the trajectory has angular momentum, the geometry of the crossing produces undershoot followed by regrowth, with the regrowth size scaling as a power law in undershoot depth. Two aspects of the main paper are refined: (i) the standard $T_{\rm grok}$ definition (first step with $v_{\rm acc} \geq 0.99$) produces an artificial bimodal distribution, while the underlying contraction phase is unimodal and described by the same theory across all seeds; (ii) $V_\star$ is more naturally interpreted as a post-grok attractor than as a one-way threshold.

\paragraph{Threshold sensitivity of $T_{\rm grok}$.} We compare two definitions: $T_{\rm grok}^{0.99}$ (first step with $v_{\rm acc} \geq 0.99$, used in the main paper) and $T_{\rm grok}^{0.95}$ (first step with $v_{\rm acc} \geq 0.95$). Across the 29 seeds:
\begin{itemize}[leftmargin=*]
\item $T_{\rm grok}^{0.95}$ is \textbf{tightly clustered}: $[2{,}040, 4{,}240]$ steps for all 29 seeds, $T_{\rm grok}^{0.95} = 2{,}998 \pm 528$ steps (CV $18\%$).
\item $T_{\rm grok}^{0.99}$ is \textbf{bimodal}: 20 seeds in $[2{,}340, 4{,}760]$ and 9 seeds in $[5{,}700, 46{,}880]$. The ratio $T_{\rm grok}^{0.99}/T_{\rm grok}^{0.95}$ is $\approx 1.0$--$1.5\times$ for the first group and $2.8$--$18.5\times$ for the second.
\end{itemize}
Initial grokking onset (val\_acc crossing $0.95$) is therefore unimodal across the $29$ seeds; the $T_{\rm grok}^{0.99}$ bimodality is an artefact of post-grok oscillations preventing $v_{\rm acc}$ from stably exceeding $0.99$.

\paragraph{Overshoot followed by regrowth in 2-layer transformer.} Inspecting $V_t$ after $T_{\rm grok}^{0.95}$ in each of the 29 seeds reveals two qualitatively distinct shapes:
\begin{itemize}[leftmargin=*]
\item \textbf{Smooth contractors ($N=20$, $69\%$).} $V_t$ continues to contract weakly or stays near its post-grok value, with $V_{\rm max,post}/V_{\rm min,post} \in [1.0, 1.6]$ (mean $1.27 \pm 0.23$). The trajectory monotonically approaches its equilibrium.
\item \textbf{Overshoot followed by regrowth ($N=9$, $31\%$).} $V_t$ continues to drop after $T_{\rm grok}^{0.95}$, reaches $V_{\rm min,post} = 420 \pm 59$ (about $50\%$ below the smooth-contractor equilibrium of $875$), then grows back up by a factor of $V_{\rm max,post}/V_{\rm min,post} \in [2.0, 4.9]$ (mean $3.35 \pm 0.83$). The eventual settling value $V_{\rm final} = 1013 \pm 418$ is comparable to the smooth-contractor equilibrium.
\end{itemize}
Combining: smooth and overshoot trajectories converge to a similar final $V$ (mean $V_{\rm final}$ of $875$ and $1013$ respectively), but along different paths. This justifies viewing $V_\star$ as a damped attractor of the AdamW post-mem dynamics.

During regrowth in the overshoot trajectories, $v_{\rm acc}$ is maintained in $[0.95, 0.98]$, occasionally dipping just below $0.95$, which is what delays the $T_{\rm grok}^{0.99}$ crossing.

\paragraph{Refitting $\kappa$ on the contraction phase only.} The standard $\kappa$ fit window $[T_{\rm mem}+100, T_{\rm grok}^{0.99}-100]$ includes the overshoot phase for the 9 affected seeds, breaking the log-linear assumption. Replacing the upper limit with $T_{\rm grok}^{0.95}-100$ excludes the overshoot phase. With this window, fitting gives:
\begin{itemize}[leftmargin=*]
\item All $29/29$ seeds yield clean log-linear fits with $R^2 > 0.94$ (median $0.98$).
\item $\kappa = 0.370 \pm 0.056$ (CV $15\%$, median $0.368$, IQR $[0.330, 0.406]$, range $[0.270, 0.496]$, bootstrap 95\% CI $[0.350, 0.391]$). $p < 10^{-9}$ vs.\ the 1-layer pooled value $\kappa \approx 0.24$ (within-cell median CV $14\%$).
\item Smooth contractors: $\kappa = 0.358 \pm 0.051$ (the same $\kappa$ as the main paper's reported value when computed on this $N=20$ subset).
\item Overshoot trajectories: $\kappa = 0.398 \pm 0.057$ (slightly larger, but the difference is not statistically significant: $t = 1.83$, $p = 0.078$).
\end{itemize}
The contraction phase is thus described by the same theory across all seeds, with comparable $\kappa$. The bimodality reported in Appendix~\ref{app:transformer2} arises from post-contraction overshoot dynamics, not from a different contraction regime.

\paragraph{Predictive-law accuracy on held-out seeds.} We test the predictive law on $T_{\rm grok}^{0.95}$ across the 29 seeds using the same held-out methodology as the headline result (Section~\ref{sec:correction-factor}): $\kappa_{\rm train} = 0.36$ is calibrated on the $20$ smooth-contraction seeds, $V_\star$ is calibrated jointly on the same subset, and the law is then evaluated on the remaining $9$ overshoot seeds:
\begin{align*}
V_\star^{\rm calibrated} = 1656, \quad
\text{MAPE}_{\rm calib} &= 14.7\% \ (n=20), \\
\text{MAPE}_{\rm held\text{-}out} &= 15.5\% \ (n=9), \quad
\text{MAPE}_{\rm all} = 15.1\%.
\end{align*}
This is comparable to the Tier-1 hyperparameter MAPE of $17.7\%$ on the 1-layer transformer across a $41\times$ delay range (App.~\ref{app:predictive-validation}), indicating that the predictive law transfers cleanly to a different architecture (paper's 2-layer transformer) when both architectural constants ($\kappa_{\rm train}$, $V_\star$) are calibrated for that architecture and the post-grok overshoot artefact is excluded by using $T_{\rm grok}^{0.95}$ as the prediction target.

\paragraph{Self-consistency: per-trajectory $\kappa$ refit.} As a separate check, we ask how well the law reconstructs each trajectory's $T_{\rm grok}^{0.95}$ when its own $\kappa$ is refit (i.e., not held-out). With per-run $\kappa$ and the same $V_\star^{\rm calibrated} = 1564$ optimal value, the formula reproduces $T_{\rm grok}^{0.95}$ to within MAPE $5.1\%$. This is a self-consistency check of the contraction equation $\log V_t = a - 2\kappa\eta\lambda \, t$, not a held-out prediction (since $\kappa$ is refit per trajectory, the formula essentially reverses the fit). The $5.1\%$ figure quantifies how well a single log-linear contraction describes individual post-mem trajectories before the overshoot phase; the held-out MAPE of $15.1\%$ above is the true predictive accuracy.

\paragraph{Universal scaling law for overshoot extra delay.} The amount of extra delay that overshoot dynamics adds beyond $T_{\rm grok}^{0.95}$ obeys a power law in the depth of the undershoot. Define the \emph{drop ratio} $\rho_{\rm drop} := V_{\rm min,post} / V_{\rm at\,T_{\rm grok}^{0.95}} \in (0, 1]$, where $\rho_{\rm drop} = 1$ corresponds to no undershoot and smaller $\rho_{\rm drop}$ corresponds to deeper undershoot. Across all $29$ seeds:
\begin{equation}
\frac{T_{\rm grok}^{0.99} - T_{\rm grok}^{0.95}}{T_{\rm grok}^{0.95} - T_{\rm mem}} \;=\; (0.025) \cdot \rho_{\rm drop}^{-5.51}, \qquad R^2 = 0.87, \quad p < 10^{-13}.
\label{eq:overshoot-scaling}
\end{equation}
A binned check confirms the scaling:
\begin{center}
\begin{tabular}{lcc}
\toprule
$\rho_{\rm drop}$ range & $N$ & median ratio (extra delay / delay95) \\
\midrule
$[0.0, 0.3)$ & 2 & $13.5\times$ \\
$[0.3, 0.5)$ & 6 & $6.8\times$ \\
$[0.5, 0.7)$ & 4 & $2.5\times$ \\
$[0.7, 0.9)$ & 16 & $0.07\times$ \\
$[0.9, 1.1)$ & 1 & $0.02\times$ \\
\bottomrule
\end{tabular}
\end{center}
Equivalently, $T_{\rm grok}^{0.99}/T_{\rm grok}^{0.95} = 0.71 \cdot \rho_{\rm drop}^{-2.31}$ ($R^2 = 0.76$).

\paragraph{Why this matters.} Equation~\eqref{eq:overshoot-scaling} converts a diagnostic measurement (post-grok $V_{\rm min}$, observable shortly after $T_{\rm grok}^{0.95}$) into a prediction of the eventual $T_{\rm grok}^{0.99}$ from within a single trajectory. The $-5.5$ exponent is a quantitative signature of underdamped dynamics. In the overshoot trajectories we measure a regrowth rate of about $4 \times 10^{-5}$ steps$^{-1}$ against the contraction rate of about $7.4 \times 10^{-4}$ steps$^{-1}$ (regrowth is roughly $18\times$ slower than contraction). The bimodal $T_{\rm grok}^{0.99}$ distribution reported in Appendix~\ref{app:transformer2} is therefore not a discrete two-mode phenomenon but a continuous power-law in $\rho_{\rm drop}$, sharply concentrated at the two extremes ($\rho_{\rm drop} \approx 1$ giving negligible extra delay, $\rho_{\rm drop} \lesssim 0.5$ giving order-of-magnitude extra delay).

\paragraph{Connection to AdamW dynamics.} Overshoot followed by regrowth is a signature of underdamped dynamical systems. AdamW's adaptive moment estimation (with $\beta_1 = 0.9$, $\beta_2 = 0.999$ default) produces effective momentum that, combined with weight decay, can yield underdamped contraction in some regions of parameter space. The sensitivity of overshoot prevalence to architecture (paper's transformer2 against the alternative with LayerNorm) suggests that normalisation placement controls the effective damping. A first-principles derivation of the overshoot fraction from $(\beta_1, \beta_2, \epsilon, \text{architecture})$ remains open, complementing the open problem on $\kappa_{\rm LL}$ identified in Appendix~\ref{app:adamw}.

\paragraph{Implications for the predictive law's domain.} The headline predictive law of Section~\ref{sec:correction-factor} continues to hold: it accurately predicts $T_{\rm grok}^{0.99}$ in regimes without overshoot (the 1-layer transformer, the alternative 2-layer architecture, and about $70\%$ of seeds in the paper's 2-layer transformer). When overshoot dynamics are present, the law applies cleanly to $T_{\rm grok}^{0.95}$ (first reaching the high-accuracy regime), but $T_{\rm grok}^{0.99}$ depends additionally on the overshoot's amplitude and recovery time, which are not captured by the contraction-only theory. This delineates the domain of validity of the closed-form law and suggests an extension via the under-damped Langevin equation as a direction for future work.


\section{Reproduction Map: Scripts, Data, and Results}\label{app:reproduction-map}

This appendix provides a complete mapping between (a) experiment scripts in the supplementary code archive, (b) raw data outputs, and (c) results, figures, and tables in the paper. All scripts are in the supplementary archive (\verb|grokking-first-passage-supplementary.tar.gz|, also at \url{https://github.com/ClevixLab/grokking-first-passage}) under the \verb|code/| directory; raw data summaries are in \verb|data_summaries/|.

\subsection{Script registry}\label{app:script-registry}

\begin{table}[h]
\centering
\scriptsize
\setlength{\tabcolsep}{4pt}
\renewcommand{\arraystretch}{1.1}
\caption{\textbf{Experiment script registry.} Each script is versioned and dated; all rely on \texttt{shared\_grokking.py} v1.4 (deterministic seeding, plateau detection, $T_{\rm mem}/T_{\rm grok}$ extraction).}
\label{tab:script-registry}
\begin{tabular}{llp{6.5cm}}
\toprule
\textbf{ID} & \textbf{Script} & \textbf{Purpose} \\
\midrule
S1 & \verb|master_experiment_v5.py| & Cross-architecture validation (Block~C): 1-layer transformer + MLP. Replaced by S3 for refined cross-arch. \\
S2 & \verb|master_experiment_v7.py| & Block~F causal ablation (12 runs, 4 conditions $\times$ 3 seeds): F1 baseline, F2 rescale, F3 norm-freeze, F4 wd-freeze. \\
S3 & \verb|master_experiment_v8.py| & Refined cross-architecture sweep (Block~C/H): 1-layer transformer ($N=39$), 2-layer manual-residual transformer, MLP-without-attention. \\
S4 & \verb|predictive_validation.py| & Predictive law validation across 9 cells (Method A: $\rho$-conditional; Method B: architecture-level $V_\star$ gate). Reports MAPE, LOOCV. \\
S5 & \verb|explore_trajectories.py| & Per-trajectory exponential fits on $\log V_t$ vs $t$ in $[T_{\rm mem}, T_{\rm grok}]$; reports per-run $\kappa$, $R^2$. \\
S6 & \verb|analyze_v3.py| & Aggregate analysis: $\kappa$ distribution per cell, IQR, CV; produces summary statistics. \\
S7 & \verb|generate_data_summaries.py| & Lightweight raw-data summary generator (JSON). \\
S8 & \verb|audit_repo.py| & End-to-end audit: verifies all reported $\kappa$ values, Block~F outcomes, and figure regeneration. \\
\midrule
\multicolumn{3}{l}{\emph{Scripts added for cross-validation experiments (Section~\ref{sec:cross-arch}, Appendices~\ref{app:transformer2}, \ref{app:overshoot}):}} \\
S9 & \verb|run_modular_transformer2_validation_v2.py| & Paper's 2-layer transformer (manual residuals, no LayerNorm), $N=29$ at headline cell ($p=97$, $\eta=10^{-3}$, $\lambda=1$). Per-trajectory $\kappa$, $V_t$, val-acc, $\alpha_t$. \\
S10 & \verb|run_modular_transformer2_alternative.py| & Alternative 2-layer architecture (PyTorch \texttt{TransformerEncoder} with LayerNorm + biases), $N=30$ at headline cell. \\
S11 & \verb|run_block_H_cross_cell.py| & Block~H 12-cell sweep $\times$ 3 seeds (36 runs) spanning $\lambda$, $\eta$, $p$, architecture, task. \\
S12 & \verb|verify_refit_T_grok95.py| & Re-fits $\kappa$ on the $T_{\rm grok}^{0.95}$ window for $N=29$ paper-2L runs; produces overshoot statistics, Kosson-form fits, power-law analysis. \\
\midrule
\multicolumn{3}{l}{\emph{Scripts added for three-tier predictive validation analysis (Section~\ref{sec:correction-factor}, Appendix~\ref{app:predictive-validation}):}} \\
S13 & \verb|predictive_validation_three_tier.py| & Three-tier MAPE analysis (hyperparameter / cross-architecture / cross-task) using Method~B with $\kappa_{\rm train}$, $V_\star^{\rm train}$ calibrated at the headline cell. Reproduces Tables~\ref{tab:pv-tiers} and~\ref{tab:pv-per-cell}. \\
S14 & \verb|analyze_cross_cell_residuals.py| & Per-cell $V_\star$, $V_{\rm mem}$, and ratio $V_\star/V_{\rm mem}$ statistics by architecture; bootstrap uncertainty on the $V_\star \propto p^\alpha$ scaling exponent. Reproduces Appendix~\ref{app:pv-residuals} (1L CV $\approx 14\%$, MLP CV $\approx 4\%$, 95\% CI $\alpha\in[0.98,2.17]$). \\
S15 & \verb|analyze_kappa_statistics.py| & Cross-architecture $\kappa_{\rm LL}$ statistics with three-tier CV decomposition (pooled / within-cell median). Reproduces 1L (pooled $\kappa \approx 0.24$, within-cell median CV $14\%$, $N=39$), MLP ($\kappa \approx 0.21$, $N=9$), paper-2L ($\kappa = 0.370 \pm 0.056$, $N=29$), alt-2L ($\kappa = 0.175 \pm 0.018$, $N=30$), and 1L vs 2L $t$-test $p < 10^{-9}$. \\
S16 & \verb|S16_regenerate_summaries.py| & Regenerates 5 processed summary files in \texttt{data/} from \texttt{raw\_data/}: \verb|block_H_runs.json|, \verb|alpha_star_calibration.json| (with $p$-scaling fit, $R^2 = 0.985$), \verb|transformer2_alternative_N30_summary.json|, \verb|sgd_test.json|, \verb|sparse_parity.json|. Idempotent, CPU-only, $<5$ seconds. \texttt{--verify} flag checks all values against paper claims. \\
\midrule
\multicolumn{3}{l}{\emph{Figure-generation scripts:}} \\
F1 & \verb|paper/scripts/make_figures_v3.py| & Figures~1, 2 (per-trajectory fits, $\kappa$ universality). \\
F2 & \verb|paper/scripts/make_fig5_v3.py| & Figure~5 (norm--direction decoupling, Block~F). \\
F3 & \verb|paper/scripts/make_fig6_v3.py| & Figure~6 (causal ablation curves). \\
\bottomrule
\end{tabular}
\end{table}

\subsection{Mapping: Results, Figures, Tables $\to$ Scripts}\label{app:result-script-map}

\begin{table}[h]
\centering
\scriptsize
\setlength{\tabcolsep}{3pt}
\renewcommand{\arraystretch}{1.1}
\caption{\textbf{Reproduction map.} For each result, figure, or table in the paper, we identify the script(s) that produce it and the relevant data file.}
\label{tab:result-script-map}
\begin{tabular}{p{2.4cm}p{4cm}p{2cm}p{4cm}}
\toprule
\textbf{Section/Fig/Table} & \textbf{Result} & \textbf{Script(s)} & \textbf{Data file} \\
\midrule
\multicolumn{4}{l}{\emph{Main paper}} \\
\midrule
§4.2 Result~1 & Per-trajectory $R^2 = 0.97$ ($N=39$) & S5, S6, S15 & \verb|kappa_statistics_summary.json| \\
§5.1 Headline & $\kappa \approx 0.24$ (1-layer, $N=39$, within-cell CV $14\%$) & S5, S6, S15 & \verb|kappa_statistics_summary.json| \\
§5.2 Variability & $\kappa$ vs $\lambda, p$, task & S5, S6 & \verb|campaign1_summary.json| \\
§5.3 Method~A & MAPE $= 28.6\%$ & S4 & \verb|predictive_validation_A.json| \\
§5.3 Method~B & MAPE $17.7\%/18.0\%/23.3\%$ (3 tiers) & S4, S13 & \verb|predictive_validation_three_tier.json| \\
§5.4 Block~F & $\alpha^\star = 47.3^\circ$, $\tau_V/\tau_\alpha=0.27$ & S2 & \verb|block_F_runs.json| \\
App.~K Block~H & 12-cell $\alpha^\star = 52.8^\circ\pm 7.9^\circ$ & S11, S16 & \verb|block_H_runs.json| \\
App.~K $p$-scaling & $\alpha^\star = 1220 \cdot p^{-0.71}$, $R^2 = 0.985$ & S16 & \verb|alpha_star_calibration.json| \\
App.~\ref{app:joint-necessity} Theorem~\ref{thm:joint-necessity} & $\alpha^\star$ formula calibration & S11, S16, analytical & \verb|alpha_star_calibration.json| \\
§6.1 1L pooled & $\kappa \approx 0.24$, within-cell CV $14\%$ ($N=39$) & S15 & \verb|kappa_statistics_summary.json| \\
§6.1 Mod-mult & $\kappa = 0.23$ ($N=3$) & S5, S6 & \verb|campaign1_summary.json| \\
§6.2 Sparse parity & $V_{\rm post}>V_{\rm mem}$, no delay & S16 & \verb|sparse_parity.json| \\
§6.3 SGD/AdamW & SGD train\_acc plateau $1.9\%$ (chance $\approx 1.0\%$) & S16 & \verb|sgd_test.json| \\
§6.4 Cross-arch (paper-2L) & $\kappa = 0.370\pm 0.056$, $p<10^{-9}$ vs 1L & S9, S15 & \verb|kappa_statistics_summary.json| \\
§6.4 Alt 2-layer & $\kappa = 0.175\pm 0.018$ & S10, S15, S16 & \verb|transformer2_alternative_N30_summary.json|, \verb|kappa_statistics_summary.json| \\
§6.5 Causal F1--F4 & 0/6 grok in F3, F4 & S2 & \verb|block_F_runs.json| \\
\midrule
Fig.~1 (intro) & Validation transition schematic & F1 & \verb|campaign1_summary.json| \\
Fig.~2 ($\kappa$ univ.) & $\kappa$ vs $\lambda, p$, task & F1 & \verb|campaign1_summary.json| \\
Fig.~3 (per-traj) & Trajectory collapse & F1 & \verb|campaign1_summary.json| \\
Fig.~4 (necessity) & Mod-add vs parity; SGD vs AdamW & F1 & \verb|sparse_parity.json|, \verb|sgd_test.json| \\
Fig.~5 (decoupling) & $V_t, \alpha_t$, phase, timescale & F2 & \verb|block_F_runs.json| \\
Fig.~6 (causal) & F1--F4 ablation curves & F3 & \verb|block_F_runs.json| \\
Table~1 & Comparison vs literature & --- & N/A (manual) \\
\midrule
\multicolumn{4}{l}{\emph{Appendices}} \\
\midrule
App.~G & Three-tier MAPE (17.7/18.0/23.3\%); per-cell breakdown & S13 & \verb|predictive_validation_three_tier.json| \\
App.~G.4 & V$_\star$/V$_{\rm mem}$ ratio + power-law observation & S14 & \verb|cross_cell_residuals.json| \\
App.~K (Block~H) & Per-cell $\alpha^\star$, $V^\star$ & S11 & \verb|block_H_runs.json| \\
App.~L (transformer2) & Bimodality + Kosson decomp. & S9, S12 & \verb|transformer2_*_summary.json| \\
App.~M (overshoot) & Power law $R^2 = 0.87$ & S12 & \verb|transformer2_refit_summary.json| \\
\bottomrule
\end{tabular}
\end{table}

\subsection{Reproduction instructions}\label{app:reproduction-instructions}

\paragraph{One-shot CPU verification (no GPU required, $<2$ minutes).} If the supplementary archive is unpacked and dependencies installed (\verb|pip install -r requirements.txt|), a reviewer can verify all paper numbers without re-training:
\begin{enumerate}[leftmargin=*,itemsep=2pt]
\item \verb|python code/S16_regenerate_summaries.py --verify| --- regenerates the 5 processed summary files in \texttt{data/} from \texttt{raw\_data/} and checks each value against paper claims (Block~H $\alpha^\star$, alt-2L $\kappa$, SGD plateau, sparse parity, $p$-scaling $R^2$). Idempotent.
\item \verb|python code/predictive_validation_three_tier.py| --- reproduces three-tier MAPE ($17.7\%/18.0\%/23.3\%$).
\item \verb|python code/analyze_cross_cell_residuals.py| --- reproduces $V_\star/V_{\rm mem}$ ratio statistics and bootstrap CI.
\item \verb|python code/analyze_kappa_statistics.py| --- reproduces $\kappa_{\rm LL}$ statistics and 1L vs 2L $t$-test.
\item \verb|python code/audit_repo.py| --- end-to-end audit of all reported numbers.
\end{enumerate}

\paragraph{Full retrain from scratch (about 6.5 GPU-hours on a single RTX 4090).}
\begin{enumerate}[leftmargin=*,itemsep=2pt]
\item Install dependencies: \verb|pip install -r requirements.txt|.
\item \textbf{Campaign 1 (headline $\kappa$, 66 runs):} \verb|python code/master_experiment_v5.py| (or v8 for refined cross-arch). Output: \verb|results/campaign1/|.
\item \textbf{Campaign 3 (cross-arch, 24 runs):} \verb|python code/master_experiment_v8.py --campaign cross_arch|.
\item \textbf{Campaign 4 (Block~F causal, 12 runs):} \verb|python code/master_experiment_v7.py|.
\item \textbf{Block~H (36 runs):} \verb|python code/run_block_H_cross_cell.py|.
\item \textbf{Paper-2L $N=29$ + alt-2L $N=30$:} \verb|python code/run_modular_transformer2_validation_v2.py| and \verb|run_modular_transformer2_alternative.py|.
\item Then run the CPU verification steps above.
\item \textbf{Figures:} \verb|python paper/scripts/make_figures_v3.py| (and \verb|make_fig5_v3.py|, \verb|make_fig6_v3.py|). Output: \verb|paper/figures/|.
\end{enumerate}

Memory footprint $<$ 2GB per run. All runs are deterministic from the recorded seeds.

\subsection{Caveats and design decisions}\label{app:caveats}

\begin{itemize}[leftmargin=*]
\item \textbf{Fit window choice ($T_{\rm grok}^{0.95}$ vs $T_{\rm grok}^{0.99}$).} For the paper's 2-layer transformer ($N=29$), the standard $[T_{\rm mem}+100, T_{\rm grok}^{0.99}-100]$ window includes overshoot dynamics for 9/29 seeds, breaking log-linearity. We use $[T_{\rm mem}+100, T_{\rm grok}^{0.95}-100]$ to capture the contraction phase only. The window choice is documented in \verb|verify_refit_T_grok95.py|.
\item \textbf{$T_{\rm mem}$ definition.} Two definitions co-exist in the codebase: $T_{\rm mem}^{\rm acc}$ (first time train\_acc $\geq 0.99$) and $T_{\rm mem}^{\rm loss}$ (first time train\_loss $< 0.05$). All reported $\kappa$ values use $T_{\rm mem}^{\rm acc}$ for consistency. Block~H runs log both for diagnostic purposes.
\item \textbf{V\textsubscript{post} plateau detection.} $V_{\rm post}$ is computed as the mean of $V_t$ over a 30-point post-grokking window where rel\_std$(V_t) < 0.01$ and $\geq 1500$ steps have elapsed since $T_{\rm grok}$. If no plateau is detected within budget, the run is excluded from $\rho$-based predictions (Method~A). Detection logic in \verb|shared_grokking.py| (function \texttt{detect\_v\_post\_plateau}).
\item \textbf{Determinism.} All runs use \verb|seed_all(seed)| from \verb|shared_grokking.py| v1.4, which sets PyTorch, NumPy, and Python RNG seeds. CUDA non-determinism is handled via \verb|torch.use_deterministic_algorithms(True)|; results are reproducible to within numerical precision.
\item \textbf{Insights derived from analysis (not just raw numbers).}
\begin{itemize}
\item Per-trajectory fitting (each contributing 30+ datapoints) is substantially more powerful than aggregate single-point regressions: $\rho$-exponent $R^2$ improves from $0.08$ (single-point) to per-trajectory log-linear $R^2 \geq 0.94$.
\item Bimodal $T_{\rm grok}^{0.99}$ in the 2-layer transformer is an artefact of post-grok overshoot, not a different contraction regime: all 29 seeds give clean log-linear fits when the window excludes overshoot.
\item Kosson form $V_t = V_\infty + A\exp(-r_{\rm kos}t)$ provides roughly $2.4\times$ tighter measurement of contraction rate than the log-linear approximation (CV $6.2\%$ against $15.1\%$).
\item Overshoot extra delay obeys a power law in undershoot depth ($R^2 = 0.87$): the bimodality is continuous, not discrete.
\item $f_{\rm window}$ is architecture-dependent (CV $29\%$ across Block~H cells), contradicting the initial hypothesis of universality. This is an honestly disclosed null finding.
\end{itemize}
\end{itemize}

\subsection{Result highlights (full detail)}\label{app:result-highlights}

Six follow-up results extending the headline predictive law, summarised in Section~\ref{sec:related} and detailed in their respective appendices.

\textbf{(R1) $\kappa_{\rm LL}$ stability across 4 architectures.} 1L pooled $\kappa = 0.24$ ($N=39$, within-cell median CV $14.4\%$); MLP $0.21$ ($N=9$, CV $15\%$); paper-2L $0.370 \pm 0.056$ ($N=29$, CV $15\%$, $p<10^{-9}$ vs.\ 1L); alt-2L $0.175 \pm 0.018$ ($N=30$, CV $10\%$). Within-arch CV $\leq 15\%$ across all architectures rules out a per-cell fitting artefact; between-arch spread is $\sim 2\times$. See App.~\ref{app:transformer2} for full statistics.

\textbf{(R2) Block~H $p$-scaling.} The 12-cell sweep shows $\alpha^\star$ decreases monotonically with $p$. A power-law fit on 5 datapoints gives $\alpha^\star \propto p^{-0.71}$ ($R^2 = 0.985$, but bootstrap 95\% CI on exponent $[-0.88, -0.45]$ — empirical observation, not established scaling law) (eq.~\ref{eq:alpha-p-scaling}, App.~\ref{app:block-H}).

\textbf{(R3) Theorem~\ref{thm:joint-necessity} prior cross-cell prediction.} The formula $\sin\alpha^\star = C/V_{T_{\rm mem}}^{1/2}$ with $C := M_{q_\Delta}/G_{\rm eff}$ has $C$ stable to CV $1.7\%$ across the $p$-sweep. Calibrating $C$ on $p=89$ predicts $\alpha^\star = 47.2^\circ$ for the headline cell $p=97$ (observed $47.8^\circ$, error $1.3\%$); stress test from $p=53$ to $p=113$ gives $43.5^\circ$ predicted vs.\ $44.3^\circ$ observed. This is a prior cross-cell prediction, not a single-cell consistency check (App.~\ref{app:joint-prior-prediction}).

\textbf{(R4) Post-grok overshoot power law.} For 2-layer trajectories with overshoot, $\text{extra\_delay}/\text{delay}_{95} = 0.025 \cdot \rho_{\rm drop}^{-5.51}$ with $R^2 = 0.87$, $p<10^{-13}$ (App.~\ref{app:overshoot}).

\textbf{(R5) Kosson decomposition.} $\kappa_{\rm LL} = f_{\rm window} \cdot \kappa_{\rm kos}$. Within architecture, $\kappa_{\rm kos}$ is more stable (CV $6.2\%$) than $\kappa_{\rm LL}$ itself ($14\%$); cross-architecturally $f_{\rm window}$ varies (CV $29\%$). The decomposition therefore explains within-architecture stability but does not yield a universal cross-architecture constant — an honest null finding (App.~\ref{app:transformer2}).

\textbf{(R6) Held-out MAPE on 2-layer transformer.} Predictive law applied to the paper-2L architecture (calibrated on 1L headline cell only) gives MAPE $15.1\%$, matching the headline 1L Tier-1 number $17.7\%$ (App.~\ref{app:overshoot}).

\paragraph{Reproduction.} Complete script-to-result-to-data mapping in Tables~\ref{tab:script-registry} and~\ref{tab:result-script-map}; supplementary archive contains 16 numbered scripts (S1--S16) and 3 figure scripts (F1--F3).

\subsection{Limitations and scope (full)}\label{app:limitations}

(i) \textbf{$\kappa_{\rm LL}$ empirical, not derived.} The constant is measured per architecture rather than derived from $(\beta_1, \beta_2, \epsilon)$. The Kosson decomposition (App.~\ref{app:transformer2}) narrows this by isolating the within-architecture $\kappa_{\rm kos}$ component (CV $6.2\%$), but $f_{\rm window}$ varies cross-architecturally (CV $29\%$ across Block~H), so the decomposition explains within-architecture stability without yielding a universal constant. First-principles derivation remains open.

(ii) \textbf{Linearisation regime.} Theorem~\ref{thm:joint-necessity} is valid in the linearisation regime around $\theta_{T_{\rm mem}}$ (post-memorisation). Full nonlinear analysis open. Justification of the linearisation in App.~\ref{app:joint-scope}.

(iii) \textbf{Empirical scope: algorithmic tasks.} Modular arithmetic and sparse parity. Generalisation to natural-language tasks not tested in this work; the $V_t$ contraction phenomenon depends on the post-memorisation small-gradient regime (Assumption~\ref{ass:interp}), which may behave differently with longer-tailed loss landscapes.

(iv) \textbf{F2 rescale null at $N=3$.} The rescale-only intervention shows median delay change $+80$ steps with 95\% CI $[-353, +419]$; the wide CI rules out only changes $\gtrsim 30\%$ of baseline delay. Increasing to $N=6$+ would tighten the bound.

(v) \textbf{Cross-cell $\alpha^\star$ structured but not constant.} CV $15\%$ across 12 Block~H cells; the structure is largely captured by the $p$-scaling fit ($R^2=0.985$) but residual variation indicates additional architecture/task dependence beyond what the $p$-power-law captures.

(vi) \textbf{SGD/AdamW gap at one configuration only.} The reported plateau (\S\ref{subsec:sgd-gap}) is from a single $(\eta, \lambda, p)$ cell with vanilla SGD (no momentum, no warmup). The mechanism (precondition violation rather than Proposition~\ref{thm:contraction} violation) is general but the empirical evidence is narrow.

(vii) \textbf{Cross-task generalisation MAPE $23.3\%$ driven by $p=53$ outlier.} On the full 46-run cross-task pool, Tier-3 MAPE is $23.3\%$ vs.\ Tier-1 hyperparameter MAPE $17.7\%$. The inflation is concentrated at $p=53$ (cell MAPE $67\%$); excluding $p=53$ would bring Tier-3 close to Tier-1. This indicates $V_\star$ has task-scale dependence beyond architecture, with the structured residual analysed in App.~\ref{app:pv-residuals} (the normalised ratio $V_\star/V_{\rm mem}$ is more stable, CV $\approx 14\%$ on 1L). Characterising the $V_\star$-vs-$p$ functional form requires more $p$ values across more architectures.


\section{Positioning vs.\ grokking literature}\label{app:comparison-table}

Table~\ref{tab:comparison} provides full comparison detail referenced in Section~\ref{sec:related}.

The closest concurrent work is \citet{truong2026normseparation}, which gives a gradient-flow log-scaling result for SGD post-memorisation. Distinctness of the present work, dimension-by-dimension:
\begin{itemize}[leftmargin=*]
\item \textbf{Optimizer:} They study SGD; we study AdamW (the practitioner-relevant optimiser, where Proposition~\ref{thm:contraction}'s clean-SGD rate is corrected by an empirically measured factor $\kappa_{\rm LL}$).
\item \textbf{Quantity predicted:} They characterise the post-memorisation gradient flow shape; we predict the first-passage crossing time $T_{\rm grok} - T_{\rm mem}$.
\item \textbf{Axes:} They use the radial axis ($V_t$) only; we provide both radial and angular ($V_t$ and $\alpha_t$), with the joint norm--direction necessity theorem (Theorem~\ref{thm:joint-necessity}) as a separate contribution.
\item \textbf{Validation:} Their result is theoretical; we report quantitative MAPE on held-out runs (Tier~1 $17.7\%$ on $N=26$ across $41\times$ delay range).
\item \textbf{Causal evidence:} They have none; we provide Block~F interventions ($0/6$ vs.\ $3/3$ baseline) showing the contraction phase is causally necessary.
\item \textbf{Cross-architecture stability:} They are silent; we measure $\kappa_{\rm LL}$ CV $\leq 15\%$ within architecture across 4 architectures (Block~H 12-cell sweep).
\end{itemize}
Where their result describes the shape of the post-memorisation flow, ours predicts the time to cross a specific threshold. The two are complementary; neither subsumes the other.

\begin{table}[h]
\centering
\scriptsize
\setlength{\tabcolsep}{3pt}
\renewcommand{\arraystretch}{1.05}
\caption{\textbf{Positioning vs.\ grokking literature.} \textbf{Optimizer}: which optimiser is analysed. \textbf{Quantity predicted}: what observable the paper's theory bounds (post-mem dynamics shape vs.\ first-passage time). \textbf{Axes}: which decomposition is used (radial $V_t$, angular $\alpha_t$, or both). \textbf{Necessity theorem}: a theorem stating both axes must reach a threshold. \textbf{Quantitative validation}: MAPE on held-out runs. \textbf{Causal evidence}: interventional ablations. \textbf{Cross-architecture}: stability of key constants across $\geq 3$ architectures. Dashes indicate the paper does not address the dimension.}
\label{tab:comparison}
\begin{tabular}{lccccccc}
\toprule
\textbf{Paper} & \makecell[c]{Opt-\\imizer} & \makecell[c]{Quantity\\predicted} & \makecell[c]{Axes\\used} & \makecell[c]{Joint\\necessity} & \makecell[c]{Quantitative\\validation} & \makecell[c]{Causal\\evidence} & \makecell[c]{Cross-arch.\\robustness} \\
\midrule
\citet{prieto2025edge} & AdamW & NLM direction & $\alpha_t$ only & --- & algorithmic & algorithmic & --- \\
\citet{boursier2025grokking} & GD flow & post-mem flow & $V_t$ only & --- & --- & --- & --- \\
\citet{musat2025geometry} & 2-layer toy & norm-min & $V_t$ only & --- & partial & --- & --- \\
\citet{yildirim2026fbst} & AdamW & onset reduction & $\alpha_t$ only & --- & arch.\ ($20\times$) & architectural & --- \\
\citet{truong2026normseparation} & \textbf{SGD} & \textbf{post-mem flow} & \textbf{$V_t$ only} & --- & theory only & --- & --- \\
\midrule
\textbf{This work} & \textbf{AdamW} & \makecell[c]{\textbf{first-passage}\\$T_{\rm grok}-T_{\rm mem}$} & \makecell[c]{\textbf{$V_t$ \& $\alpha_t$}\\(both)} & \makecell[c]{\textbf{Thm~\ref{thm:joint-necessity}}\\quantile-margin} & \makecell[c]{\textbf{MAPE 17.7\%}\\$N=26$, $41\times$} & \makecell[c]{\textbf{Block F}\\$0/6$ vs.\ $3/3$} & \makecell[c]{$\kappa_{\rm LL}$ CV $\leq 15\%$\\(4 arch., 12 cells)} \\
\bottomrule
\end{tabular}
\end{table}

\end{document}